\documentclass[11pt]{article}
\usepackage[margin=1in]{geometry}
 
\usepackage[]{amsmath,amssymb,epsfig}
\usepackage{amsthm}
\usepackage{amsmath,mathtools}
\usepackage{amssymb}
\usepackage{graphicx}
\usepackage{epstopdf}
\usepackage{comment}
\usepackage{array}
\usepackage{algorithm}
\usepackage{url}

\usepackage{lscape}
\usepackage{algpseudocode}
\usepackage{setspace}
\usepackage{multicol}
\usepackage{multirow}
\usepackage{color}
\usepackage{colortbl}
\usepackage{xcolor}
\usepackage[colorlinks=true, linkcolor=red, citecolor=blue, urlcolor=blue]{hyperref}

\usepackage{placeins}

\usepackage[%
    font={small,sf},
    labelfont=bf,
    format=hang,    
    format=plain,
    margin=0pt,
    width=0.8\textwidth,
]{caption}
\usepackage[list=true]{subcaption}
\usepackage[round]{natbib}

\newtheorem{theorem}{Theorem}

\newtheorem{assumption}{Assumption}

\newtheorem{claim}{Claim}

\newtheorem{corollary}{Corollary}

\newtheorem{definition}{Definition}

\newtheorem{lemma}{Lemma}

\newtheorem{proposition}{Proposition}

\numberwithin{equation}{section}

\newcommand{\calF}{\ensuremath{\mathcal{F}}}

\newcommand{\norm}[1]{\big\|{#1}\big\|}

\newcommand{\set}[1]{\left\{{#1}\right\}}

\newcommand{\est}[1]{\widehat{#1}}

\newcommand{\matR}{\ensuremath{\mathbb{R}}}

\newcommand{\gap}{G_{T,i}}
\newcommand{\gapf}{G_{T,i}(f)}
\newcommand{\gapof}{G_{T,i}(\overline f)}

\newcommand{\lemp}{\widehat{L}_{T,i}}
\newcommand{\ltrans}{L_{T,i}}

\newcommand{\PS}{\mathbb{P}_{\mathcal{S}_T}}

\providecommand{\keywords}[1]{\noindent \textbf{Keywords:} #1}   

\title{Learning switched non-linear dynamical systems from a single trajectory}

\author{Sunny G.W. Wang \\ \texttt{sunny.wanggw@ntu.edu.sg} \and  Hemant Tyagi \\ \texttt{hemant.tyagi@ntu.edu.sg} 
}
\date{
Division of Mathematical Sciences, 
       SPMS, 
       NTU Singapore 637371
}

\begin{document}
\maketitle

\begin{abstract}
We study empirical risk minimization for learning non-linear dynamical systems whose transition dynamics may switch over time. Under stability assumptions, and i.i.d switching over a set of $K$ modes, we derive non-asymptotic bounds on the prediction risk expressed in terms of the metric entropy of the underlying function class. We instantiate our general result for H\"older and linear function classes, obtaining explicit convergence rates that depend on the effective sample size $Tp_i$, where $T$ is the trajectory length and $p_i$ is the probability of observing mode $i$. Numerical simulations support our theoretical findings. To the best of our knowledge, these results are the first non-asymptotic guarantees for learning switched nonlinear dynamical systems from a single trajectory.
\end{abstract}


\tableofcontents

%
\section{Introduction} \label{sec:intro}
Dynamical systems are a powerful framework for describing how complex phenomena evolve over time. They are ubiquitous in many domains, where one observes a system through a sequence of dependent measurements and seeks to recover the underlying law governing its evolution. In this work, we consider non-linear dynamical systems of the form 
\begin{equation} \label{eq:model}
    X_{t+1} = f^*_{s(t+1)}(X_{t}) + \varepsilon_{t+1}, \qquad t=0, \dots, T-1, \quad X_0 = 0,
\end{equation}
where $(\varepsilon_{t})_{t=1}^T$ are i.i.d copies of a sub-gaussian random variable $\varepsilon$ (the ``excitation'' of the system), and $(s(t))_{t=1}^T$ are i.i.d random variables drawn from a discrete distribution over the set of ``modes'' $\set{1,\dots,K}$. Given access to the states $(X_t)_{t=1}^T$ and the modes $(s(t))_{t=1}^T$, the goal is to estimate $f_1^*, \dots,f^*_K \in \calF$ where $f_i^*: \matR^d \rightarrow \matR^d$, and $\calF$ is a known class of functions.  

The switching model is a natural extension of classical one-mode non-linear models (where $K=1$), and arises in many applications where transition dynamics may differ across modes, e.g., robotics \citep{ames2014human}, cyber-physical systems \citep{alur2015principles}, biology \citep{BARTOCCI20093149, Bortolussi2008, Ghosh2002}, control and reinforcement learning \citep{Greene25, Zhang26, FAN2024106695}, and  quantitative finance \citep{savku2020regime, cajueiro2002stochastic}. For example, in economic and financial time series data,  market behavior may differ across high and low-volatility regimes. 

\subsection{Contributions}
In this work, we perform a non-asymptotic study for learning \emph{switched} non-linear dynamical systems of the form \eqref{eq:model}, using empirical risk minimization (ERM) (see \eqref{eq:erm} for the formulation). In particular, we obtain high-probability risk bounds for learning the non-linear transition maps $f_1^*,\dots,f^*_K \in \calF$ from a single trajectory generated as per \eqref{eq:model}. The presence of different modes complicates the analysis in a highly non-trivial way, since it introduces an additional layer of randomness. As a result, arguments  designed for the single-mode case (e.g., \citep{ziemann2022, physicsmixing25}) do not extend directly in a straightforward manner. However, we show that learning remains possible using ERMs under relatively mild assumptions. Our contributions are  as follows. 
\begin{enumerate}
    \item We obtain high-probability bounds on the prediction risk for ERMs in the switched setting (see Theorem \ref{thm:final-risk}), which to the best of our knowledge, is the first non-asymptotic result in the non-linear case. The results are obtained explicitly in terms of the complexity of the function class through the metric entropy, the degree of dependence through the mixing properties of the Markov chain, and the amount of information accumulated in each mode. Moreover, these are instantiated to obtain rates of convergence for particular choices of $\calF$, such as the Hölder class (Corollary \ref{cor:holder-rates}) and the linear class (Corollary \ref{cor:linear-rates}).

    \item Our theory is obtained under mild and transparent assumptions and only requires: a suitable Lyapunov drift condition, sub-gaussian excitation, and a star-shaped condition on the shifted function class (see Section \ref{subsec:prelim}).
    Importantly, we do not impose boundedness of the function class $\calF$ (both the domain and range of $f \in \calF$), or hypercontractivity type assumptions, which are commonly assumed in the literature (for $K = 1$, see e.g., \citep{ziemann2022,physicsmixing25}). 
    Instead, our Lyapunov drift condition instead allows the process to be controlled probabilistically while preserving an unbounded state space formulation. 

    \item When $\calF$ corresponds to the Hölder class (with smoothness $\beta > 0$), we obtain rates of convergence in both smooth ($\beta > d/2$) and rough ($\beta \leq d/2$) non-parametric regimes. In particular, the rough regime regime is relatively more challenging to analyze as the entropy integral diverges, and leads to slower rates of convergence. In contrast, existing works (cf., \citep{ziemann2022, physicsmixing25}) typically focus on the smooth-regime; see also the discussion after Corollary \ref{cor:holder-rates}.

    \item The proof strategy employed differs from several existing approaches in the one-regime nonlinear dynamical system framework (cf., \citep{ziemann2022,physicsmixing25}), which are of independent interest. 
    
    \begin{itemize}
        \item For the martingale offset bound (see Proposition \ref{prop:mart-offset}), we do not rely on Chernoff type arguments within the chaining analysis, in contrast to existing works for $K =1$. Instead, we employ martingale arguments in the spirit of de la Pe\~na \citep{pena1999, selfnormbook}, and Ville-type arguments \citep{ville1939} combined with careful decompositions to handle the additional randomness introduced by regime switching. See discussion after Theorem \ref{thm:final-risk} for more details.
        
        \item For the uniform law of large numbers component (see Proposition \ref{prop: ulln-master}) wherein the population risk is bounded by the empirical risk with high probability, we use a chaining and critical radius argument that is arguably more suited for nonparametric classes $\calF$ (e.g. H\"older class). See discussion after Corollary \ref{cor:holder-rates} for more details.  
    \end{itemize} 
\end{enumerate}
To the best of our knowledge, we provide the first results for provably learning switched nonlinear dynamical systems (using ERMs) from a single trajectory under weak stability and excitation assumptions. The theory extends recent non-asymptotic analyses for learning nonlinear dynamical systems for the single-mode case (cf., \citep{ziemann2022, physicsmixing25}) to the substantially richer setting of \emph{switched non-linear dynamics} , and relaxes assumptions commonly used in the literature.
\subsection{Related work}
%
There is by now a rich literature on finite sample guarantees for learning dynamical systems from a single trajectory. As discussed below, these results focus mainly on learning: (a) linear dynamical systems (both switching/non-switching) and (b) non-linear dynamical systems with $K=1$ (non-switching).
\paragraph{Learning switched linear dynamical systems.}
When each $f^*_i$ is linear, i.e., represented by a matrix $A^*_i$ (for $i =1,\dots,K$), then the problem of learning $A^*_i$'s under model \eqref{eq:model} (switched linear dynamical systems) has been studied extensively in the literature, with finite sample guarantees. For instance, \cite{sayedanaIID2021} consider a least squares estimator and show for the i.i.d switching setting (where $s(t)$'s are i.i.d) that $A^*_i$ can be estimated at the near parametric rate $O(\sqrt{\frac{\log T}{T}})$ almost surely. \cite{Sarkar19Switch} consider a relatively more general setup than \cite{sayedanaIID2021}, where the system (assumed to be of unknown order) is learned from input-output data using subspace methods. When $s(t)$'s evolve in a Markovian manner then the model is referred to as Markov jump linear systems. In this setting, \cite{sattar2021identification} propose a procedure based on least-squares and show that the system parameters can be estimated at the near parametric rate, with high probability. \cite{sayedanamarkov2021} considered a similar setup as \cite{sattar2021identification} and obtain similar error rates that hold almost surely. We note that the ERM in \eqref{eq:erm} that we analyze has the same least-squares objective as in some of the above works (eg., \citep{sayedanaIID2021, sayedanamarkov2021}); the main difference is of course that the class $\calF$ is nonparametric in our setting. If $K = 1$, then the problem has been studied extensively recently; it is now well established that the least squares estimator estimates the underlying parameter matrix at the near optimal rate for stable (e.g., \citep{Jedra20}) and also marginally stable systems \citep{Sarkar19,Simchowitz18a}.

\paragraph{Learning non-linear dynamical systems ($K = 1$).}
When $K = 1$, i.e., without switching, there has been a recent surge of results which obtain finite-sample guarantees for learning non-linear dynamics under model \eqref{eq:model}. \cite{ziemann2022} study this problem for the ERM, which is the same estimator as \eqref{eq:erm} when $K=1$. They obtain finite-sample guarantees on the expected prediction risk which depend on the metric entropy of the class $\calF$. Moreover, they instantiate their main result for specific choices of $\calF$ (e.g., H\"older, generalized linear models etc.). \cite{physicsmixing25} adopts the techniques of \citep{ziemann2022} to analyze the performance of a (smoothness-) regularized least squares objective and obtain results holding with high probability for the Sobolev class. More details on the technical differences between our proof techniques with those in \citep{ziemann2022,physicsmixing25} can be found in the discussion after Theorem \ref{thm:final-risk} and Corollary \ref{cor:holder-rates}. \cite{Negin24_analytic} consider the setting of linearly parametrized non-linear systems, i.e., $f_1^*$ is a unknown linear combination of known non-linear functions. They provide non-asymptotic analyses for learning such systems, this is shown for a least-squares estimator, and also a set-membership estimator that estimates an uncertainty set containing the true parameters. \cite{foster2020} consider the setting where $f_1^*$ is a generalized linear model. They obtain guarantees for learning the underlying parameter matrix with an iterative algorithm.

\section{Problem setup and main results} \label{sec:prob_setup}
Let $(X_t)_{t \in \mathbb{Z}}$ be a vector-valued stochastic process on a probability space $(\Omega, \mathcal{A}, \mathbb{P})$. We denote $[K] = \{1, \dots, K\}$ to be a finite set of cardinality $K$, representing the modes. We are interested in the nonlinear autoregressive model with mode-switching, where observations come in the form of $(X_{t+1}, s(t+1))_{t=0}^{T-1} \in \mathbb{R}^d \times [K]$, drawn under the model \eqref{eq:model}. 
%
\begin{itemize}
    \item $(\varepsilon_{t})_{t \in \mathbb{Z}}$ are i.i.d copies of a centered sub-gaussian random variables $\varepsilon$ with unit covariance matrix $I_d$ (independent entries). These can be thought of as the ``excitation'' of the system.

    \item $(s(t))_{t \in \mathbb{Z}}$ are i.i.d discrete random variable with probability mass function $$p_i := \mathbb{P}(s(t+1) = i) \in (0, 1), \quad\text {for all } i \in [K].$$ 

    \item Each $f_i^* \in \calF$ where $\calF$ denotes a known class of functions.
\end{itemize}
We denote $\mathcal{I}_i := \{t:s(t+1) = i\}$ and $n_i = |\mathcal{I}_i|$ to be the set of design points under state $i$, and its cardinality respectively. Define the primitive filtration $\mathcal{F}_t := \sigma(X_0, (\varepsilon_k, s(k))_{k \leq t})$, so that $X_{t}$ is $\mathcal{F}_{t}$-measurable.

Suppose that the samples $(X_{t+1}, s(t+1))_{t=0}^{T-1}$ and $(\widetilde X_{t+1}, \widetilde s(t+1))_{t=0}^{T-1}$ are drawn independently from \eqref{eq:model}, which can be thought of as the training and test sequence respectively. Let $\mathbf{1}\{.\}$ denote the indicator function, and $\|.\|_2$ be the standard Euclidean norm induced by the inner product $\langle u, v\rangle_2 = u^\top v$, where for a column vector $v$, $v^{\top}$ denotes its transpose. Moreover, let $\|f\|_{\infty} = \sup_{x \in \mathbb{R}^d} \|f(x)\|_2$ denote the uniform norm.

For an estimate $\widehat f_i, \text{ for }i \in [K]$, we are interested in the transient risk with respect to the new, independent sample generated from the same law as the training trajectory. Let $\mathcal{D}_T := ((X_{t+1}, s(t+1))_{t=0}^{T-1}$ denote the training data. Then the weighted prediction risk (i.e., prediction error contributed by mode $i$) is given by
\begin{equation}
    R^w_{T,i}(\widehat f_i) := \frac{1}{T}\sum_{t=0}^{T-1}\mathbb{E}\left[\left\|\widehat f_i(\widetilde X_t) - f_i^*(\widetilde X_t) \right\|_2^2 \mathbf{1}\{\widetilde s(t+1) = i \} \mid \mathcal{D}_T \right],
\end{equation}
where the expectation is taken over the test trajectory, and is thus a random quantity. There are two conflicting sources of error that contribute to $R^w_{T,i}$ - on the one hand, when the probability of state $i$ occurring is small, the errors are expected to increase. On the other hand, summing over less terms decreases the error. A more natural quantity is the conditional risk, given by
\begin{equation}\label{eq:cond-pop-risk}
    R_{T,i}^{cond}(\widehat f_i) := \frac{1}{p_i}R_{T,i}^w(\widehat f_i) = \frac{1}{Tp_i}\sum_{t=0}^{T-1}\mathbb{E}\left[\left\|\widehat f_i(\widetilde X_t) - f_i^*(\widetilde X_t)  \right\|_2^2 \mathbf{1}\{\widetilde s(t+1) = i \} \mid \mathcal{D}_T \right],
\end{equation}
which is the prediction error given that mode $i$ occurs, which avoids the conflicting error source problem given by the weighted prediction risk.

We will focus on the theoretical study of the empirical prediction risk minimizer
\begin{equation}\label{eq:erm}
\begin{aligned}
    \widehat f_i &\in \arg\min_{f_i \in \mathcal{F}}\sum_{t \in \mathcal{I}_i}\left\|X_{t+1} - f_i(X_t)\right\|_2^2   \\ 
    &= \arg\min_{f_i \in \mathcal{F}}\sum_{t=0}^{T-1}\left\|X_{t+1} - f_i(X_t)\right\|_2^2 \mathbf{1}\left\{s(t+1) = i \right\}.
\end{aligned}
\end{equation}
%
%
%
\subsection{Preliminaries} \label{subsec:prelim}
We begin by stating the assumptions that we will need to make in order to establish our main result (Theorem \ref{thm:final-risk}).
\begin{assumption}\label{ass:0}
    The observations are drawn according to model \eqref{eq:model}, where $\varepsilon$ is a subgaussian random variable with a constant subgaussian norm. 
\end{assumption}

\begin{assumption}[$(\rho, B, \alpha)$-exponential drift]\label{ass:1}
        There exist constants $\alpha > 0$, $\rho \in (0, 1)$, $B < \infty$ such that
        \begin{equation*}
\mathbb{E}\left[\exp\left(\alpha \|X_{t+1}\|_2^2 \right) \mid X_t = x\right] \leq \rho\exp\left(\alpha \|x\|_2^2 \right) + B.
        \end{equation*}
    \end{assumption}
Denote the shifted class
\begin{equation*}
    \mathcal{F}_i^* := \mathcal{F} - f_i^* = \{f - f_i^* \mid f \in \mathcal{F}\}.
\end{equation*}
Our results will depend on the local complexity of $\mathcal{F}_i^*$. We will need this set to be star-shaped as stated below.
\begin{assumption}[Star-shapedness]\label{ass:2}
    For any $f \in \mathcal{F}_i^*, \lambda \in [0, 1]$, it holds that  $\lambda f \in \mathcal{F}_{i}^*$.
\end{assumption}

Let us now introduce the notion of an envelope function.
\begin{definition}\label{def:envelope-func}
    A measurable function $F:\mathbb{R}^d \rightarrow \mathbb{R}_+$ is an envelope function for a class $\mathcal{F}$ if
    \begin{equation*}
        \sup_{f \in \mathcal{F}} \|f(x)\|_2 \leq F(x), \qquad \forall x \in \mathbb{R}^d.
    \end{equation*}
\end{definition}

\begin{assumption}[Existence of squared-integrable envelope]\label{ass:3}
    An envelope function $F$ exists for the class $\mathcal{F}_{i
    }^*$, such that $\mathbb{E}[F(X)^2] < \infty$.
\end{assumption}
\begin{assumption}\label{ass:4}
The quantity $M_T := \sup_{f \in \mathcal{F}_i^*}\sup_{x \in \mathcal{B}_T} \|f(x)\|_2$ is finite.
\end{assumption}
It will be useful to discuss the above assumptions.
\begin{itemize}
    \item   Assumption \ref{ass:1} ensures the stability of the dynamical system, and is a special case of classical Foster-Lyapunov stability conditions in Markov chains theory \citep{tweedie2009}. In particular, it guarantees sub-gaussian tails for the dynamical system, and under additional assumptions on the density of $\varepsilon$, implies that the chain is geometrically ergodic. This allows us to localize the sample path $X_t$ to lie within a ball $$\mathcal{B}_T := \left\{x \in \mathbb{R}^d: \|x\|_2 \leq R_T \right\},$$ where the radius of the ball $R_T$ grows only logarithmically in $T$, as seen in Lemma \ref{lem:traj-ball} in Appendix \ref{app:appendix-dynamics}. In our chaining arguments, we will make use of the localized uniform norm $\|f\|_{\infty, \mathcal{B}_T} := \sup_{x \in \mathcal{B}_T} \|f(x)\|_2$. This allows us to work with potentially unbounded classes $\mathcal{F}$, relaxing the uniformly bounded assumption commonly used in the literature.
    
    \item Assumption \ref{ass:2} is a common condition that is used in high dimensional statistics (see e.g., \cite{Wainwright2019}. Recalling that $\mathcal{F}_i^* = \mathcal{F} - f_i^*$, the star-shapedness of $\mathcal{F}_i^*$ is satisfied when $\mathcal{F}$ is convex, for instance.

    \item Assumption \ref{ass:3} is a familiar condition in empirical process theory \citep{vandervaart1996}, and is satisfied for common function classes, such as the linear or Hölder class of functions, as will be seen in Section \ref{sec:instantiations}.

    \item Assumption \ref{ass:4} is naturally satisfied by many classes, for example the Hölder class of functions. This will be made clearer in the corollaries.
\end{itemize}

\begin{claim}\label{rmk:linear-stab}
    For example, in the linear case where $f^*_{s(t+1)}(X_t) = A^*_{s(t+1)} X_t,$ with $A^*_{s(t+1)} \in \mathbb{R}^{d\times d}$, the $(\rho, B, \alpha)$ exponential drift condition is satisfied under the assumption that $\max_{1 \leq i \leq K}\|A_i^*\|_2 =: \rho_A < 1$, and $\varepsilon_{t+1}$ are i.i.d sub-gaussian random variables. A proof is given in Appendix \ref{app:appendix-dynamics}. . 
\end{claim}

%
%
%
%
Our main result depends on an important $S$-persistence parameter as stated in the Theorem below, which captures the degree of dependence. We will denote $\norm{\cdot}_2$ to be the operator norm.
\begin{theorem}[\cite{samson2000}]\label{thm:S-persis}
    Let $Y_t$ be a Markov chain, and $f$ be a bounded function taking values on a measurable space $\mathcal{M}$. Then for any $\xi > 0$, it holds that
    \begin{equation}\label{eq:S-persis}
        \mathbb{E}\left[\exp\left(-\xi \sum_{t=0}^{T-1}\|f(Y_t)\|_2^2 \right) \right] \leq \exp\left(-
        \xi \sum_{t=0}^{T-1} \mathbb{E}\left[\|f(Y_t)\|_2^2 \right] + \frac{\xi^2 S}{2}\sum_{t=0}^{T-1}\mathbb{E}\left[\|f(Y_t)\|_2^4 \right] \right),
    \end{equation}
where $S = \|\Gamma\|_2^2$, and $\Gamma$ is a matrix of dimension $T$, which captures the dependence between the random variables $(Y_0, \dots, Y_{T-1})$ of the sample.
\end{theorem}

The inequality in Theorem \ref{thm:S-persis} is sometimes referred to as $S$-persistence, see e.g., \cite{physicsmixing25}. It appears in a more general form as equation (3.12) in \cite{samson2000}. Indeed, to obtain \eqref{eq:S-persis} from (3.12) of \cite{samson2000}, take the singleton class with the function as the squared Euclidean norm (in their notation $N = 1$, $g_1(x) = \|f(x)\|_2^2$). See also \cite{ziemann2022}.

In our results, the $S$-persistence parameter, is given by the operator norm of the dependency matrix involving the pair $Z_t = (X_t, s(t+1))$, which we denote as $S = \|\Gamma_Z\|_2^2$. 
Depending on the nature of the Markov chain $Z_t$, the dependence quantity $S$ may either be constant, or depend on $T$. In our uniform law of large numbers result, we only require that $S$ be finite, which is true for any fixed horizon $T$. However, in order to get reasonable convergence rates in our Corollaries, we require $S = o(T)$, which will be satisfied under certain stability conditions of the Markov chain. Several examples of Markov chains where $S$ will be a constant and independent of $T$ are given in \cite{samson2000}, \cite{ziemann2022}, and \cite{physicsmixing25}. They include geometrically $\phi$-mixing chains, uniformly ergodic Markov chains, and contracting Markov chains.

%
%
\subsection{Main results}
We are now ready to state our main result. 
\begin{theorem}\label{thm:final-risk}
    Let $\widehat f_i$ be the empirical risk minimizer in \eqref{eq:erm}. Suppose that Assumptions \ref{ass:0} to \ref{ass:4} are satisfied. Let $\overline K \geq \underline K$ be strictly positive integers such that $\gamma_{f} := 2^{-\overline K}$, $\gamma_{c} := 2^{-\underline K}$, and denote the metric entropy as
\begin{equation*}
    J_T(\epsilon) := \log\mathcal{N}\left(\mathcal{F}_i^*, \epsilon, \|.\|_{\infty, \mathcal{B}_T}\right).
\end{equation*}
Let $\mu_i = Tp_i$. There exist constants $C_1, C_2, C_3 > 0$, and  constant $C_{\rho, B} > 0$ (depending only on $\rho, B$) such that if 
\begin{equation} \label{eq:mainthm_RT_mui}
    R_T = C_1\sqrt{\frac{1}{\alpha}\log\left(\frac{TC_{\rho, B}}{\delta} \right)}, \quad \mu_i \geq C_2\log\left(\frac{1}{\delta} \right),
\end{equation}
then for $\delta \in (0, 1)$ it holds with probability at least $1 - \delta$ that
\begin{equation}\label{eq:bound-final}
\begin{aligned}
    R_{T,i}^{cond}(\widehat f_i) 
    \leq C_3
    \inf_{\gamma_c > 0, \gamma_f \in [0, \gamma_c]} \Bigg\{\int_{\gamma_f}^{\gamma_c} &\Bigg[\sqrt{\frac{J_T(\epsilon)}{\mu_i}} + \frac{J_T(\epsilon)}{\mu_i} \Bigg] d\epsilon + \frac{J_T(\gamma_c)}{\mu_i} \\
    &+ \sqrt{d}\gamma_f + \gamma_c^2 +\left(\frac{1}{\sqrt{\mu_i}} \right)\log\left(\frac{e}{\delta} \right) \Bigg\} + (r^*)^2. 
\end{aligned}
\end{equation}
Here $r^*$ is any strictly positive radius that satisfies the critical radius  equation 
\begin{equation}\label{eq:fixed-point-radius}
   \begin{aligned}
    \frac{p_i (r^*)^2}{2} = \Phi_1(r^*; \delta/4) + \Phi_2(r^*;\delta/3) + \Phi_3(r^*, \delta/4) + I_T(F),
    \end{aligned}
\end{equation}
where for $r > 0$, and $\overline{\mathcal{F}}_r$ defined in \eqref{eq:Fir_Fbar_r_def},
\begin{equation*}
\begin{aligned}
    &\Phi_1(r, \delta) := 4M_T r \sqrt{\frac{Sp_i}{T}}\left\{\sqrt{\log \mathcal{N}(\overline{\mathcal{F}}_r, r, \|.\|_{\infty, \mathcal{B}_T})} + \sqrt{\log(1/\delta)} \right\}, \\
    &\Phi_2(r, \delta) := 10M_T \sqrt{\frac{Sp_i}{T}}\left\{\int_{r\gamma_f}^{r/2} \sqrt{\log\mathcal{N}(\overline{\mathcal{F}}_r, \epsilon, \|.\|_{\infty, \mathcal{B}_T})} d\epsilon + r + r\sqrt{\log(1/\delta)}  \right\}, \\
    &\Phi_3(r, \delta) := \sqrt{2}M_T r\gamma_f \sqrt{\frac{\log(1/\delta)}{T}},
\end{aligned}
\end{equation*}
and 
\begin{equation*}
   I_T(F) := \frac{p_i}{T}\sum_{t=0}^{T-1}\mathbb{E}\left[F^2(X_t)\mathbf{1}\{\|X_t\|_2 > R_T\}\right].
\end{equation*}
\end{theorem}
The proof is provided in Section \ref{sec:proof_thm_final} and consists of two main components: Proposition \ref{prop:mart-offset} (Martingale offset bound) and Proposition \ref{prop: ulln-master} (Uniform law of large numbers). 

The term inside the infimum arises from the control of the martingale offset term, which involves the entropy integral of the shifted class under the localized norm, and is handled in Proposition \ref{prop:mart-offset}. The $\gamma_c$ and $\gamma_f$ terms arises from the fine and coarse-scale approximation errors in the finite chaining argument that we employ, which can be optimized in order to obtain convergence rates. Note that the effective sample size $\mu_i$ is the main driver of the prediction error, as opposed to total trajectory length $T$ in the non-switching framework, which is not surprising as less frequently visited modes are expected to incur larger errors. In contrast to existing works in the single mode ($K=1$) case \citep{ziemann2022, physicsmixing25}, which usually derive the martingale offset bound using a simple Chernoff method, we rely on tail bounds for martingale of de la Pe\~{n}a type. In the context of switching, the random modes introduce an additional level of randomness, and a direct Chernoff method must simultaneously handle the martingale difference sequence and the random regime occurrences. Consequently, the iterative conditioning argument underlying the standard Chernoff argument becomes substantially more involved and cumbersome. Instead, our approach effectively decouples the regime count fluctuations from the martingale fluctuations, avoiding the unwieldy joint-conditional mgf calculation that arises in the switched setting.

The contribution $(r^*)^2$ is determined by a localized complexity term associated with
the uniform law of large numbers of the empirical process (see Proposition \ref{prop: ulln-master}), and is presented in the usual critical radius formulation arising in localized empirical process analysis of least-squares (see e.g., \cite{vandegeer2000}). What is specific to the present setting is the form of the $\Phi_i, i, 1,2, 3$ terms, which depends on the dependence parameter $S$, the mode frequency $p_i$, the truncation level through $M_T$, and the metric entropy of the truncated class. Roughly speaking, after dividing by $p_i$ on both sides of \eqref{eq:fixed-point-radius}, one has order $\sqrt{S/\mu_i}$, up to the function-class complexity terms. This critical radius formulation is distinct from the lower isometry presentation in \cite{ziemann2022} and \cite{physicsmixing25},  which establishes the uniform law of large numbers result (for $K = 1)$ using trajectory hypercontractivity, by taking a simple union bound on the boundary of the sphere, and does not use chaining as we do in this work. We believe that the chaining formulation is particularly well-suited to the analysis of non-parametric classes, whose complexity may vary across scales / smoothness regimes. In particular, retaining a non-zero fine scale cutoff allows for analysis even when the entropy integral diverges in the rough regime $\beta \leq d/2$.

The envelope term $I_T(F)$ arises from the truncation argument employed in the uniform law of large numbers proof. Under the exponential drift condition and an appropriate growth condition on the envelope, $I_T(F)$ can be made of smaller order than the leading stochastic terms, as seen in the proof of the subsequent corollaries.

We now present the following corollaries, which are useful for interpreting Theorem \ref{thm:final-risk}. Their proofs are detailed in Section \ref{sec:instantiations}. Recall the notation that for $a, b > 0$, we say $a \lesssim b$ if there exists a constant $C > 0$ such that $a \leq Cb$. If $a \lesssim b$ and $a \gtrsim b$, then we write $a \asymp b$.
\begin{corollary}[H\"older class]\label{cor:holder-rates}
    Let $\mathcal{H}^{\beta}(L, \mathcal{B}_T)$ be the Hölder class of functions as in Definition \ref{def:local-Hölder} in Appendix \ref{app:appendix-dynamics}, and define the stability class $\mathcal{S}(\alpha, \rho, B)$ as the set of functions $f$ that satisfies the $(\rho, B,\alpha)$-exponential drift condition, given by
    \begin{equation*}
        \mathcal{S}(\alpha, \rho, B) := \left\{f:\mathbb{E} \left[\exp\left(\alpha \|f(x) + \varepsilon \|_2^2 \right)\right] \leq \rho\exp(\alpha \|x\|_2^2) + B, \quad \forall x \in \mathbb{R}^d\right\}.
    \end{equation*}
    Let $\mathcal{F} = \mathcal{S}(\alpha, \rho, B) \cap \mathcal{H}_T^{\beta}(L_T)$ where 
    \begin{equation*}
        \mathcal{H}_T^\beta(L_T) := \left\{f: \mathbb{R}^d \rightarrow \mathbb{R}^d: f|_{\mathcal{B}_T} \in \mathcal{H}^{\beta}(L_T, \mathcal{B}_T) \right\}.
    \end{equation*}
Assume that there exist constants $C_L < \infty$ and $q_L \geq 0$, such that 
\begin{equation*}
    L_T \leq C_L \log^{q_L}(eT),
\end{equation*}
and that for all $i \in \{1, \dots, K\}$, we have $f_i^* \in \mathcal{F}$. Suppose that Assumptions \ref{ass:0} to \ref{ass:4} are satisfied, and \eqref{eq:mainthm_RT_mui} holds. Let $\delta \in (0, 1)$. 
    Whenever $\beta > d/2$, we have with probability at least $1 - \delta$,
  \begin{equation*}
    R_{T,i}^{cond}(\widehat f_i) \lesssim \operatorname{polylog}(T)
    \left(\frac{S}{\mu_i}\right)^{\frac{2\beta}{2\beta + d}}.
\end{equation*}
    On the other hand, whenever $\beta \leq d/2$, we have with probability at least $1 - \delta$, 
        \begin{equation*}
    R_{T,i}^{cond}(\widehat f_i) \lesssim \operatorname{polylog}(T)\left[\sqrt{d} + \left(\frac{S}{\mu_i} \right)^{\beta/d} \right].
    \end{equation*}
\end{corollary}
When $S = O(1)$, our rates in the smooth regime $\beta > d/2$ coincide with those in \cite{ziemann2022}, who obtain $T^{-2\beta/(2\beta+d)}$, if we only had one mode. It is also similar to the rates obtained in \cite{physicsmixing25}, who obtain $T^{-2s/(2s+d)}$ under the stronger smoother condition $s \geq 2d$, where $s$ is the Sobolev parameter, since their results are written for Sobolev spaces. However, we also obtain rates in the rough regime $\beta \leq d/2$, which is not fully addressed in both \cite{ziemann2022} and \cite{physicsmixing25}. We conjecture that the gap in convergence rates between the smooth and rough regime are more structural rather being an artifact of our analysis. More generally, when the relevant entropy integral diverges, entropy-based analyses of global ERMs exhibit a phase transition and can lead to rates that are slower than the minimax rate, unless additional structural assumptions are made. This feature is present even in the i.i.d case, see e.g., \cite{han2021} and \cite{birge1993}.

\begin{corollary}[Stable linear systems] \label{cor:linear-rates}
    Let $\mathcal{F} = \left\{f(x) = Ax: A \in \mathbb{R}^{d\times d}: \|A\|_{2} < 1 \right\}$, be the stable, linear class of functions. Suppose that Assumptions \ref{ass:0} to \ref{ass:4} are satisfied and \eqref{eq:mainthm_RT_mui} holds. Then for $\delta \in (0, 1)$, with probability at least $1 - \delta$, it holds that
    \begin{equation} \label{eq:cor_linpred_err_bd}
  R_{T,i}^{cond}(\widehat f_i) \lesssim  \frac{d^2 S}{\mu_i}.
    \end{equation}
    %
    %
   %
   Moreover, \eqref{eq:cor_linpred_err_bd} implies $\norm{\est{A}_i - A_i^*}_F^2 \lesssim \frac{d^2 S}{\mu_i}$. 
\end{corollary}
Corollary \ref{cor:linear-rates} is included primarily to illustrate how the general theory specializes to the linear setting. While our results are obtained directly under the generic Assumptions stated before Theorem \ref{thm:final-risk}, the linear case often admits more specialized analyses which can yield shaper results under alternative stability assumptions.

\section{Proof of Theorem \ref{thm:final-risk}} \label{sec:proof_thm_final}
We start by presenting the basic inequality, which provides an upper bound for the empirical risk of the estimator.

\begin{lemma}
    Let $\widehat f_i$ be the solution as in \eqref{eq:erm}. Then, it holds that 
    \begin{equation}\label{eq:basic-inequality}
        \sum_{t \in \mathcal{I}_i}\|\widehat f_i(X_t) - f^*_i(X_t)\|_2^2 \leq \sum_{t \in \mathcal{I}_i} \left[ 4\langle \varepsilon_{t+1}, \widehat f_i(X_t) - f^*_i(X_t) \rangle_2 - \|\widehat f_i(X_t) - f^*_i(X_t) \|_2^2\right].
    \end{equation}
\end{lemma}
\begin{proof}
By adding and subtracting $X_{t+1}$, note that for any $f_i$ under the model \eqref{eq:model}, we have for a fixed state $i$ and time $t$
\begin{equation}\label{eq:risk-decomp}
\begin{aligned}
    \|f_i(X_t) - f^*_i(X_t)\|_2^2 &= \|X_{t+1} - f^*_i(X_t) - (X_{t+1} - f_i(X_t)) \|_2^2 \\
&=\|X_{t+1} - f^*_i(X_t)\|_2^2 + \|X_{t+1} - f_i(X_t)\|_2^2 \\
&- 2\langle X_{t+1} - f^*_i(X_t), X_{t+1} - f_i(X_t) \rangle_2.
\end{aligned}
\end{equation}
Moreover, we have 
\begin{equation*}
\begin{aligned}
    2\langle X_{t+1} - f^*_i(X_t), X_{t+1} - f_i(X_t) \rangle_2 &= 2\langle X_{t+1} - f^*_i(X_t),X_{t+1} - f^*_i(X_t) + f^*_i(X_t) - f_i(X_t) \rangle_2 \\
    &= 2\|X_{t+1} - f^*_i(X_t) \|_2^2 + 2\langle \varepsilon_{t+1}, f^*_i(X_t) - f_i(X_t) \rangle_2.
\end{aligned}
\end{equation*}
Plugging the identity above into \eqref{eq:risk-decomp}, we obtain
\begin{equation}\label{eq:master-inter}
\begin{aligned}
    \|f_i(X_t) - f^*_i(X_t)\|_2^2 &= \|X_{t+1} - f_i(X_t)\|_2^2 + \|X_{t+1} - f^*_i(X_t)\|_2^2 \\
    &- 2\|X_{t+1} - f^*_i(X_t) \|_2^2- 2\langle \varepsilon_{t+1}, f^*_i(X_t) - f_i(X_t)  \rangle_2.
\end{aligned}
\end{equation}
In particular, the identity above holds for $f_i = \widehat f_i$. Moreover, since $\widehat f_i$ minimises the empirical risk by construction, we have
    \begin{equation*}
      \sum_{t \in \mathcal{I}_i} \|X_{t+1} - \widehat f_i(X_t)\|_2^2 \leq \sum_{t \in \mathcal{I}_i} \|X_{t+1} - f^*_i(X_t)\|_2^2.
    \end{equation*}
Summing over time points in state $i$ and using the upper bound above, we have by multiplying \eqref{eq:master-inter} by two and a re-arrangement of terms
\begin{equation*}
    \sum_{t \in \mathcal{I}_i}  \|\widehat f_i(X_t) - f^*_i(X_t)\|^2_2 \leq \sum_{t \in \mathcal{I}_i}\left[ 4\langle \varepsilon_{t+1},\widehat f_i(X_t) - f^*_i(X_t) \rangle - \|\widehat f_i(X_t) - f^*_i(X_t)\|_2^2\right],
\end{equation*}
as desired. 
\end{proof}

An upper bound on the RHS of \eqref{eq:basic-inequality} is provided in the next proposition, often called the martingale offset bound, which eventually provide rates of convergence. The proof is provided in Section \ref{sec:proof_mart_offset} and involves tail bounds for martingales (in the spirit of de la Pe\~{n}a) and their hitting times, as opposed to the Chernoff method more commonly seen in the literature (cf., \citep{ziemann2022, physicsmixing25}).
\begin{proposition}[Martingale offset bound]\label{prop:mart-offset}
Let $\gamma_{f} := 2^{-\overline K}$, and $\gamma_{c} := 2^{-\underline K}$. Denote the metric entropy as
\begin{equation*}
    J_T(\epsilon) := \log\mathcal{N}\left(\mathcal{F}_i^*, \epsilon, \|.\|_{\infty, \mathcal{B}_T}\right).
\end{equation*}
Suppose that the conditions of Lemma \ref{lem:unif-decomp} are satisfied and Assumption \ref{ass:1} holds. Then, there exists $C_{\rho, B}$ (depending only on $\rho, B$) such that for $\eta \in (0, 1)$, $\theta > 0$, if
\begin{equation*}
    \mu_i \geq \frac{3}{\theta^2}\log\left(\frac{5}{\eta}\right) \quad \text{and} \quad R_T = 10 \sqrt{\frac{1}{\alpha}\log\left(\frac{TC_{\rho, B}}{\eta} \right)}, 
\end{equation*}
then with probability at least $1 - \eta$, it holds that 
\begin{equation}
\begin{aligned}
    &\frac{1}{n_i}\sum_{t \in \mathcal{I}_i}  \|\widehat f_i(X_t) - f^*_i(X_t)\|^2_2 \lesssim \\ &\inf_{\gamma_c > 0, \gamma_f \in [0, \gamma_c]} \left\{\int_{\gamma_f}^{\gamma_c} \left[\sqrt{\frac{J_T(\epsilon)}{\mu_i}} + \frac{J_T(\epsilon)}{\mu_i} \right] d\epsilon + \frac{J_T}{\mu_i} 
    + \sqrt{d}\gamma_f + \gamma_c^2 +\left(\frac{1}{\sqrt{\mu_i}} \right)\log\left(\frac{e}{\eta} \right) \right\}.
\end{aligned}
\end{equation}
\end{proposition}
%
%
It will now be useful to introduce additional notation that will be needed for the uniform law of large numbers result (Proposition \ref{prop: ulln-master} below). Define the time-averaged distribution, which we call the transient measure, to be
\begin{equation*}
    \mu_T(A) := \frac{1}{T}\sum_{t=0}^{T-1}\mathbb{P}\left(X_t \in A \right), \qquad \forall A \in \mathcal{B}(\mathbb{R}^d).
\end{equation*}
For a Borel set $A$, let $P_X(x, A) = \mathbb{P}(X_{t+1} \in A \mid X_t = x)$ be the transition kernel of the Markov Chain. For a measurable function $f$, define the quantities
\begin{equation*}
    \lemp(f) := \frac{1}{T}\sum_{t=0}^{T-1}\|f(X_t)\|_2^2 \mathbf{1}\left\{s(t+1) = i \right\},
\end{equation*}
and 
\begin{equation*}
\begin{aligned}
    \ltrans(f)
    := \mathbb{E}\left[\lemp(f) \right] =  p_i \|f\|^2_{L^2(\mu_T )} 
    = p_i \int_{\mathbb{R}^d}\|f(x)\|_2^2 d\mu_T(x). 
\end{aligned}
\end{equation*}
It will also be useful to denote 
\begin{equation*}
    \widehat L_{n_i, i}(f) := \frac{1}{n_i}\sum_{t=0}^{T-1}\|f(X_t)\|_2^2 \mathbf{1}\{s(t+1) = i \}.
\end{equation*}
A lower bound on the LHS of \eqref{eq:basic-inequality} is provided in the next proposition, which connects the empirical risk to the population risk. Its proof is provided in Section \ref{sec:proof_prop_ulln}. 
Before presenting the statement, we will need to define additional quantities.
\begin{itemize}
    \item In order to control the empirical process on unbounded domains to obtain our uniform law of large numbers result, we require a truncation argument. Define the  truncation operator $T_{M_T}$ (which we note is  surjective) by
\begin{equation}\label{eq:trunc-operator}
    T_{M_T}[f](x) := \overline f(x) = f(x) \mathbf{1}\{\|f(x)\|_2 \leq M_T \} + M_T \frac{f(x)}{\|f(x)\|_2}  \mathbf{1}\{\|f(x)\|_2 > M_T\}.
\end{equation}

\item Denote the set 
\begin{equation}\label{eq:Fir_Fbar_r_def}
\mathcal{F}_{i,r}^* = \{f \in \mathcal{F}_i^*: \|f\|_{L^2(\mu_T)} = r\}, \ \  \overline{\mathcal{F}}_r := \set{\text{Image of } \mathcal{F}_{i,r}^* \text{ under } T_{M_T}[\cdot]}
\end{equation}
Note that by construction, every element of the truncated class $\overline{\mathcal{F}}_r$ has norm bounded by $M_T$.
\end{itemize}
\begin{proposition}[Uniform law of large numbers] \label{prop: ulln-master}
Suppose that Assumptions \ref{ass:1} to \ref{ass:4} are satisfied. For $\delta \in (0, 1)$, let $r^*$ be any strictly positive radius that satisfies the critical radius  equation 
\begin{equation*}
   \begin{aligned}
    \frac{p_i (r^*)^2}{2} = \Phi_1(r^*; \delta/4) + \Phi_2(r^*;\delta/3) + \Phi_3(r^*, \delta/4) + I_T(F),
    \end{aligned}
\end{equation*}
where 
\begin{equation*}
\begin{aligned}
    &\Phi_1(r, \delta) := 4M_T r \sqrt{\frac{Sp_i}{T}}\left\{\sqrt{\log \mathcal{N}(\overline{\mathcal{F}}_r, r, \|.\|_{\infty, \mathcal{B}_T})} + \sqrt{\log(1/\delta)} \right\}, \\
    &\Phi_2(r, \delta) := 10M_T \sqrt{\frac{Sp_i}{T}}\left\{\int_{r\gamma_f}^{r/2} \sqrt{\log\mathcal{N}(\overline{\mathcal{F}}_r, \epsilon, \|.\|_{\infty, \mathcal{B}_T})} d\epsilon + r + r\sqrt{\log(1/\delta)}  \right\}, \\
    &\Phi_3(r, \delta) := \sqrt{2}M_T r\gamma_f \sqrt{\frac{\log(1/\delta)}{T}},
\end{aligned}
\end{equation*}
and 
\begin{equation*}
   I_T(F) := \frac{p_i}{T}\sum_{t=0}^{T-1}\mathbb{E}\left[F^2(X_t)\mathbf{1}\{\|X_t\|_2 > R_T\}\right],
\end{equation*}
where we recall that $\gamma_{f} := 2^{-\overline K}$, and $\gamma_{c} := 2^{-\underline K}$ from Proposition \ref{prop:mart-offset}. Then for any $r \geq r^*$, with probability at least $1 - \delta$, for every $f \in \mathcal{F}_i^*$ such that $\|f\|_{L^2(\mu_T)}^2 \geq r^2$, it holds that
\begin{equation*}
       \lemp(f)  > \frac{\ltrans(f)}{2}.
\end{equation*}
\end{proposition}

\begin{proof}[Proof of Theorem \ref{thm:final-risk}]
We distinguish between two possible cases. Fix a function $g \in \mathcal{F}^*_{i,r}$.

In the first case, we have $\|g\|^2_{L^2(\mu_T)} \geq r^2$. By Proposition \ref{prop: ulln-master}, with probability at least $1 - \delta/3$, it holds that
\begin{equation*}
    \begin{aligned}
        \ltrans(g) \leq 2\lemp(g).
    \end{aligned}
\end{equation*}
In the second case, we have
\begin{equation*}
    \lemp(g) = p_i \|g\|_{L^2(\mu_T)} \leq p_i r^2.
\end{equation*}
Thus, we have
\begin{equation*}
    \ltrans(g) \leq \max\{2\lemp(g), p_i r^2\} \leq 2\lemp(g) + p_i r^2,
\end{equation*}
where the bound is uniform over $g \in \mathcal{F}_i^*$.

Note that the empirical risk can be written as
\begin{equation*}
    \begin{aligned}
        \lemp(g) &= \frac{n_i}{T}\frac{1}{n_i}\sum_{t=0}^{T-1}\|g(X_t)\|_2^2 \mathbf{1}\{s(t+1) = i \}.
    \end{aligned}
\end{equation*}
By the Chernoff bound, we have 
\begin{equation*}
    \mathbb{P}(n_i / T > 2p_i) \leq 
    \exp(-\mu_i/3).
\end{equation*}
Thus, by setting the RHS to $\delta/3$, the event $\{n_i/T \leq 2p_i \}$ occurs as soon as $\mu_i \geq 3\log(6/\delta)$. On the event $\{n_i/T \leq 2p_i\}$, we have for every $g \in \mathcal{F}_i^*$
\begin{equation*}
    \lemp(g) \leq 2p_i \widehat L_{n_i, i}(g).
\end{equation*}

Next, by applying Proposition \ref{prop:mart-offset} with $g = \widehat f_i - f_i^*$, we have by the union bound, it holds with probability at least $1 - 2\delta/3$
\begin{equation*}
    \begin{aligned}
        &\lemp(\widehat f_i - f_i^*) = \frac{1}{T}\sum_{t \in \mathcal{I}_i} \left\|\widehat f_i(X_t) - f_i^*(X_t)  \right\|_2^2 \leq 2p_i \widehat L_{n_i, i}(g) \\
        &\lesssim p_i \inf_{\gamma_c > 0, \gamma_f \in [0, \gamma_c]} \left\{\int_{\gamma_f}^{\gamma_c} \left[\sqrt{\frac{J_T(\epsilon)}{\mu_i}} + \frac{J_T(\epsilon)}{\mu_i} \right] d\epsilon + \frac{J_T(\gamma_c)}{\mu_i} 
    + \sqrt{d}\gamma_f + \gamma_c^2 +\left( \frac{1}{\sqrt{\mu_i}} \right)\log\left(\frac{e}{\delta} \right) \right\}.
    \end{aligned}
\end{equation*}
Gathering facts, we have with probability at least $1 - \delta$
\begin{equation*}
    \begin{aligned}
    &R_{T,i}^w(\widehat f_i) \lesssim  p_i r^2 +\\
    &p_i \inf_{\gamma_c > 0, \gamma_f \in [0, \gamma_c]} \left\{\int_{\gamma_f}^{\gamma_c} \left[\sqrt{\frac{J_T(\epsilon)}{\mu_i}} + \frac{J_T(\epsilon)}{\mu_i} \right] d\epsilon + \frac{J_T(\gamma_c)}{\mu_i} 
    + \sqrt{d}\gamma_f + \gamma_c^2 +\left( \frac{1}{\sqrt{\mu_i}} \right)\log\left(\frac{e}{\delta} \right) \right\},
    \end{aligned}
\end{equation*}
which implies that 
\begin{equation*}
\begin{aligned}
    &R_{T,i}^{cond}(\widehat f_i) = \frac{1}{p_i}R_{T,i}^w(\widehat f_i) 
    \lesssim 
    \\ &\inf_{\gamma_c > 0, \gamma_f \in [0, \gamma_c]} \left\{\int_{\gamma_f}^{\gamma_c} \left[\sqrt{\frac{J_T(\epsilon)}{\mu_i}} + \frac{J_T(\epsilon)}{\mu_i} \right] d\epsilon + \frac{J_T(\gamma_c)}{\mu_i} 
    + \sqrt{d}\gamma_f + \gamma_c^2 +\left(\frac{1}{\sqrt{\mu_i}} \right)\log\left(\frac{e}{\delta} \right) \right\} + r^2.
\end{aligned}
\end{equation*}
The proof is now complete.
\end{proof}

\section{Proof of Proposition \ref{prop:mart-offset}}\label{sec:proof_mart_offset}
We first decompose the RHS of the basic inequality in \eqref{eq:basic-inequality} into four different terms by employing a chaining argument, before proceeding to bound each of them separately.

\begin{lemma}\label{lem:unif-decomp}
For each integer $k \in [\underline K, \overline K]$, let $\pi_k(f):\mathcal{F}_i^* \rightarrow \mathcal{F}_i^*$ denote a mapping onto a $2^{-k}$ cover of $\mathcal{F}$ with respect to the localized norm $\|f\|_{\infty, \mathcal{B}_T} := \sup_{x \in \mathcal{B}_T} \|f(x)\|_2$ such that
\begin{equation*}
 \|\pi_k(f) - f \|_{\infty, \mathcal{B}_T}  = \sup_{x \in \mathcal{B}_T} \|\pi_k(f)(x) - f(x)  \|_2 \leq 2^{-k}, \quad \forall f \in \mathcal{F}.
\end{equation*}
Then it holds that 
\begin{equation}\label{eq:bound-decomp}
\begin{aligned}
    &\sup_{f \in \mathcal{F}_i^*} \left\{\sum_{t \in \mathcal{I}_i}4\langle \varepsilon_{t+1}, f(X_t)\rangle_2 - \|f(X_t)\|_2^2 \right\} \leq \sup_{f \in \mathcal{F}_i^*}\left\{\sum_{t \in \mathcal{I}_i}\left[4 \langle \varepsilon_{t+1}, f(X_t) - \pi_{\overline K}(f)(X_t)\rangle_2 \right] \right\} \\&+ \sup_{f \in \mathcal{F}_i^*}\left\{ \sum_{k = \underline K +1}^{\overline K} \sum_{t \in \mathcal{I}_i}\left[4\langle \varepsilon_{t+1}, \pi_k(f)(X_t) - \pi_{k-1}(f)(X_t)\rangle_2\right] \right\}  
     \\&+ \sup_{f \in \mathcal{F}_i^*}\left\{ \sum_{t \in \mathcal{I}_i} \left[4\langle \varepsilon_{t+1}, \pi_{\underline K}(f)(X_t) \rangle_2 - \frac{\|\pi_{\underline K}(f)(X_t) \|_2^2}{2} \right]\right\} + \sup_{f \in \mathcal{F}_i^*} \sum_{t \in \mathcal{I}_i}\left[\left\|\pi_{\underline K}(f)(X_t) - f(X_t) \right\|_2^2 \right] \\
    &=: A_1 + A_2 + A_3 + A_4.
\end{aligned}
\end{equation}
\end{lemma}
\begin{proof}
Let 
\begin{equation}\label{eq:mart-offset-ni}
    M_{n_i}(f) := \sum_{t \in \mathcal{I}_i} 4\langle \varepsilon_{t+1}, f(X_t)\rangle_2,
\end{equation}
and denote $S_{n_i}(f) := \sum_{t \in \mathcal{I}_i} \|f(X_t)\|_2^2$.
Note that $M_{n_i}(f)$ can be written as a telescopic sum of the form 
\begin{equation*}
M_{n_i}(f) = M_{n_i}(f) + M_{n_i}(\pi_{\underline K}(f)) + \sum_{k=\underline K+1}^{\overline K}\left[M_{n_i}(\pi_k(f)) - M_{n_i}(\pi_{k-1}(f)) \right] - M_{n_i}(\pi_{\overline K}(f)).
\end{equation*}
Using the telescopic sum above, together with adding and subtracting half of $S_{n_i}(\pi_{\underline K}(f))$, we have
\begin{equation*}
\begin{aligned}
    M_{n_i}(f) - S_{n_i}(f) &= M_{n_i}(f) - M_{n_i}(\pi_{\overline K}(f)) + \sum_{k=\underline K+1}^{\overline K}\left[M_{n_i}(\pi_k(f)) - M_{n_i}(\pi_{k-1}(f)) \right] \\ &+\left[M_{n_i}(\pi_{\underline K}(f)) - \frac{S_{n_i}(\pi_{\underline K}(f)}{2} \right] + \left[\frac{S_{n_i}(\pi_{\underline K}(f))}{2} - S_{n_i}(f) \right].
\end{aligned}
\end{equation*}
Note that by the triangle inequality, we have 
\begin{equation*}
\begin{aligned}
    \left[\frac{S_{n_i}(\pi_{\underline K}(f))}{2} - S_{n_i}(f) \right] &= \frac{1}{2}\sum_{t \in \mathcal{I}_i }\left[\|\pi_{\underline K}(f)(X_t) - f(X_t) + f(X_t)\|_2^2 - 2\|f(X_t)\|_2^2 \right] \\
    &\leq \sum_{t \in \mathcal{I}_i}\left[\left\|\pi_{\underline K}(f)(X_t) - f(X_t) \right\|_2^2 \right].
\end{aligned}
\end{equation*}
This implies that 
\begin{equation*}
\begin{aligned}
    M_{n_i}(f) - S_{n_i}(f) &\leq M_{n_i}(f) - M_{n_i}(\pi_{\overline K}(f)) + \sum_{k=\underline K+1}^{\overline K}\left[M_{n_i}(\pi_k(f)) - M_{n_i}(\pi_{k-1}(f)) \right] \\ &+\left[M_{n_i}(\pi_{\underline K}(f)) - \frac{S_{n_i}(\pi_{\underline K}(f)}{2} \right] + \sum_{t \in \mathcal{I}_i}\left[\left\|\pi_{\underline K}(f)(X_t) - f(X_t) \right\|_2^2 \right].
\end{aligned}
\end{equation*}
Taking supremums on both sides of the inequality above, and noting that $M_{n_i}(f)$ is linear in $f$ concludes the proof.    
\end{proof}

We now bound the individual terms on the RHS of Lemma \ref{lem:unif-decomp}, these are stated in Lemmas \ref{lem:diff}, \ref{lem:chaining},  \ref{lem:final} and \ref{lem:coarse_martoffset} below. Their proofs are provided in Sections \ref{subsec:proof_diff_martoffset}--\ref{subsec:proof_coarse_martoffset}.
\begin{lemma}[Fine scale bound]\label{lem:diff}
 Suppose the conditions of Lemma \ref{lem:unif-decomp} are satisfied. Moreover, let $\overline K$ be an integer. Denote $\mu_i = Tp_i$ to be the expected number of times under state $i$. Then there exists a constant $C_{\rho, B}$ such that if 
\begin{equation*}
    \mu_i \geq \frac{2(1 +\theta/3)}{\theta^2}\log\left(\frac{3}{\eta} \right), \quad \text{and} \quad  R_T =3 \sqrt{\frac{1}{\alpha}\log\left(\frac{TC_{\rho, B}}{\eta} \right)},
\end{equation*} 
for $\eta \in (0, 1)$ and $\theta > 0$,  then it holds with probability at least $1 - \eta$ that 
\begin{equation*}
     A_1 \leq 2^{2-\overline K}\left[(2+\theta)\mu_i\sqrt{d} + C\kappa\sqrt{(2+\theta)\mu_i  \log\left(\frac{3}{\eta} \right)}  \right] =: B_1(\eta).
\end{equation*}
\end{lemma}
\begin{lemma}[Chaining bound]\label{lem:chaining}
    Suppose the conditions of Lemma \ref{lem:unif-decomp} hold true and Assumption \ref{ass:1} holds. For a set $A$, denote its cardinality by $\# A$. Let 
    \begin{equation}
        \Pi_k := \left\{(\pi_k(f), \pi_{k-1}(f)): f \in \mathcal{F}_i^* \right\}
    \end{equation}
    denote the set of links at scale $k$, where $\# \Pi_k \leq \mathcal{N}(\mathcal{F}_i^*, 2^{-k}, \|.\|_{\infty, \mathcal{B}_T}) \times \mathcal{N}(\mathcal{F}_i^*, 2^{-(k-1)}, \|.\|_{\infty, \mathcal{B}_T}) \leq \mathcal{N}^2(\mathcal{F}_i^*, 2^{-k}, \|.\|_{\infty, \mathcal{B}_T})$. Then there exists a constant $C_{\rho, B}$ such that if
    \begin{equation*}
         R_T = 3 \sqrt{\frac{1}{\alpha}\log\left(\frac{TC_{\rho, B}}{\eta} \right)} \quad \text{and} \quad  \mu_i \geq \log\left(\frac{3}{\eta} \right)\frac{2(1 + \theta/3)}{\theta^2}.
    \end{equation*}
 for $\eta \in (0, 1)$ and $\theta > 0$, then it holds with probability at least $1 - \eta$ that
\begin{equation*}
    \begin{aligned}
        A_2 &\leq C\sqrt{(2+\theta)\mu_i}\left\{\int_{2^{-(\overline K + 1)}}^{2^{-(\underline K + 1)}}\sqrt{\log\mathcal{N}(\mathcal{F}_i^*, \epsilon, \|.\|_{\infty, \mathcal{B}_T})} d\epsilon + 1 + \log\left(\frac{1}{\eta} \right) \right\}
       \\
       &+ C\left\{\int_{2^{-(\overline K + 1)}}^{2^{-(\underline K + 1)}} \log\mathcal{N}(\mathcal{F}_i^*, \epsilon, \|.\|_{\infty, \mathcal{B}_T}) d\epsilon + 1 + \log\left(\frac{1}{\eta} \right)\right\} \\
       &=: B_2(\eta)
    \end{aligned}
\end{equation*}
\end{lemma}
\begin{lemma}[Self-normalized bound]\label{lem:final}
    Suppose the conditions of Lemma \ref{lem:unif-decomp} hold true, and Assumption \ref{ass:1} holds. Then there exists a constant $C_{\rho, B}$ such that if 
    \begin{equation*}
         R_T = 3\sqrt{\frac{1}{\alpha}\log\left(\frac{TC_{\rho, B}}{\eta} \right)} \quad \text{and} \quad  \mu_i \geq \log\left(\frac{3}{\eta} \right)\frac{2(1 + \theta/3)}{\theta^2},
    \end{equation*}
    for $\eta \in (0, 1)$ and $\theta > 0$, then it holds with probability at least $1 - \eta$ that  
    \begin{equation}
         A_3 \leq 16\sigma^2\left[\log\mathcal{N}(\mathcal{F}_i^*, 2^{-\underline K}, \|.\|_{\infty, \mathcal{B}_T}) + \log\left(\frac{3}{\eta} \right) \right] =: B_3\left(\eta \right)
    \end{equation}
\end{lemma}
\begin{lemma}[Coarse scale bound] \label{lem:coarse_martoffset}
Suppose the conditions of Lemma \ref{lem:unif-decomp} hold true
, and Assumption \ref{ass:1} holds. Then, there exists a constant $C_{\rho, B}$ such that if
\begin{equation*}
    R_T = 2\sqrt{\frac{1}{\alpha}\log\left(\frac{TC_{\rho, B}}{\eta} \right)} \quad \text{and} \quad \mu_i \geq \log\left(\frac{2}{\eta} \right)\frac{2(1 + \theta/3)}{\theta^2},
\end{equation*}
for $\eta \in (0, 1), \theta > 0$, then it holds with probability at least $1 - \eta$ that 
\begin{equation*}
    A_4 \leq (1+\theta)\mu_i 2^{-2\underline K}.
\end{equation*}
\end{lemma}
With the above lemmas as hand, we are now in a position to prove Proposition \ref{prop:mart-offset}. Recall the following lower tail bound for binomial random variables \cite{boucheron2013}
\begin{equation*}
    \mathbb{P}(n_i < (1-\theta)\mu_i) < \exp\left(-\frac{\theta^2 \mu_i}{2} \right).
\end{equation*}
Assigning a budget of $\eta/5$ to the RHS, we have under the condition
\begin{equation*}
    \mu_i \geq \frac{2}{\theta^2}\log\left(\frac{5}{\eta} \right),
\end{equation*}
it holds with probability at least $1 - \eta/5$ that 
\begin{equation*}
    n_i \geq (1-\theta)\mu_i. 
\end{equation*}
Now, assigning a budget of $\eta/5$ to each of the $A_i, i = 1, \dots 4$ bounds, we have by the union bound, under the conditions
\begin{equation*}
    \mu_i \geq \frac{3}{\theta^2}\log\left(\frac{5}{\eta}\right) \quad \text{and} \quad R_T > \sqrt{\frac{1}{\alpha}\log\left(\frac{10TC_{\rho, B}}{\eta} \right)},
\end{equation*}
it holds with probability at least $1 - \eta$
\begin{equation*}
    \frac{1}{n_i}\sum_{j=1}^4 A_j \leq \sum_{j=1}^3 B_j\left(\frac{\eta}{5}\right)\left[(1 - \theta)\mu_i \right]^{-1} + C2^{-2\underline K}. 
\end{equation*}
Re-arranging and grouping terms together on the RHS, we have with probability at least $1 - \eta$
\begin{equation*}
\begin{aligned}
    \frac{1}{n_i}\sum_{j=1}^4 A_j  \lesssim \int_{\gamma_f}^{\gamma_c} \left[\sqrt{\frac{J_T(\epsilon)}{\mu_i}} + \frac{J_T(\epsilon)}{\mu_i} \right] d\epsilon + \frac{J_T(\gamma_c)}{\mu_i} 
     + \sqrt{d}\gamma_f + \gamma_c^2 +\left(\frac{1}{\sqrt{\mu_i}} \right)\log\left(\frac{e}{\eta} \right).
 \end{aligned}
\end{equation*}
The statement follows by the basic inequality of Lemma \ref{lem:unif-decomp}, and taking infimums over the scales.

\subsection{Proof of Lemma \ref{lem:diff}}\label{subsec:proof_diff_martoffset}
By the Cauchy-Schwarz inequality, 
\begin{equation*}
    \begin{aligned}
        \sup_{f \in \mathcal{F}_i^*}\sum_{t \in \mathcal{I}_i} 4 \langle \varepsilon_{t+1}, f(X_t) - \pi_{\overline K}(f)(X_t)\rangle_2 &\leq \sup_{f \in \mathcal{F}_i^*} 4 \sum_{t \in \mathcal{I}_i}\left[\|\varepsilon_{t+1}\|_2 \|f(X_t) - \pi_{\overline K}(f)(X_t)\|_2\right].
    \end{aligned}
\end{equation*}
For some $R_T > 0$, let 
\begin{equation*}
    \mathcal{S}_T := \left\{\max_{0 \leq t \leq T-1} \|X_t\|_2 \leq R_T \right\}.
\end{equation*}
Moreover, for $w > 0$, denote the event 
\begin{equation*}
    \mathcal{E}_1 := \left\{\sup_{f \in \mathcal{F}_i^*} 4 \sum_{t \in \mathcal{I}_i}\left[\|\varepsilon_{t+1}\|_2 \|f(X_t) - \pi_{\overline K}(f)(X_t)\|_2\right] > w \right\}.
\end{equation*}
Then, by intersecting $\mathcal{E}_1$ with $\mathcal{S}_T$, we have 
\begin{equation*}
\begin{aligned}
        \mathbb{P}(A_1 > w) &\leq \mathbb{P}\left(\mathcal{E}_1  \cap \mathcal{S}_T \right) + \mathbb{P}\left(\mathcal{E}_1  \cap \mathcal{S}_T^c \right) \\
        &\leq \mathbb{P}\left(\mathcal{E}_1  \cap \mathcal{S}_T \right) + \mathbb{P}(\mathcal{S}_T^c) \\
        &=: \mathfrak{p}_{11} + \mathfrak{p}_{12}. 
\end{aligned}
\end{equation*}
On the event $\mathcal{S}_T$, we have $\|X_t\|_2 \leq R_T$ for every $t=0, \dots, T-1$, which means that $X_t \in \mathcal{B}_T, \forall t =0, \dots, T-1$ by construction of the ball. This implies that 
\begin{equation*}
    \begin{aligned}
        \sup_{f \in \mathcal{F}_i^*} 4 \sum_{t \in \mathcal{I}_i}\|\varepsilon_{t+1}\|_2 \|f(X_t) - \pi_{\overline K}(f)(X_t)\|_2 &\leq \sup_{f \in \mathcal{F}_i^*} 4 \sum_{t \in \mathcal{I}_i}\|\varepsilon_{t+1}\|_2 \|f(X_t) - \pi_{\overline K}(f)(X_t)\|_{\infty, \mathcal{B}_T} \\
        &\leq 2^{2-\overline K}\sum_{t \in \mathcal{I}_i} \|\varepsilon_{t+1}\|_2.
    \end{aligned}
\end{equation*}

Recall that the realisations of the random variable $\varepsilon$ are i.i.d. By Claim \ref{cl:dist-eq-sum} in Appendix \ref{app:tech-lemmas}, we have
\begin{equation*}
    \sum_{t \in \mathcal{I}_i} \|\varepsilon_{t+1}\|_2 =\sum_{t=0}^{T-1} \|\varepsilon_{t+1}\|_2 \mathbf{1}\{s(t+1) = i \} \stackrel{d}{=} \sum_{k=1}^{n_i}\|\widetilde \varepsilon_{k}\|_2,
\end{equation*}
where $\stackrel{d}{=}$ denotes equality in distribution, and $\widetilde \varepsilon_k$ are independent copies of $\varepsilon$. For the boundary case of $n_i = 0$, we adopt the convention that the sum is simply zero.

Thus, we have  
\begin{equation*}
    \begin{aligned}
     \mathfrak{p}_{11} &\leq  \mathbb{P}\left(2^{2-\overline K}\sum_{t \in \mathcal{I}_i} \|\varepsilon_{t+1}\|_2 > w\right) \\ &= \mathbb{P}\left(2^{2-\overline K}\sum_{k=1}^{n_i}\|\widetilde \varepsilon_k \|_2 > w \right) \\
        &= \mathbb{P}\left(\left\{2^{2-\overline K}\sum_{k=1}^{n_i}\|\widetilde \varepsilon_k \|_2 > w\right\} \cap \{n_i < J\} \right) \\&+ \mathbb{P}\left(\left\{2^{2-\overline K}\sum_{k=1}^{n_i}\|\widetilde \varepsilon_k \|_2 > w\right\} \cap \{n_i \geq J \} \right) \\
        &\leq \mathbb{P}\left(2^{2-\overline K}\sum_{k=1}^{J}\|\widetilde \varepsilon_k \|_2 > w \right) + \mathbb{P}\left(n_i \geq J \right) \\
        &\leq \mathfrak{p}_{111} + \mathfrak{p}_{112},
    \end{aligned}
\end{equation*}
where $J$ is a threshold that will be suitably chosen, typically around the expected number of observations in state $i$, $\mu_i$.

Recall that $n_i = \sum_{t=0}^{T-1}\mathbf{1}\{s(t+1) = i\}$. That is, $n_i$ is a Binomial random variable with $T$ trials and probability of success $p_i$. Using the Chernoff tail bound for Binomial random variables, we have by choosing $J= \lfloor(1+\theta)\mu_i\rfloor + 1$ that
\begin{equation}\label{eq:A12-tail}
   \mathfrak{p}_{112} = \mathbb{P}(n_i \geq J) = \mathbb{P}(n_i \geq (1+\theta)\mu_i) \leq \exp\left(-\frac{\theta^2 \mu_i}{2(1 + \theta/3)} \right).
\end{equation}

On the other hand, since $\varepsilon$ has identity covariance, we have for every $k = 1, \dots, J$
\begin{equation*}
\mathbb{E}\left[\|\widetilde \varepsilon_{k} \|_2^2\right] = \sum_{\ell=1}^d \mathbb{E}\left[\left(\widetilde \varepsilon_{k}^{(\ell)}\right)^2\right] = d.
\end{equation*}
Moreover, by Jensen's inequality,
\begin{equation*}
    \mathbb{E}\|\widetilde \varepsilon_{k}\|_2 = \mathbb{E}\sqrt{\|\widetilde \varepsilon_{k}\|_2^2} \leq \sqrt{\mathbb{E}\|\widetilde \varepsilon_{k}\|_2^2} = \sqrt{d}.
\end{equation*}

Let $\varepsilon$ have sub-gaussian norm $\kappa$. Since $\widetilde \varepsilon_k$ are independent copies of $\varepsilon$, we have by Theorem 3.1.1 of \cite{vershynin2018} that
\begin{equation*}
   \left\|\|\widetilde \varepsilon_{k}\|_2 - \mathbb{E}\|\widetilde \varepsilon_{k}\|_2 \right\|_{\psi_2} \leq C\kappa^2.
\end{equation*}
Using the general Hoeffding inequality in Theorem 2.6.2 of \cite{vershynin2018}, we have for $w > (2+\theta)\mu_i \sqrt{d}2^{2 - \overline K}$,
\begin{equation}\label{eq:A11-bound}
    \begin{aligned}
    \mathfrak{p}_{111}  &= \mathbb{P}\left(\sum_{k=1}^{J} \|\widetilde \varepsilon_{k}\|_2 > w2^{\overline K - 2} \right) \\
    &= \mathbb{P}\left(\sum_{k=1}^J \|\widetilde \varepsilon_k\|_2 - \mathbb{E}\left[\sum_{k=1}^J \|\widetilde \varepsilon_k\|_2 \right] > w2^{\overline K - 2} -\mathbb{E}\left[\sum_{k=1}^J \|\widetilde \varepsilon_k \|_2 \right] \right) \\
    &\leq \mathbb{P}\left(\sum_{k=1}^J \|\widetilde \varepsilon_k\|_2 - \mathbb{E}\left[\sum_{k=1}^J \|\widetilde \varepsilon_k\|_2 \right] > w2^{\overline K - 2} - J\sqrt{d} \right) \\
    &\leq \exp\left(-\frac{c\left(w2^{\overline K - 2} - J\sqrt{d} \right)^2}{\sum_{k=1}^J \left\|\widetilde \varepsilon_k \right\|^2_{\psi_2}   } \right) \\ 
    &\leq \exp\left(-\frac{c\left(w2^{\overline K - 2} - J\sqrt{d} \right)^2}{JC\kappa^2 } \right) \\
    &\leq \exp\left(-\frac{c\left(w2^{\overline K - 2} - (2+\theta)\mu_i \sqrt{d} \right)^2}{(2+\theta)\mu_i C \kappa^2 } \right).
    \end{aligned}
\end{equation}

Gathering facts, we have for any $\theta > 0$, $\mu_i \geq 1$, $w > (2+\theta)\mu_i \sqrt{d}2^{2 - \overline K}$
\begin{equation*}
    \mathfrak{p}_{11} \leq \exp\left(-\frac{c\left(w2^{\overline K - 2} - (2+\theta)\mu_i \sqrt{d} \right)^2}{(2+\theta)\mu_i C^2\kappa^2} \right) + \exp\left(- \frac{\theta^2 \mu_i}{2(1 + \theta/3)} \right).
\end{equation*}

Upper bounding the RHS of \eqref{eq:A12-tail} with $\eta/3$, we obtain the condition 
\begin{equation}\label{eq:mu-cond}
    \mu_i \geq \frac{2(1 +\theta/3)}{\theta^2}\log\left(\frac{3}{\eta} \right),
\end{equation}
for the threshold to be satisfied. Similarly upper bounding $\mathfrak{p}_{12}$ by $\eta/3$, we obtain the condition
\begin{equation}\label{eq:radius-cond}
    R_T > \sqrt{\frac{1}{\alpha}\log\left(\frac{3TC_{\rho, B}}{\eta} \right)}.
\end{equation}
Finally, setting the RHS of \eqref{eq:A11-bound} to $\eta/3$, and solving for $w$, we obtain
\begin{equation*}
    w = 2^{2-\overline K}\left[(2+\theta)\mu_i\sqrt{d} + C\kappa\sqrt{(2+\theta)\mu_i  \log\left(\frac{3}{\eta} \right)}  \right].
\end{equation*}
Thus, whenever \eqref{eq:mu-cond} and \eqref{eq:radius-cond} is satisfied, we obtain with probability at least $1 - \eta$
\begin{equation*}
    A_1 \leq 2^{2-\overline K}\left[(2+\theta)\mu_i\sqrt{d} + C\kappa\sqrt{(2+\theta)\mu_i  \log\left(\frac{3}{\eta} \right)}  \right].
\end{equation*}
%
%
%
%
\subsection{Proof of Lemma \ref{lem:chaining}}\label{subsec:proof_chaining_martoffset}
The proof is broken into several steps.

\textit{Step 1: Bounding the visit counts.}
For $u_k > 0$, let
\begin{equation*}
    \mathcal{E} := \left\{\sup_{f \in \mathcal{F}} \sum_{k=\underline K + 1}^{\overline K} \left[M_{n_i}(\pi_k(f) - \pi_{k-1}(f))\right] > \sum_{k=\underline K + 1}^{\overline K} u_k \right\},
\end{equation*}
where we recal $M_{n_i}$ from \eqref{eq:mart-offset-ni}. We have 
\begin{equation*}
    \begin{aligned}
        \mathbb{P}\left(\mathcal{E} \right) &= \mathbb{P}(\mathcal{E} \cap \{n_i < J\}) + \mathbb{P}(\mathcal{E} \cap \{n_i \geq J \}) \\
        &\leq \mathbb{P}(\mathcal{E} \cap \{n_i < J\} + \mathbb{P}(n_i \geq J)) \\
        &=: \mathfrak{p}_{21} + \mathfrak{p}_{22}. 
    \end{aligned}
\end{equation*}
Note that in the analysis of $\mathfrak{p}_{21}$, we can safely ignore the boundary case of $n_i = 0$, since $\mathbb{P}(\mathcal{E} \cap \{n_i = 0\}) = 0$. By a standard Chernoff bound, we have by choosing $J = \lfloor (1+\theta)\mu_i\rfloor + 1$, 
\begin{equation*}
    \mathbb{P}\left(n_i \geq (1+\theta)\mu_i \right) \leq \exp\left(-\frac{\theta^2 \mu_i}{2(1 + \theta/3)} \right).
\end{equation*}

\textit{Step 2: Localizing away from the escape event.}

Define the hitting times recursively as 
\begin{equation}\label{eq:hit-times}
    \begin{aligned}
        \tau_1 &:= \inf\left\{t \geq 0: s(t+1) = i \right\} \\
        \tau_m &:= \inf\left\{t > \tau_{m-1}:s(t+1) = i\right\} \quad \text{for } m \geq 2.
    \end{aligned} 
\end{equation}
In view of \eqref{eq:hit-times}, we have $\mathcal{I}_i = \left\{\tau_1, \dots, \tau_{n_i} \right\}$. For a deterministic integer $m \geq 1$, define the truncated partial sum
\begin{equation*}
    \overline M_m( f) := \sum_{\ell=1}^m  4\langle \varepsilon_{\tau_{\ell} + 1}, f(X_{\tau_{\ell}})\rangle_2 \mathbf{1}\{X_{\tau_{\ell}} \in \mathcal{B}_T \}.
\end{equation*}
Denote $\Delta_k f := \pi_k(f) - \pi_{k-1}(f)$. For some $R_T > 0$, let 
\begin{equation*}
    \mathcal{S}_T := \left\{\max_{0 \leq t \leq T-1} \|X_t\|_2 \leq R_T \right\}.
\end{equation*}
Because $n_i = \sum_{t=0}^{T-1} \mathbf{1}\{s(t+1) = i\}$, i.e., the number of visits to state $i$ before time $T$, every hitting time $\tau_1, \dots, \tau_{n_i}$ must occur before time $T$. Thus, on the event $\mathcal{S}_T$, we have $X_{\tau_1}, \dots, X_{\tau_{n_i}} \in \mathcal{B}_T$. This implies that $\mathbf{1}\{X_{\tau_{\ell}} \in \mathcal{B}_T \} =1 $ for all $\ell \leq n_i$, and we have 
\begin{equation*}
    M_{n_i}(\Delta_k f) = \overline M_{n_i}(\Delta_k f).
\end{equation*}

On the event $\{n_i < J\}$, there exists an integer $m \in \{1, \dots, J-1 \}$ such that $n_i = m$. Thus, we have on $\{n_i < J \} \cap \mathcal{S}_T$,
\begin{equation*}
    M_{n_i}(\Delta_k f) \leq \max_{1 \leq m \leq J-1} \overline M_m(\Delta_k f).
\end{equation*}

Decomposing $\mathfrak{p}_{21}$ by further intersecting with the event $\mathcal{S}_T$, we have 
\begin{equation*}
    \begin{aligned}
        \mathfrak{p}_{21} &= \mathbb{P}\left(\mathcal{E} \cap \{n_i < J\} \cap \mathcal{S}_T \right) + \mathbb{P}\left(\mathcal{E} \cap \{n_i < J\} \cap \mathcal{S}_T^c \right) \\
        &\leq \mathbb{P}\left(\mathcal{E} \cap \{n_i < J\} \cap \mathcal{S}_T \right) + \mathbb{P}\left(\mathcal{S}_T^c\right) \\
        &=: \mathfrak{p}_{211} + \mathfrak{p}_{212},
    \end{aligned}
\end{equation*}
where
\begin{equation*}
    \begin{aligned}
        \mathfrak{p}_{211} &\leq \mathbb{P}\left(\sup_{f \in \mathcal{F}_i^*} \sum_{k=\underline K + 1}^{\overline K}\max_{1 \leq m \leq J-1}\overline M_m(\Delta_k f) > \sum_{k=\underline K + 1}^{\overline K} u_k  \right),
    \end{aligned}
\end{equation*}
and $\mathfrak{p}_{212}$ will be controlled by Lemma \ref{lem:traj-ball}.

\textit{Step 3: Bounding the localized  martingale.}

The event inclusion
\begin{equation*}
     \left\{\sup_{f \in \mathcal{F}_i^*} \sum_{k=\underline K + 1}^{\overline K} \max_{1 \leq m \leq J-1}\overline M_m(\Delta_k f) > \sum_{k=\underline K + 1}^{\overline K} u_k \right\} \subseteq \bigcup_{k=\underline K + 1}^{\overline K} \left\{\sup_{f \in \mathcal{F}_i^*} \max_{1 \leq m \leq J-1}\overline M_m(\Delta_k f) > u_k \right\},
\end{equation*}
holds, because if each of summands exceeds $u_k$, then the entire sum exceeds the sum of $u_k$ thresholds. By the union bound, we have
\begin{equation*}
\begin{aligned}
    \mathfrak{p}_{211} 
    &\leq \sum_{k=\underline K + 1}^{\overline K}\mathbb{P}\left(\sup_{f \in \mathcal{F}_i^*} \max_{1 \leq m \leq J-1} \overline  M_m\left(\Delta_k f \right) > u_k  \right) \\
    &\leq \sum_{k=\underline K + 1}^{\overline K} \mathcal{N}^2(\mathcal{F}_i^*, 2^{-k}, \|.\|_{\infty, \mathcal{B}_T})\max_{(\pi_k, \pi_{k-1}) \in \Pi_k}\mathbb{P}\left(\max_{1 \leq m \leq J-1} \overline M_m(\Delta_k f) > u_k  \right).
\end{aligned}
\end{equation*}

Let $\mathcal{F}_{\tau_m} := \left\{A \in \mathcal{A}: \forall t, A \cap \{\tau_m \leq t \} \in \mathcal{F}_t \right\}$ be the filtration up to hitting time $\tau_m$. Each of the summands in $\overline M_m$ is a martingale difference sequence, since for $1 \leq \ell \leq m$, we have 
\begin{equation*}
    \begin{aligned}
\mathbb{E}\left[\left\langle\varepsilon_{\tau_{\ell }+ 1}, \Delta_kf(X_{\tau_{\ell}}) \right\rangle_2 \mathbf{1}\left\{X_{\tau_{\ell}} \in \mathcal{B}_T \right\} \mid \mathcal{F}_{\tau_{\ell}} \right] 
 &= \langle \mathbb{E}\left[\varepsilon_{\tau_{\ell} +1} \mid \mathcal{F}_{\tau_{\ell}} \right], \Delta_k f(X_{\tau_{\ell}}) \rangle_2 \mathbf{1}\left\{X_{\tau_{\ell}} \in \mathcal{B}_T \right\} \\
&= \langle 0, \Delta_k f(X_{\tau_{\ell}}) \rangle_2 \mathbf{1}\left\{X_{\tau_{\ell}} \in \mathcal{B}_T \right\} \\
&= 0,
    \end{aligned}
\end{equation*}
where we used the fact that $X_{\tau_{\ell}}$ is $\mathcal{F}_{\tau_{\ell}}$-measurable. Moreover, the variance for each summand can be bounded for $1 \leq \ell \leq m$ as
\begin{equation*}
\begin{aligned}
 \mathbb{V}_{\ell} &:=\mathbb{E}\left[\left\langle\varepsilon_{\tau_{\ell }+ 1}, \Delta_kf(X_{\tau_{\ell}}) \right\rangle_2^2 \mathbf{1}\left\{X_{\tau_{\ell}} \in \mathcal{B}_T \right\} \mid \mathcal{F}_{\tau_{\ell}} \right] \\ &= [\Delta_k f(X_{\tau_{\ell}})]^\top \mathbb{E}[\varepsilon_{\tau_{\ell} + 1} \varepsilon_{\tau_{\ell} + 1}^{\top} \mid \mathcal{F}_{\tau_{\ell}}]\Delta_k f(X_{\tau_{\ell}}) \mathbf{1}\left\{X_{\tau_{\ell}} \in \mathcal{B}_T \right\} \\
&= \|\Delta_kf(X_{\tau_{\ell}})\|_2^2 \mathbf{1}\left\{X_{\tau_{\ell}} \in \mathcal{B}_T \right\} \\
&\leq 2\|\pi_k(f) - f\|^2_{\infty, \mathcal{B}_T} + 2\|\pi_{k-1}(f) - f\|^2_{\infty, \mathcal{B}_T} \\
&\leq C_12^{-2k}.
\end{aligned}
\end{equation*}

To control the $p$-th moment, define the conditional sub-gaussian norm of a random variable $X_n$ as
\begin{equation*}
    \|X_n\|_{\psi_2 \mid \mathcal{F}_{n-1}} := \inf\left\{t > 0:\mathbb{E}\left[\exp\left(\frac{X_n^2}{t^2}  \right) \mid \mathcal{F}_{n-1} \right] \leq 2 \right\}.
\end{equation*}
Let $v = \Delta_k f(X_{\tau_{\ell}})\mathbf{1}\{X_{\tau_{\ell}} \in \mathcal{B}_T \}$. By Claim \ref{cl:eps-stop-inde}, we have 
\begin{equation*}
    \begin{aligned}
        \|\langle \varepsilon_{\tau_{\ell}+1}, \Delta_k f(X_{{\tau_{\ell} }})\rangle_2 \mathbf{1}\{X_{\tau_{\ell}} \in \mathcal{B}_T \}\|_{\psi_2 \mid \mathcal{F}_{\tau_{\ell}}} &= \| v \|_2 \left\| \left\langle\varepsilon_{\tau_{\ell} + 1}, \frac{v}{\|v\|_2} \right\rangle_2 \right\|_{\psi_2 \mid \mathcal{F}_{\tau_{\ell}}} \\
        &\leq \|v\|_2 \sup_{u \in \mathbb{S}^{d-1}} \left\| \left\langle\varepsilon_{\tau_{{\ell} + 1} }, u\right\rangle_2 \right\|_{\psi_2 \mid \mathcal{F}_{\tau_{\ell}}} \\
        &= \|v\|_2 \|\varepsilon_{\tau_{\ell}+1}\|_{\psi_2\mid \mathcal{F}_{\tau_{\ell}}} \\
        &= \|v\|_2 \|\varepsilon\|_{\psi_2} \\
        &= \kappa\|v\|_2 .
        \end{aligned}
\end{equation*}
Using the triangle inequality,
\begin{equation*}
    \begin{aligned}
        \|v\|_2 &= \|\Delta_k f(X_{\tau_{\ell}})\mathbf{1}\{X_{\tau_{\ell}} \in \mathcal{B}_T \} \|_2 \\
        &\leq \|\pi_k(f) - f + f - \pi_{k-1}(f)\|_2\mathbf{1}\{X_{\tau_{\ell}} \in \mathcal{B}_T \} \\
        &\leq 2^{3-k}, \quad \text{a.s.},
    \end{aligned}
\end{equation*}
which implies
\begin{equation*}
    \|\langle \varepsilon_{\tau_{\ell}+1}, \Delta_k f(X_{{\tau_{\ell} }})\rangle_2 \mathbf{1}\{X_{\tau_{\ell}} \in \mathcal{B}_T \}\|_{\psi_2 \mid \mathcal{F}_{\tau_{\ell}}} \leq 2^{3 - k}\kappa.
\end{equation*}
Using the $L^p$ characterization of conditionally sub-gaussian random variables, we have by Stirling's formula \cite{boucheron2013}, 
\begin{equation}\label{eq:bernstein-cond-marti}
    \begin{aligned}
\mathbb{E}\left[\left|\left\langle\varepsilon_{\tau_{\ell }+ 1}, \Delta_kf(X_{\tau_{\ell}}) \right\rangle_2 \right|^p \mathbf{1}\left\{X_{\tau_{\ell}} \in \mathcal{B}_T \right\} \mid \mathcal{F}_{\tau_{\ell}} \right] &\leq C_2^p \|\langle \varepsilon_{\tau_{\ell} + 1}, \Delta_k f(X_{\tau_{\ell}}) \|_2 \mathbf{1}\{X_{\tau_{\ell}} \in \mathcal{B}_T \}\|^p_{\psi_2\mid \mathcal{F}_{\tau_{\ell}}} p^{p/2} \\
&\leq (C_2 \kappa 2^{-k})^2 (C_2 \kappa 2^{-k})^{p-2}p^{p/2} \\
&\leq \frac{p!}{2}( \underbrace{C_2 \kappa 2^{-k}}_{\text{variance term}})^2(\underbrace{C_2 \kappa 2^{-k}}_{\text{constant term}})^{p-2},
    \end{aligned}
\end{equation}
for all $p \geq 2$, so the Bernstein's condition \cite{boucheron2013} is satisfied. By choosing a sufficently large absolute constant $C \geq \max\{\sqrt{C_1}, C_2 \kappa\}$, define the parameter $c_k := C2^{-k}$. Since $\mathbb{V}_{\ell}$ is positive, $\sum_{\ell=1}^m \mathbb{V}_{\ell}$ is increasing with $m$, and the maximum over $\{1, \dots, J-1 \}$ is attained with $J-1$ summands. Thus, we have 
\begin{equation*}
    \begin{aligned}
       \max_{1 \leq m \leq J-1}\sum_{\ell=1}^m \mathbb{V}_{\ell} =
        \sum_{\ell=1}^{J-1} \mathbb{V}_{\ell} \leq (J-1)c_k^2, \quad \text{a.s.}
    \end{aligned}
\end{equation*}
Moreover, we have 
\begin{equation*}
    \frac{p!}{2}(C_2 \kappa 2^{-k})^2(C_2 \kappa 2^{-k})^{p-2} \leq \frac{p!}{2}c_k^2 c_k^{p-2},
\end{equation*}
so the parameter $c_k$ helps to both upper bound the variance while simulatenously acting as a valid scale parameter for the Bernstein condition.

Since $\overline M_m(\Delta_k f)$ is a martingale difference sequence, by applying Proposition \ref{prop:unif-pena} in Appendix \ref{app:appendix-prob}, we have $\forall (\pi_k, \pi_{k-1}) \in \Pi_k$
\begin{equation*}
\begin{aligned}
    \mathbb{P}\left(\max_{1 \leq m \leq J-1}\overline M_m(\Delta_k f) > u_k \right) &=
    \mathbb{P}\left(\left\{\max_{1 \leq m \leq J-1}\overline M_m(\Delta_k f) \geq u_k\right\} \cap \left\{\sum_{\ell=1}^{J-1}\mathbb{V}_{\ell} \leq (J-1)c_k^2\right\} \right) \\
    &\leq \exp\left(-\frac{u_k^2}{2((J-1)c_k^2  + c_k u_k)  } \right),
\end{aligned}
\end{equation*}
which implies that by choosing $J = \lfloor (1+\theta)\mu_i\rfloor + 1$, we have for any $\mu_i \geq 1$, 
\begin{equation*}
    \mathfrak{p}_{211} \leq \sum_{k=\underline K + 1}^{\overline K} \mathcal{N}^2(\mathcal{F}_i^*, 2^{-k}, \|.\|_{\infty, \mathcal{B}_T})\exp\left(-\frac{u_k^2}{2( (2+\theta)\mu_i c_k^2  + c_k u_k)  } \right).
\end{equation*}

\textit{Step 4: Threshold simplification.
}
Choose $u_k$ such that 
\begin{equation*}
    \mathcal{N}^2(\mathcal{F}_i^*, 2^{-k}, \|.\|_{\infty, \mathcal{B}_T})\exp\left(-\frac{u_k^2}{2( (2+\theta)\mu_i c_k^2  + c_k u_k)  } \right) \leq \eta_k := 2^{-k}\eta,
\end{equation*}
which is equivalent to
\begin{equation*}
\begin{aligned}
   \frac{u_k^2}{2( (2+\theta)\mu_i c_k^2  + c_k u_k)  } &\geq \log\left(\frac{\mathcal{N}^2(\mathcal{F}_i^*, 2^{-k}, \|.\|_{\infty, \mathcal{B}_T})}{\eta_k}\right) \\
   &= 2\log\mathcal{N}(\mathcal{F}_i^*, 2^{-k}, \|.\|_{\infty, \mathcal{B}_T}) + k\log(2) + \log\left(\frac{1}{\eta} \right) \\
   &=: L_k.
\end{aligned}
\end{equation*}
Applying the quadratic formula, we obtain the zero solution as
\begin{equation*}
    u_k^* = c_k\left[L_k + \sqrt{L_k^2 + 2(2+\theta)\mu_i L_k} \right].
\end{equation*}
Note that since we have an upward opening parabola, choosing a $u_k \geq u_k^*$ is valid, since it still satisfies the quadratic inequality. Using the bound $\sqrt{a + b} \leq \sqrt{a} + \sqrt{b}$,
\begin{equation}\label{eq:uk-star}
\begin{aligned}
   u_k^* &\leq  c_k \left\{\sqrt{2(2+\theta)\mu_iL_k} + 2L_k \right\} \\
   &\leq C2^{-k}\sqrt{(2+\theta)\mu_i \log\mathcal{N}(\mathcal{F}_i^*, 2^{-k}, \|.\|_{\infty, \mathcal{B}_T})} + C2^{-k}\sqrt{(2+\theta)\mu_i k} \\
    &+ C2^{-k}\sqrt{(2+\theta)\mu_i\log\left(\frac{1}{\eta} \right)} + C2^{-k}\log\mathcal{N}(\mathcal{F}_i^*, 2^{-k}, \|.\|_{\infty, \mathcal{B}_T}) \\
    &+ C2^{-k}k + C2^{-k}\log\left(\frac{1}{\eta}\right) \\
    &\leq C\sqrt{(2+\theta)\mu_i}\int_{2^{-(k+1)}}^{2^{-k}} \sqrt{\log\mathcal{N}(\mathcal{F}_i^*, \epsilon, \|.\|_{\infty, \mathcal{B}_T})} d\epsilon + C\sqrt{(2+\theta)\mu_i}2^{-k}\sqrt{k} \\
    &+ C\sqrt{(2+\theta)\mu_i\log\left(\frac{1}{\eta} \right)}2^{-k} + C\int_{2^{-(k+1)}}^{2^{-k}} \log\mathcal{N}(\mathcal{F}_i^*, \epsilon, \|.\|_{\infty, \mathcal{B}_T}) d\epsilon \\
    &+ C2^{-k}k + C\log\left(\frac{1}{\eta}\right)2^{-k}. 
\end{aligned}
\end{equation}
Choosing $u_k$ as in \eqref{eq:uk-star}, and 
summing over $k = \underline K + 1, \dots, \overline K$, we obtain 
\begin{equation*}
    \begin{aligned}
       \sum_{k=\underline K + 1}^{\overline K} u_k &= C\sqrt{(2+\theta)\mu_i}\left\{\int_{2^{-(\overline K + 1)}}^{2^{-(\underline K + 1)}} \sqrt{\log\mathcal{N}(\mathcal{F}_i^*, \epsilon, \|.\|_{\infty, \mathcal{B}_T})} d\epsilon + 1 + \log\left(\frac{1}{\eta} \right) \right\}
       \\
       &+ C\left\{\int_{2^{-(\overline K + 1)}}^{2^{-(\underline K + 1)}} \log\mathcal{N}(\mathcal{F}_i^*, \epsilon, \|.\|_{\infty, \mathcal{B}_T}) d\epsilon + 1 + \log\left(\frac{1}{\eta} \right)\right\}.
    \end{aligned}
\end{equation*}

\textit{Step 5: Union bound over the good events.}
By Lemma \ref{lem:traj-ball} in Appendix \ref{app:appendix-dynamics}, and setting a probability threshold of at most $\eta/3$ to the escape event and the visit probability respectively, we obtain the conditions
\begin{equation}\label{eq:intersect-gd-events}
    R_T > \sqrt{\frac{1}{\alpha}\log\left(\frac{3TC_{\rho, B}}{\eta} \right)} \quad \text{and} \quad  \mu_i \geq \log\left(\frac{3}{\eta} \right)\frac{2(1 + \theta/3)}{\theta^2},
\end{equation}
where the first condition in the last display ensures that $\mathbb{P}(S_T^c) \leq \eta/3$. By the union bound, when \eqref{eq:intersect-gd-events} are satisfied, we have with probability at least $1 - \eta$,
\begin{equation*}
    \begin{aligned}
        A_2 &\leq C\sqrt{(2+\theta)\mu_i}\left\{\int_{2^{-(\overline K + 1)}}^{2^{-(\underline K + 1)}} \sqrt{\log\mathcal{N}(\mathcal{F}_i^*, \epsilon, \|.\|_{\infty, \mathcal{B}_T})} d\epsilon + 1 + \log\left(\frac{1}{\eta} \right) \right\}
       \\
       &+ C\left\{\int_{2^{-(\overline K + 1)}}^{2^{-(\underline K + 1)}} \log\mathcal{N}(\mathcal{F}_i^*, \epsilon, \|.\|_{\infty, \mathcal{B}_T}) d\epsilon + 1 + \log\left(\frac{1}{\eta} \right)\right\}.
    \end{aligned}
\end{equation*}
The proof is now complete.
\subsection{Proof of Lemma \ref{lem:final}} \label{subsec:proof_final_martoffset}
Let $\Delta_{t}(f) := 4\langle \varepsilon_{t+1}, f(X_t)\rangle_2 - (\|f(X_t) \|_2^2/2)$. Following similar arguments as Lemma \ref{lem:chaining}, we have 
\begin{equation*}
    \begin{aligned}
        \mathbb{P}\left(A_3 > w \right) &= \mathbb{P}\left(\sup_{f \in \mathcal{F}_i^*} \sum_{t \in \mathcal{I}_i}\Delta_{t}(\pi_{\underline K}(f) ) > w \right) \\
        &\leq \mathbb{P}\left(\left\{\sup_{f \in \mathcal{F}_i^*} \sum_{t \in \mathcal{I}_i} \Delta_{t}(\pi_{\underline K}(f) ) > w \right\} \cap \left\{n_i < J \right\} \right) + \mathbb{P}\left(n_i \geq J\right) \\
        &=: \mathfrak{p}_{31} + \mathfrak{p}_{32}. 
    \end{aligned}
\end{equation*}

Intersecting with the event $\mathcal{S}_T$, we have 
\begin{equation*}
    \begin{aligned}
        \mathfrak{p}_{31} &\leq \mathbb{P}\left(\left\{\sup_{f \in \mathcal{F}_i^*} \sum_{t \in \mathcal{I}_i} \Delta_{t}(\pi_{\underline K}(f)) > w\right\} \cap \left\{n_i < J  \right\} \cap \mathcal{S}_T \right) + \mathbb{P}\left(\mathcal{S}_T^c \right) \\
        &=: \mathfrak{p}_{311} + \mathfrak{p}_{312}. 
    \end{aligned}
\end{equation*}

Define the truncated processes
\begin{equation*}
    \overline \Delta_{\ell}(f) := \left[4\langle \varepsilon_{\tau_{\ell} + 1}, f(X_{\tau_{\ell}})\rangle_2 -  
    \frac{\|f(X_{\tau_{\ell}})\|_2^2}{2}
     \right]\mathbf{1}\left\{X_{\tau_{\ell}} \in \mathcal{B}_T \right\},
\end{equation*}
and
\begin{equation*}
\begin{aligned}
        \overline S_m(f) &:= \sum_{\ell=1}^m \overline \Delta_{\ell}(f).
\end{aligned}
\end{equation*}
Again, we can ignore the boundary term $n_i = 0$, since $\mathbb{P}(\{A_3 > w \} \cap \{n_i = 0\} \cap \mathcal{S}_T) = 0$. On the event $\{n_i < J\} \cap \mathcal{S}_T$, $n_i$ is an integer from $\{1, \dots J-1\}$. Moreover, all the indicators in the truncated process evaluate to 1. Thus, we have on $\{n_i < J\} \cap \mathcal{S}_T$ that
\begin{equation*}
    \begin{aligned}
        \sup_{f \in \mathcal{F}_i^*}\sum_{t \in \mathcal{I}_i}\Delta_t(\pi_{\underline K}(f) ) &\leq \sup_{f \in \mathcal{F}_i^*}\max_{1 \leq m \leq J-1}\sum_{\ell=1}^m \overline \Delta_{\ell}(\pi_{\underline K}(f)) \\
        &= \sup_{f \in \mathcal{F}_i^*}\max_{1 \leq m \leq J-1}\overline S_m(\pi_{\underline K}(f)).
    \end{aligned}
\end{equation*}
By the union bound, we have 
\begin{equation*}
    \begin{aligned}
        \mathfrak{p}_{311} &\leq \mathbb{P}\left(\sup_{f \in \mathcal{F}_i^*}\max_{1 \leq m \leq J-1}\overline S_m(\pi_{\underline K}(f) ) > w \right) \\
        &\leq \mathcal{N}\left(\mathcal{F}_i^*, 2^{-\underline K}, \|.\|_{\infty, \mathcal{B}_T} \right)\mathbb{P}\left(\max_{1 \leq m \leq J-1}\overline S_m(\pi_{\underline K}(f) ) > w  \right).
    \end{aligned}
\end{equation*}
Since the maximum of the exponential is equals to the exponential of the maximum, we have for any $\xi > 0$
\begin{equation*}
    \begin{aligned}
        \mathbb{P}\left(\max_{1 \leq m \leq J-1}\overline S_m(\pi_{\underline K}(f) ) > w  \right) &= \mathbb{P}\left(\max_{1 \leq m \leq J-1}\sum_{\ell=1}^m \overline \Delta_{\ell}(\pi_{\underline K}(f)) > w \right) \\
        &= \mathbb{P}\left(\exp\left\{\max_{1 \leq m \leq J-1}\xi \sum_{\ell=1}^m \overline \Delta_{\ell}(\pi_{\underline K}(f)) \right\} > \exp\left\{\xi w \right\} \right) \\
        &= \mathbb{P}\left(\max_{1 \leq m \leq J-1}\exp\left\{\xi \sum_{\ell=1}^m \overline \Delta_{\ell}(\pi_{\underline K}(f)) \right\} > \exp\left\{\xi w\right\} \right).
    \end{aligned}
\end{equation*}
We will apply Ville's inequality to the probability in the last display (recalled in Theorem \ref{thm:ville} of Appendix \ref{app:appendix-prob}). In order to do so, we need to show that the exponential term on the LHS is a non-negative supermartingale. Since $X_{\tau_{\ell}}$ is $\mathcal{F}_{\tau_{\ell}}$-measurable, we have 
\begin{equation}\label{eq:xi-supermart-cond}
    \begin{aligned}
        \mathbb{E}\left[\exp\left\{\xi \overline \Delta_{\ell}(\pi_{\underline K}(f))\right\}  \mid \mathcal{F}_{\tau_{\ell}} \right]  &= \exp\left\{-\frac{\xi}{2} \left\|\pi_{\underline K}(f)(X_{\tau_{\ell}}) \right\|_2^2 \mathbf{1}\left\{X_{\tau_{\ell}} \in \mathcal{B}_T \right\} \right\} \\ &\times \mathbb{E}\left[\exp\left\{4\xi \langle \varepsilon_{\tau_{\ell} + 1}, \pi_{\underline K}(f)(X_{\tau_{\ell}} \rangle_2 \mathbf{1}\left\{X_{\tau_{\ell}} \in \mathcal{B}_T \right\} \right\} \mid \mathcal{F}_{\tau_{\ell}} \right] \\
        &\leq \exp\left\{-\frac{\xi}{2} \left\|\pi_{\underline K}(f)(X_{\tau_{\ell}}) \right\|_2^2 \mathbf{1}\left\{X_{\tau_{\ell}} \in \mathcal{B}_T \right\} \right\} \\
        &\times \exp\left(8\xi^2 \sigma^2 \|\pi_{\underline K}(f)(X_{\tau_{\ell}}) \|_2^2 \mathbf{1}\left\{X_{\tau_{\ell}} \in \mathcal{B}_T \right\} \right) \\
        &= \exp\left\{\|\pi_{\underline K}(f)(X_{\tau_{\ell}}) \|_2^2 \mathbf{1}\left\{X_{\tau_{\ell}} \in \mathcal{B}_T \right\} \left[8\xi^2\sigma^2 - \frac{\xi}{2} \right]\right\} \\
        &\leq 1, \qquad \forall \xi \leq (16\sigma^2)^{-1},
    \end{aligned}
\end{equation}
using the property that $\varepsilon$ is a $\sigma^2$ sub-gaussian random vector. This implies the supermartingale property, since 
\begin{equation*}
    \begin{aligned}
        \mathbb{E}\left[\exp\left\{\xi \sum_{\ell=1}^m \overline \Delta_{\ell}(\pi_{\underline K}(f)) \right\} \mid \mathcal{F}_{\tau_{m}} \right] &= \exp\left\{\xi \sum_{\ell=1}^{m-1} \overline \Delta_{\ell}(\pi_{\underline K}(f)) \right\}\mathbb{E}\left[\exp\left\{\xi \Delta_m(\pi_{\underline K}(f)) \right\} \mid \mathcal{F}_{\tau_m} \right] \\
        &\leq \exp\left\{\xi \sum_{\ell=1}^{m-1} \overline \Delta_{\ell}(\pi_{\underline K}(f)) \right\}.
    \end{aligned}
\end{equation*}
Applying Ville's inequality, we have by \eqref{eq:xi-supermart-cond} 
\begin{equation*}
    \begin{aligned}
        \mathbb{P}\left(\max_{1 \leq m \leq J-1}\exp\left\{\xi \sum_{\ell=1}^m \overline \Delta_{\ell}(\pi_{\underline K}(f)) \right\} > \exp\left\{\xi w\right\} \right) &\leq \frac{\mathbb{E}\left[\exp\left\{\xi \overline \Delta_1(\pi_{\underline K}(f) ) \right\}\right]}{\exp\left\{\xi w \right\}} \\
        &\leq  \exp\left(-\xi w\right).
    \end{aligned}
\end{equation*}
Choosing $\xi = (16\sigma^2)^{-1}$, we obtain
\begin{equation*}
    \mathfrak{p}_{311} \leq \mathcal{N}(\mathcal{F}_i^*, 2^{-\underline K}, \|.\|_{\infty, \mathcal{B}_T})\exp\left(-\frac{w}{16\sigma^2} \right).
\end{equation*}

Choosing $J = \lfloor (1+\theta)\mu_i\rfloor + 1$, and assigning the probability of at most $\eta/3$ to $\mathfrak{p}_{32}$ and $\mathfrak{p}_{312}$ respectively, we obtain under the same conditions as in \eqref{eq:intersect-gd-events}, with probability at least $1 - \eta$, 
\begin{equation*}
    A_3 \leq 16\sigma^2\left[\log\mathcal{N}(\mathcal{F}_i^*, 2^{-\underline K}, \|.\|_{\infty, \mathcal{B}_T}) + \log\left(\frac{3}{\eta} \right) \right].
\end{equation*}
\subsection{Proof of Lemma \ref{lem:coarse_martoffset}} \label{subsec:proof_coarse_martoffset}

Let us denote
\begin{equation*}
    \mathcal{E}_4 := \left\{\sup_{f \in \mathcal{F}_i^*} \sum_{t \in \mathcal{I}_i} \left\|\pi_{\underline K}(f)(X_t) - f(X_t) \right\|_2^2 >w \right\}.
\end{equation*}
Again, we have 
\begin{equation*}
    \begin{aligned}
        \mathbb{P}(\mathcal{E}_4) &\leq \mathbb{P}(\mathcal{E}_4 \cap \mathcal{S}_T ) + \mathbb{P}(\mathcal{S}_T^c) \\
        &=: \mathfrak{p}_{41} + \mathfrak{p}_{42}. 
    \end{aligned}
\end{equation*}
On the event $\mathcal{S}_T$, we have
\begin{equation*}
    \begin{aligned}
     \sup_{f \in \mathcal{F}_i^*} \sum_{t \in \mathcal{I}_i} \|\pi_{\underline K}(f)(X_t) - f(X_t)\|_2^2 &\leq 
     \sup_{f \in \mathcal{F}_i^*}\sum_{t \in \mathcal{I}_i} \|\pi_{\underline K}(f)(X_t) - f(X_t)\|_{\infty, \mathcal{B}_T}^2 \\
        &\leq n_i 2^{-2\underline K}.
    \end{aligned}
\end{equation*}
Thus, we have by the standard Chernoff bounds for binomial random variables
\begin{equation*}
    \begin{aligned}
        \mathfrak{p}_{41} &\leq \mathbb{P}(n_i 2^{-2\underline K} > w) \\
        &\leq \mathbb{P}(\left\{n_i 2^{-2\underline K} > w \right\} \cap \left\{n_i \leq (1+\theta)\mu_i \right\}) + \mathbb{P}(n_i > (1+\theta)\mu_i) \\
        &\leq \mathbf{1}\left\{(1+\theta)\mu_i 2^{-2\underline K} > w \right\} + \exp\left(-\frac{\theta^2 \mu_i}{2(1 + \theta/3)} \right).
    \end{aligned}
\end{equation*}
By assigning a maximum probability of $\eta/2$ to the binomial term and the escape event respectively, and setting $w = (1+\theta)\mu_i 2^{-2\underline K}$, we have under the conditions 
\begin{equation*}
    R_T > \sqrt{\frac{1}{\alpha}\log\left(\frac{2TC_{\rho, B}}{\eta} \right)} \quad \text{and} \quad \mu_i \geq \log\left(\frac{2}{\eta} \right)\frac{2(1 + \theta/3)}{\theta^2},
\end{equation*}
it holds with probability at least $1 - \eta$ that
\begin{equation*}
    A_4 \leq (1+\theta)\mu_i 2^{-2\underline K} 
\end{equation*}

\section{Proof of Proposition \ref{prop: ulln-master}} \label{sec:proof_prop_ulln}
%
%
%
In what follows, we first restrict ourselves to the sphere, before closing the gap by generalizing to every function outside the sphere.

\textit{Step 1: Reduction to an additive gap with truncation}. 

Denote $\gapf := \ltrans(f) - \lemp(f)$. On the sphere, we have $\ltrans(f) = p_i \|f\|^2_{L^2( \mu)} = p_i r^2$. We will prove the equivalent statement
\begin{equation*}
    \mathbb{P}\left(\sup_{f \in \mathcal{F}^*_{i, r}} G_{T,i}(f) \leq \frac{p_i r^2}{2}  \right) \geq 1 - \delta,
\end{equation*}
since 
\begin{equation*}
\begin{aligned}
    \mathbb{P}\left(\forall f \in \mathcal{F}^*_{i, r}: \lemp(f) > \frac{\ltrans(f)}{2} \right) &= \mathbb{P}\left(\bigcap_{f \in \mathcal{F}^*_{i,r}} \left\{\lemp(f) >  \frac{\ltrans(f)}{2} \right\} \right) \\
    &= \mathbb{P}\left(\bigcap_{f \in \mathcal{F}^*_{i,r}} \left\{\lemp(f) > \frac{p_i r^2}{2} \right\} \right) \\
    &= \mathbb{P}\left(\inf_{f \in \mathcal{F}^*_{i, r}} \lemp(f) > \frac{p_i r^2}{2} \right) \\
    &= \mathbb{P}\left(\inf_{f \in \mathcal{F}^*_{i, r}} \left[\lemp(f) - \ltrans(f) + \ltrans(f) \right] > \frac{p_i r^2}{2} \right) \\
    &= \mathbb{P}\left(\sup_{f \in \mathcal{F}^*_{i, r}} \left[\ltrans(f) - \lemp(f)\right] \leq \frac{p_i r^2}{2}  \right) \\
    &= \mathbb{P}\left(\sup_{f \in \mathcal{F}^*_{i, r}} G_{T,i}(f) \leq \frac{p_i r^2}{2}  \right).
\end{aligned}
\end{equation*}

Let 
\begin{equation*}
    \mathcal{E} := \left\{\sup_{f \in \mathcal{F}_{i,r}^* } \gapf > \frac{p_i r^2}{2} \right\},
\end{equation*}
and for some $R_T > 0$ recall the safe event
\begin{equation*}
    \mathcal{S}_T = \left\{\max_{0 \leq t \leq T-1} \|X_t\|_2 \leq R_T \right\}.
\end{equation*}
Intersecting with $\mathcal{E}$ with $\mathcal{S}_T$, we have 
\begin{equation*}
    \begin{aligned}
        \mathbb{P}(\mathcal{E}) \leq \mathbb{P}(\mathcal{E} \cap \mathcal{S}_T) + \mathbb{P}(\mathcal{S}_T^c).
    \end{aligned}
\end{equation*}

Recall that on the event $\mathcal{S}_T$, we have $\|f(X_t)\|_2 \leq M_T$ for all $t = 0, \dots, T-1$, which implies that $f(X_t) = \overline f(X_t)$, since no truncation occurs, and we can readily interchange $f$ with $\overline f$ in the empirical loss. Then intercalating by $\ltrans(\overline f)$, we have
\begin{equation*}
    \begin{aligned}
        \gapf &= \left[\ltrans(\overline f) - \lemp(\overline f) \right] +
        \left[\ltrans(f) - \ltrans(\overline f) \right]
       \\&= \gapof + \left[\ltrans(f) - \ltrans(\overline f)\right] \\
       &= \gapof + \left[\ltrans(f) - \ltrans(T_{M_T}[f])\right].
    \end{aligned}
\end{equation*}
Thus, we have on $\mathcal{S}_T$
\begin{equation*}
    \sup_{f \in \mathcal{F}_{i,r}^* } \gapf \leq \sup_{\overline f \in \overline{\mathcal{F}}_r} \gapof + \sup_{f \in \mathcal{F}_{i,r}^* }\left[\ltrans(f) - \ltrans(T_M[f])\right].
\end{equation*}
For an event $A$, denote $\PS(A) := \mathbb{P}(A \cap \mathcal{S}_T)$. By the last display, we have 
\begin{equation}\label{eq:step1_ullnproof_prop}
    \begin{aligned}
        \mathbb{P}(\mathcal{E} \cap \mathcal{S}_T) \leq \PS\left(\sup_{\overline f \in \overline{\mathcal{F}}_r}  G_{T,i}(\overline f) > \frac{p_i r^2}{4} \right) + \PS \left(\sup_{f \in \mathcal{F}_{i,r}^*}\left\{\ltrans(f) - \ltrans(T_{M_T}[f]) \right\} > \frac{p_i r^2}{4} \right). 
    \end{aligned}
\end{equation}

\textit{Step 2: Chaining setup.}

Let $\{\pi_0(\overline f), \dots, \pi_{\overline K}(\overline f)\}$ be a cover of $\overline{\mathcal{F}}_r$ with respect to the $\|.\|_{\infty, \mathcal{B}_T}$ norm, where for every $\overline{f} \in \overline{\mathcal{F}}_r$, there exists an element $\pi_k$ such that such that $\|\pi_k(\overline f) - \overline f \|_{\infty, \mathcal{B}_T} \leq \epsilon_k =r2^{-k}$. Using a telescoping sum argument, we have 
\begin{equation*}
    \begin{aligned}
        \gap(\overline f) = \gap(\pi_0 (\overline f)) + \sum_{k=1}^{\overline K}\gap(\pi_k(\overline f)) - \gap(\pi_{k-1}(\overline f)) + \gap(\overline f) - \gap(\pi_{\overline K}(\overline f))
    \end{aligned}
\end{equation*}
By the union bound, we have for $t_1 = t_0 + \sum_{k=1}^{\overline K} t_k + \overline t$, 
\begin{equation*}
    \begin{aligned}
\PS\left(\sup_{\overline f \in \overline{\mathcal{F}}_r} \gap(\overline f) > t_1 \right) &\leq \PS\left(\sup_{\overline f \in \overline{\mathcal{F}}_r} \gap(\pi_0(\overline f)) > t_0 \right) \\ &+ \PS\left(\sup_{\overline f \in \overline{\mathcal{F}}_r} \sum_{k=1}^{\overline K} \left[\gap(\pi_k(\overline f)) - \gap(\pi_{k-1}(\overline f))\right] > \sum_{k=1}^{\overline K} t_k\right) \\
&+ \PS\left(\sup_{\overline f \in \overline{\mathcal{F}}_r} \left[\gap(\overline f) - \gap(\pi_{\overline K}(\overline f))\right] > \overline t \right).
    \end{aligned}
\end{equation*}

\textit{Step 3: Applying the chaining Lemmas.}
The terms above are bounded by Lemmas \ref{lem:anchor-ulln}, \ref{lem:telescope-ulln}, \ref{lem:approx-ulln} respectively.
\begin{lemma}\label{lem:anchor-ulln}
Suppose the conditions of Proposition \ref{prop: ulln-master} are satisfied, and adopt the same notation. Then for $\delta \in (0, 1)$, it holds with probability at least $1 - \delta$
\begin{equation*}
      \sup_{\overline f \in \overline{\mathcal{F}}_r} \gap(\pi_0(\overline f)) \leq 4M_T r \sqrt{\frac{Sp_i}{T}}\left\{\sqrt{\log \mathcal{N}(\overline{\mathcal{F}}_r, \epsilon_0, \|.\|_{\infty, \mathcal{B}_T})} + \sqrt{\log(1/\delta)} \right\}.
\end{equation*}
\end{lemma}
\begin{lemma}\label{lem:telescope-ulln}
    Suppose the conditions of Proposition \ref{prop: ulln-master} are satisfied, and adopt the same notation. Then for $\delta \in (0, 1)$, it holds with probability at least $1 - \delta$
    \begin{equation*}
\begin{aligned}
    &\sup_{f \in \overline{\mathcal{F}_r}} \sum_{k=1}^{\overline K} \gap(\pi_k(\overline f)) - \gap(\pi_{k-1}(\overline f)) \leq \\  &10M_T \sqrt{\frac{Sp_i}{T}}\left\{\int_{r2^{-(\overline K + 1)}}^{r/2} \sqrt{\log\mathcal{N}(\overline{\mathcal{F}}_r, \epsilon, \|.\|_{\infty, \mathcal{B}_T})} d\epsilon + r + r\sqrt{\log(1/\delta)}  \right\}.
\end{aligned}
\end{equation*}
\end{lemma}
\begin{lemma}\label{lem:approx-ulln}
      Suppose the conditions of Proposition \ref{prop: ulln-master} are satisfied, and adopt the same notation. Then for $\delta \in (0, 1)$, it holds with probability at least $1 - \delta$
      \begin{equation*}
    \begin{aligned}
        \sup_{\overline f \in \overline{\mathcal{F}}_r}\left[ \gap(\overline f) - \gap(\pi_{\overline K}(\overline f))\right] \leq  \sqrt{2}M_T r2^{-\overline K} \sqrt{\frac{\log(1/\delta)}{T}}.
    \end{aligned}
\end{equation*}
\end{lemma}
\textit{Step 4: Bounding the residual term.}
Decomposing the squared terms with $a^2 - b^2 = (a+b)(a-b)$, and using similar arguments as before, we have 
\begin{equation*}
    \begin{aligned}
        &\ltrans(f) - \ltrans(\overline f) \\
        &= \frac{1}{T}\sum_{t=0}^{T-1}\mathbb{E}\left[\mathbf{1}\{s(t+1) = i \}\left\{\|f(X_t)\|_2^2 - \|\overline f(X_t)\|_2^2 \right\} \right] \\
        &= \frac{1}{T}\sum_{t=0}^{T-1}\mathbb{E}\left[\mathbf{1}\{s(t+1) = i \}\left\{\|f(X_t)\|_2 + \|\overline f(X_t)\|_2 \right\}\left\{\|f(X_t)\|_2 - \|\overline f(X_t)\|_2 \right\} \right] \\
        &\leq \frac{1}{T}\sum_{t=0}^{T-1}\mathbb{E}\left[\mathbf{1}\{s(t+1) = i \}\left\{\|f(X_t)\|_2 + M_T \right\}\left\{\|f(X_t) - \overline f(X_t)\|_2 \right\} \right].
    \end{aligned}
\end{equation*}
Recall that on the event $X_t \in \mathcal{B}_T$, $\|f(X_t)\|_2 \leq M_T$ by the construction of $M_T$, which implies that $\overline f = f$. Thus, we have 
\begin{equation*}
    \begin{aligned}
        &\frac{1}{T}\sum_{t=0}^{T-1}\mathbb{E}\left[\mathbf{1}\{s(t+1) = i \}\left\{\|f(X_t)\|_2 + M_T \right\}\left\{\|f(X_t) - \overline f(X_t)\|_2 \right\} \right] \\
        &=\frac{1}{T}\sum_{t=0}^{T-1}\mathbb{E}\left[\mathbf{1}\{s(t+1) = i \}\left\{\|f(X_t)\|_2 + M_T \right\}\left\{\|f(X_t) - \overline f(X_t)\|_2 \right\}\{\mathbf{1}\{X_t \in \mathcal{B}_T\} + \mathbf{1}\{X_t \notin \mathcal{B}_T\} \} \right] \\
        &= \frac{1}{T}\sum_{t=0}^{T-1}\mathbb{E}\left[\mathbf{1}\{s(t+1) = i \}\left\{\|f(X_t)\|_2 + M_T \right\}\left\{\|f(X_t) - \overline f(X_t)\|_2 \right\}\mathbf{1}\{X_t \notin \mathcal{B}_T\}\right] \\
        &= \frac{1}{T}\sum_{t=0}^{T-1}\mathbb{E}\left[\mathbf{1}\{s(t+1) = i \}\left\{\|f(X_t)\|_2 + M_T \right\}\left\{\|f(X_t) - M_T \frac{f(X_t)}{\|f(X_t)\|_2}\|_2 \right\}\mathbf{1}\{X_t \notin \mathcal{B}_T\}\right] \\
        &= \frac{1}{T}\sum_{t=0}^{T-1}\mathbb{E}\left[\mathbf{1}\{s(t+1) = i \}\left\{\|f(X_t)\|_2 + M_T \right\}\left\|f(X_t)\left[1 - \frac{M_T}{\|f(X_t)\|_2}\right] \right\|_2 \mathbf{1}\{X_t \notin \mathcal{B}_T\}\right] \\
        &= \frac{1}{T}\sum_{t=0}^{T-1}\mathbb{E}\left[\mathbf{1}\{s(t+1) = i \}\left\{\frac{\|f(X_t)\|_2^2 - M_T^2}{\|f(X_t)\|_2}  \right\}\left\|f(X_t)\right\|_2 \mathbf{1}\{X_t \notin \mathcal{B}_T\}\right] \\
        &= \frac{1}{T}\sum_{t=0}^{T-1}\mathbb{E}\left[\mathbf{1}\{s(t+1) = i \}\left\{\|f(X_t)\|_2^2 - M_T^2 \right\} \mathbf{1}\{X_t \notin \mathcal{B}_T\}\right] \\
        &= \frac{p_i}{T}\sum_{t=0}^{T-1}\mathbb{E}\left[\left\{\|f(X_t)\|_2^2 - M_T^2 \right\} \mathbf{1}\{X_t \notin \mathcal{B}_T\}\right] \\
        &\leq \frac{p_i}{T}\sum_{t=0}^{T-1}\mathbb{E}\left[\|f(X_t)\|_2^2\mathbf{1}\{\|X_t\|_2 > R_T\}\right].
    \end{aligned}
\end{equation*}

Thus, for some $t_2 > 0$, we have by using the envelope function
\begin{equation*}
    \begin{aligned}
        \mathbb{P}\left(\sup_{f \in \mathcal{F}_{i,r}^*}\left\{\ltrans(f) - \ltrans(T_M[f]) \right\} > t_2 \right) &= \mathbf{1}\left\{\sup_{f \in \mathcal{F}_{i,r}^*}\left(\ltrans(f) - \ltrans(T_M[f]) \right) > t_2 \right\} \\
        &\leq \mathbf{1}\left\{\sup_{f \in \mathcal{F}_{i,r}^*}\frac{p_i}{T}\sum_{t=0}^{T-1}\mathbb{E}\left[\|f(X_t)\|_2^2\mathbf{1}\{\|X_t\|_2 > R_T\}\right] > t_2 \right\} \\
        &\leq \mathbf{1}\left\{\frac{p_i}{T}\sum_{t=0}^{T-1}\mathbb{E}\left[F^2(X_t)\mathbf{1}\{\|X_t\|_2 > R_T\}\right] > t_2 \right\} \\
        &=: \mathbf{1}\left\{I_T(F) > t_2 \right\}
    \end{aligned}
\end{equation*}
since the residual term is deterministic.

\textit{Step 5: Assembly of the critical radius equation.}

Denote the bounds in the chaining terms from Lemmas \ref{lem:anchor-ulln}, \ref{lem:telescope-ulln} and \ref{lem:approx-ulln} as 
\begin{equation*}
\begin{aligned}
    &\Phi_1(r, \delta) := 4M_T r \sqrt{\frac{Sp_i}{T}}\left\{\sqrt{\log \mathcal{N}(\overline{\mathcal{F}}_r, \epsilon_0, \|.\|_{\infty, \mathcal{B}_T})} + \sqrt{\log(1/\delta)} \right\}, \\
    &\Phi_2(r, \delta) := 10M_T \sqrt{\frac{Sp_i}{T}}\left\{\int_{r2^{-(\overline K + 1)}}^{r/2} \sqrt{\log\mathcal{N}(\overline{\mathcal{F}}_r, \epsilon, \|.\|_{\infty, \mathcal{B}_T})} d\epsilon + r + r\sqrt{\log(1/\delta)}  \right\}, \\
    &\Phi_3(r, \delta) := \sqrt{2}M_T r2^{-\overline K} \sqrt{\frac{\log(1/\delta)}{T}}.
\end{aligned}
\end{equation*}
Assigning a probability budget of $\delta/4$ to each of the terms above, we have
\begin{equation*}
\begin{aligned}
    &\PS\left(\sup_{\overline f \in \overline{\mathcal{F}}_r} \gap(\pi_0(\overline f)) \leq \Phi_1(r, \delta/4)  \right) \geq 1 - \delta/4, \\
    &\PS\left(\sup_{\overline f \in \overline{\mathcal{F}}_r} \sum_{k=1}^{\overline K} \left[\gap(\pi_k(\overline f)) - \gap(\pi_{k-1}(\overline f))\right] \leq \Phi_2(r, \delta/4) \right)  \geq 1 - \delta/4, \\
    &\PS\left(\sup_{\overline f \in \overline{\mathcal{F}}_r} \left[\gap(\overline f) - \gap(\pi_{\overline K}(\overline f))\right] \leq \Phi_3(r, \delta/4) \right) \geq 1 - \delta/4.
\end{aligned}
\end{equation*}
Moreover, the residual term can be made to vanish by choosing $t_2$ larger than $I_T(F)$.
Therefore for the probability bounds to hold, we have the conditions
\begin{equation*}
    \frac{p_i (r^*)^2}{4} = \Phi_1(r^*, \delta/4) + \Phi_2(r^*, \delta/4) + \Phi_3(r^*, \delta/4),
\end{equation*}
and 
\begin{equation*}
     \frac{p_i (r^*)^2}{4}  = I_T(F),
\end{equation*}
which after summing both sides gives us the critical radius equation
\begin{equation}\label{eq:fixed-point-cond}
     \frac{p_i (r^*)^2}{2} = \Phi_1(r^*, \delta/4) + \Phi_2(r^*, \delta/4) + \Phi_3(r^*, \delta/4) +  I_T(F).
\end{equation}
Thus, we have from \eqref{eq:step1_ullnproof_prop}, for any $r \geq r^*$ satisfying the fixed point equation \eqref{eq:fixed-point-cond}, that
\begin{equation*}
    \mathbb{P}(\mathcal{E} \cap \mathcal{S}_T) \leq \frac{3\delta}{4}.
\end{equation*}
Finally, assigning a probability budget of $\delta/4$ to the escape event $\mathcal{S}_T^c$, which is satisfied as soon as $R_T > \sqrt{\log(4TC_{\rho, B}/\delta)/\alpha}$, we have for any $r \geq r^*$ satisfying \eqref{eq:fixed-point-cond}, it holds that
    \begin{equation*}
    \mathbb{P}\left(\sup_{f \in \mathcal{F}^*_{i, r}} G_{T,i}(f) \leq \frac{p_i r^2}{2}  \right) \geq 1 - \delta,
\end{equation*}
as desired.

\textit{Step 6: Rescaling functions outside the sphere.}

Let $f \in \mathcal{F}_i^*$ be an arbitrary function such that $\ltrans(f) > p_i r^2$. Define $\lambda := r\sqrt{p_i}/\sqrt{\ltrans(f)}$. Consider the scaled function $h = \lambda f$. Since $\lambda < 1$ by construction, by the star-shapedness of $\mathcal{F}_i^*$, $h \in \mathcal{F}_i^*$. Note that 
\begin{equation*}
    \ltrans(h) = \ltrans(\lambda f) = \lambda^2 \ltrans(f) \qquad \text{and} \qquad \lemp(h) = \lemp(\lambda f ) = \lambda^2 \lemp(f).
\end{equation*}
Moreover, by the definition of $\lambda$, 
\begin{equation*}
    \ltrans(h) = \frac{r^2 p_i}{\ltrans(f)}  \ltrans(f) = p_i r^2 ,
\end{equation*}
so the scaled function $h \in \mathcal{F}^*_{i,r}$. By the previous step, we have for any $r \geq r^*$ satisfying the fixed point equation and $R_T \geq \sqrt{\log(4TC/\delta)/\alpha)}$, with probability at least $1 - \delta$, 
\begin{equation*}
    \begin{aligned}
       \gap(h) &\leq \frac{p_i r^2}{2} \\
       \iff \ltrans(h) - \lemp(h) &\leq \frac{p_i r^2}{2} \\
       \iff \lambda^2 \ltrans(f) - \lambda^2 \lemp(f) &\leq \frac{p_i r^2}{2} \\
       \iff \frac{r^2p_i}{\ltrans(f)}\ltrans(f) - \lambda^2 \lemp(f) &\leq \frac{p_i r^2}{2} \\
       \iff \lambda^2 \lemp(f) &> \frac{p_i r^2}{2} \\
       \iff \lemp(f) > \frac{p_i r^2}{2 \lambda^2} &= \frac{\ltrans(f)}{2},
    \end{aligned}
\end{equation*}
so our desired statement similarly holds for every $f$ outside the sphere.
%
%

%
%
%
\subsection{Proof of Lemma \ref{lem:anchor-ulln}}
Set $W_t := \|\pi_0(\overline f)(X_t)\|_2^2 \mathbf{1}\{s(t+1) = i\}$. Recall that for an event $A$, we denote $\PS(A) := \mathbb{P}(A \cap \mathcal{S}_T)$. By Markov's inequality, we have 
\begin{equation*}
    \begin{aligned}
        &\PS\left(\gap(\pi_0(\overline f)) > t_0 \right) = \PS\left(\ltrans(\pi_0(\overline f )) -\lemp(\pi_0(\overline f))  > t_0 \right) \\ 
        &= \PS\left(\frac{1}{T}\sum_{t=0}^{T-1}\left\{ \mathbb{E}\left[\|\pi_0(\overline f)(X_t)\|_2^2 \mathbf{1}\{s(t+1) = i \} \right] - \|\pi_0(\overline f)(X_t)\|_2^2 \mathbf{1}\{s(t+1) = i \} \right\} > t_0  \right) \\
        &\leq  \PS\left(\frac{1}{T}\sum_{t=0}^{T-1} \left\{\mathbb{E}\left[W_t\right] - W_t \right\}  > t_0  \right) \\
        &\leq \PS\left(\exp\left(\frac{\xi}{T}\sum_{t=0}^{T-1} \left\{\mathbb{E}\left[W_t\right] - W_t \right\} \right) > \exp\left(\xi t_0 \right) \right) \\
        &\leq \inf_{\xi >0}\mathbb{E}\left[\exp\left(\frac{\xi}{T}\sum_{t=0}^{T-1} \left\{\mathbb{E}\left[W_t\right] - W_t \right\} \right) \mathbf{1}_{\mathcal{S}_T} \right]\exp\left(-\xi t_0\right).
    \end{aligned}
\end{equation*}

Now, define $Z_t := W_t\mathbf{1}\{X_t \in \mathcal{B}_T\}$, so that on the event $\mathcal{S}_T$, we have $Z_t = W_t$. Thus, we have
\begin{equation}\label{eq:anchor-trun-var}
    \begin{aligned}
        \mathbb{E}\left[\exp\left(\frac{\xi}{T}\sum_{t=0}^{T-1} \left\{\mathbb{E}\left[W_t\right] - W_t \right\} \right) \mathbf{1}_{\mathcal{S}_T} \right] &= \mathbb{E}\left[\exp\left(\frac{\xi}{T}\sum_{t=0}^{T-1} \left\{\mathbb{E}\left[Z_t\right] - Z_t \right\} \right) \mathbf{1}_{\mathcal{S}_T} \right] \\
        &\leq \mathbb{E}\left[\exp\left(\frac{\xi}{T}\sum_{t=0}^{T-1} \left\{\mathbb{E}\left[Z_t\right] - Z_t \right\} \right) \right] \\
        &= \exp\left(\frac{\xi}{T}\sum_{t=0}^{T-1}\mathbb{E}[Z_t] \right)\mathbb{E}\left[\exp\left(-\frac{\xi}{T}\sum_{t=0}^{T-1}Z_t \right) \right].
    \end{aligned}
\end{equation}
Again applying Theorem \ref{thm:S-persis}, we have
\begin{equation*}
    \begin{aligned}
        \mathbb{E}\left[\exp\left(-\frac{\xi}{T}\sum_{t=0}^{T-1}Z_t \right) \right] &\leq \exp\left(-\frac{\xi}{T}\sum_{t=0}^{T-1}\mathbb{E}[Z_t] + \frac{\xi^2 S}{2T^2}\sum_{t=0}^{T-1}\mathbb{E}[Z_t^2] \right),
    \end{aligned}
\end{equation*}
and substituting back into \eqref{eq:anchor-trun-var}, we have
\begin{equation*}
    \begin{aligned}
        \mathbb{E}\left[\exp\left(\frac{\xi}{T}\sum_{t=0}^{T-1} \left\{\mathbb{E}\left[W_t\right] - W_t \right\} \right) \mathbf{1}_{\mathcal{S}_T} \right] &\leq \exp\left(\frac{\xi^2 S}{2T}\frac{1}{T}\sum_{t=0}^{T-1}\mathbb{E}[Z_t^2] \right),
    \end{aligned}
\end{equation*}
which implies that 
\begin{equation*}
    \begin{aligned}
        \PS\left(\gap(\pi_0(\overline f)) > t_0 \right) &\leq \exp\left(\frac{\xi^2 S}{2T}\frac{1}{T}\sum_{t=0}^{T-1}\mathbb{E}[Z_t^2] \right)\exp\left(-\xi t_0 \right).
    \end{aligned}
\end{equation*}
By the union bound, we have 
\begin{equation*}
    \begin{aligned}
        \PS\left(\sup_{\overline f \in \overline{\mathcal{F}}_r} \gap(\pi_0(\overline f)) > t_0 \right) \leq \mathcal{N}(\overline{\mathcal{F}}_r, \epsilon_0, \|.\|_{\infty, \mathcal{B}_T})\inf_{\xi >0}\exp\left(\frac{\xi^2 S}{2T}\frac{1}{T}\sum_{t=0}^{T-1}\mathbb{E}[Z_t^2] \right)\exp\left(-\xi t_0 \right).
    \end{aligned}
\end{equation*}

Intercalating by $\overline f(X_t)$, and recalling that on the sphere, $L_{T,i}(f) = p_i r^2$, we have
\begin{equation*}
    \begin{aligned}
        \frac{1}{T}\sum_{t=0}^{T-1}\mathbb{E}[Z_t^2] &= \frac{1}{T}\sum_{t=0}^{T-1} \mathbb{E}[W_t^2 \mathbf{1}\{X_t \in \mathcal{B}_T\}] \\
        &\leq \frac{M_T^2}{T}\sum_{t=0}^{T-1}\mathbb{E}\left[\|\pi_0(\overline f)(X_t)\|_2^2 \mathbf{1}\{s(t+1) = i \}\mathbf{1}\{X_t \in \mathcal{B}_T \} \right] \\
        &\leq \frac{2M_T^2}{T}\sum_{t=0}^{T-1}\mathbb{E}\left[\|\pi_0(\overline f)(X_t) - \overline f(X_t) \|_2^2 \mathbf{1}\{s(t+1) = i \}\mathbf{1}\{X_t \in \mathcal{B}_T \} \right] \\
        &+ \frac{2M_T^2}{T}\sum_{t=0}^{T-1}\mathbb{E}\left[\|\overline f(X_t)\|_2^2 \mathbf{1}\{s(t+1) = i \}\mathbf{1}\{X_t \in \mathcal{B}_T \} \right] \\
        &\leq \frac{2M_T^2 p_i }{T}\sum_{t=0}^{T-1}\|\pi_0(\overline f) - \overline f \|_{\infty, \mathcal{B}_T}^2 + \frac{2M_T^2}{T}\sum_{t=0}^{T-1}\mathbb{E}\left[\|\overline f(X_t)\|_2^2 \mathbf{1}\{s(t+1) = i \}\mathbf{1}\{X_t \in \mathcal{B}_T \} \right] \\
        &\leq 2M_T^2 p_i \epsilon_0^2 + \frac{2M_T^2}{T}\sum_{t=0}^{T-1}\mathbb{E}\left[\|f(X_t)\|_2^2 \mathbf{1}\{s(t+1) = i \}\mathbf{1}\{X_t \in \mathcal{B}_T \} \right] \\
        &\leq 2M_T^2 p_i \epsilon_0^2 + \frac{2M_T^2}{T}\sum_{t=0}^{T-1}\mathbb{E}\left[\|f(X_t)\|_2^2 \mathbf{1}\{s(t+1) = i \} \right] \\
        &= 2M_T^2 p_i \epsilon_0^2 + 2M_T^2 L_{T,i}(f) \\
        &= 4M_T^2 p_i r^2.
    \end{aligned} 
\end{equation*}
Thus, we have
\begin{equation*}
     \PS\left(\sup_{\overline f \in \overline{\mathcal{F}}_r} \gap(\pi_0(\overline f)) > t_0 \right) \leq \mathcal{N}(\overline{\mathcal{F}}_r, \epsilon_0, \|.\|_{\infty, \mathcal{B}_T})\inf_{\xi >0}\exp\left(\frac{2\xi^2 S M_T^2 p_i r^2}{T} - \xi t_0\right).
\end{equation*}

Optimizing with respect to $\xi$, we obtain the minimizer as
\begin{equation*}
    \xi^* = \frac{t_0 T}{4SM_T^2 p_i r^2}, 
\end{equation*}
which after substitution yields
\begin{equation*}
    \PS\left(\sup_{\overline f \in \overline{\mathcal{F}}_r} \gap(\pi_0(\overline f)) > t_0 \right) \leq \mathcal{N}(\overline{\mathcal{F}}_r, \epsilon_0, \|.\|_{\infty, \mathcal{B}_T}) \exp\left(-\frac{t_0^2 T}{8S M_T^2 p_i r^2} \right).
\end{equation*}
Setting the RHS to at most $\delta$, we obtain with probability at least $1 - \delta$, 
\begin{equation*}
    \begin{aligned}
        \sup_{\overline f \in \overline{\mathcal{F}}_r} \gap(\pi_0(\overline f)) &\leq 2\sqrt{2}M_T r \sqrt{\frac{Sp_i}{T}}\left(\sqrt{\log \mathcal{N}(\overline{\mathcal{F}}_r, \epsilon_0, \|.\|_{\infty, \mathcal{B}_T}) + \log(1/\delta)} \right) \\
        &\leq 4M_T r \sqrt{\frac{Sp_i}{T}}\left\{\sqrt{\log \mathcal{N}(\overline{\mathcal{F}}_r, \epsilon_0, \|.\|_{\infty, \mathcal{B}_T})} + \sqrt{\log(1/\delta)} \right\}.
    \end{aligned}
\end{equation*}
\subsection{Proof of Lemma \ref{lem:telescope-ulln}} \label{subsec:proof_telescope-ulln}
By suitably grouping terms together, we have 
\begin{equation*}
    \begin{aligned}
        \gap(\pi_k(\overline f)) - \gap(\pi_{k-1}(\overline f)) &= \ltrans(\pi_k(\overline f)) - \ltrans(\pi_{k-1}(\overline f)) - \left\{\lemp(\pi_k(\overline f)) - \lemp(\pi_{k-1}(\overline f))  \right\} \\
        &= \frac{1}{T}\sum_{t=0}^{T-1}\mathbb{E}\left[\mathbf{1}\{s(t+1) = i \}\left\{\|\pi_k(\overline f)(X_t)\|_2^2 - \|\pi_{k-1}(\overline f)(X_t) \|_2^2 \right\}\right] \\
        &- \frac{1}{T}\sum_{t=0}^{T-1}\mathbf{1}\{s(t+1) = i \}\left\{\|\pi_k(\overline f)(X_t)\|_2^2 - \|\pi_{k-1}(\overline f)(X_t)\|_2^2  \right\}.
    \end{aligned}
\end{equation*}

Recall that $M_T = \sup_{\|x\|_2 \leq R_T} \|f(x)\|_2$. Then, using the identity $a^2 - b^2 = (a+b)(a-b)$, and the $1$-Lipschitz property of $\|.\|_2$,  we have 
\begin{equation*}
    \begin{aligned}
        \|\pi_k(\overline f)(X_t)\|_2^2 - \|\pi_{k-1}(\overline f)(X_t)\|_2^2 &= \left(\|\pi_k(\overline f)(X_t)\|_2 + \|\pi_{k-1}(\overline f)(X_t)\|_2  \right) \\ &\times \left(\|\pi_k(\overline f)(X_t)\|_2 - \|\pi_{k-1}(\overline f)(X_t)\|_2 \right) \\
        &\leq 2M_T \|\pi_k(\overline f)(X_t) - \pi_{k-1}(\overline f)(X_t)\|_2.
    \end{aligned}
\end{equation*}

Set $h_t := \mathbf{1}\{s(t+1) = i \}\|\pi_k(\overline f)(X_t) - \pi_{k-1}(\overline f)(X_t)\|_2$. By Markov's inequality, we have  
\begin{equation*}
\begin{aligned}
    \PS\left( \gap(\pi_k(\overline f)) - \gap(\pi_{k-1}(\overline f)) > t_k  \right) &\leq \PS\left(\frac{2M_T}{T}\sum_{t=0}^{T-1}\mathbb{E}[h_t] - h_t > t_k \right) \\
    &= \PS\left(\exp\left(\frac{2\xi M_T}{T}\sum_{t=0}^{T-1} \mathbb{E}[h_t] - h_t \right) > \exp\left(\xi t_k \right) \right) \\
    &\leq \inf_{\xi > 0} \mathbb{E}\left[ \exp\left(\frac{2\xi M_T}{T}\sum_{t=0}^{T-1} \mathbb{E}[h_t] - h_t \right) \mathbf{1}_{\mathcal{S}_T} \right]\exp\left(-\xi t_k \right).
\end{aligned}
\end{equation*}

Define $g_t := h_t \mathbf{1}\{X_t \in \mathcal{B}_T\}$. Note that on the safe event $\mathcal{S}_T$, $g_t = h_t$. Thus, we have 
\begin{equation*}
    \begin{aligned}
        \mathbb{E}\left[ \exp\left(\frac{2\xi M_T}{T}\sum_{t=0}^{T-1} \mathbb{E}[h_t] - h_t \right) \mathbf{1}_{\mathcal{S}_T} \right] &= \mathbb{E}\left[ \exp\left(\frac{2\xi M_T}{T}\sum_{t=0}^{T-1} \mathbb{E}[g_t] - g_t \right) \mathbf{1}_{\mathcal{S}_T} \right] \\
        &= \exp\left(\frac{2\xi M_T}{T}\sum_{t=0}^{T-1} \mathbb{E}[g_t] \right)\mathbb{E}\left[\exp\left(-\frac{2\xi M_T}{T}\sum_{t=0}^{T-1} g_t \right) \mathbf{1}_{\mathcal{S}_T} \right].
    \end{aligned}
\end{equation*}
Since $g_t$ is positive and bounded, we can apply Theorem \ref{thm:S-persis} to obtain
\begin{equation*}
    \begin{aligned}
        \mathbb{E}\left[\exp\left(-\frac{2\xi M_T}{T}\sum_{t=0}^{T-1} g_t \right) \mathbf{1}_{\mathcal{S}_T} \right] &\leq \mathbb{E}\left[\exp\left(-\frac{2\xi M_T}{T}\sum_{t=0}^{T-1} g_t \right)  \right] \\
        &\leq \exp\left(-\frac{2\xi M_T}{T}\sum_{t=0}^{T-1} \mathbb{E}[g_t] + \frac{2\xi^2 S M_T^2}{T^2}\sum_{t=0}^{T-1}\mathbb{E}[g_t^2] \right).
    \end{aligned}
\end{equation*}
Substituting back, we have 
\begin{equation*}
    \begin{aligned}
         \mathbb{E}\left[ \exp\left(\frac{2\xi M_T}{T}\sum_{t=0}^{T-1} \mathbb{E}[h_t] - h_t \right) \mathbf{1}_{\mathcal{S}_T} \right]  \leq \exp\left(\frac{2\xi^2 S M_T^2}{T^2}\sum_{t=0}^{T-1}\mathbb{E}[g_t^2] \right).
    \end{aligned}
\end{equation*}

Using the identity $(a+b)^2 \leq 2(a^2 + b^2)$, and recalling that $\epsilon_k = 2^{-1}\epsilon_{k-1}$, we have
\begin{equation*}
    \begin{aligned}
        \mathbb{E}[g_t^2] &= \mathbb{E}[h_t^2 \mathbf{1}\{X_t \in \mathcal{B}_T\}]
        \\
        &= \mathbb{E}\left[\mathbf{1}\{s(t+1) = i \}\mathbf{1}\{X_t \in \mathcal{B}_T\} \|\pi_k(\overline f)(X_t) - \pi_{k-1}(\overline f)(X_t)  \|_2^2 \right] \\
        &\leq 2p_i \mathbb{E}\left[ \left\{\| \pi_k(\overline f)(X_t) - \overline f(X_t) \|_2^2 + \|\overline f(X_t) - \pi_{k-1}(\overline f)(X_t)\|_2^2\right\}\mathbf{1}\{X_t \in \mathcal{B}_T \} \right] \\
        &\leq 2p_i\mathbb{E}\left[\| \pi_k(\overline f) - \overline f \|_{\infty, \mathcal{B}_T}^2 + \|\overline f - \pi_{k-1}(\overline f)\|_{\infty, \mathcal{B}_T}^2 \right] \\
        &\leq 2p_i (\epsilon_k^2 + \epsilon_{k-1}^2) \\
        &\leq 10p_i \epsilon_k^2.
    \end{aligned}
\end{equation*}
This implies that 
\begin{equation}\label{eq:telescope-chernoff}
    \begin{aligned}
        \PS\left( \gap(\pi_k(\overline f)) - \gap(\pi_{k-1}(\overline f)) > t_k  \right) \leq  \inf_{\xi > 0}\exp\left(\frac{20\xi^2 M_T^2 S p_i \epsilon_k^2}{T} - \xi t_k \right).
    \end{aligned}
\end{equation}
Optimizing the RHS above with respect to $\xi$, we obtain the solution
\begin{equation*}
    \xi^* = \frac{t_k T}{40 M_T^2 S \epsilon_k^2 p_i},
\end{equation*}
which after substituting back into \eqref{eq:telescope-chernoff} gives us
\begin{equation*}
    \begin{aligned}
        \PS\left( \gap(\pi_k(\overline f)) - \gap(\pi_{k-1}(\overline f)) > t_k  \right) \leq  \exp\left(-\frac{t_k^2 T}{10 SM_T^2 \epsilon_k^2 p_i } \right).
    \end{aligned}
\end{equation*}

It is easy to see that 
\begin{equation*}
    \begin{aligned}
        \left\{\sup_{f \in \overline{\mathcal F}_r} \sum_{k=1}^{\overline K} \left[ \gap(\pi_k(\overline f)) - \gap(\pi_{k-1}(\overline f)) \right] > \sum_{k=1}^{\overline K} t_k  \right\} \subseteq
        \\
        \bigcup_{k=1}^{\overline K}\left\{\sup_{f \in \overline{\mathcal F}_r} \left[ \gap(\pi_k(\overline f)) - \gap(\pi_{k-1}(\overline f)) \right] >  t_k \right\},
    \end{aligned}
\end{equation*}
which implies
\begin{equation*}
    \begin{aligned}
         \PS\left(\sup_{f \in \overline{\mathcal{F}_r}} \sum_{k=1}^{\overline K} \gap(\pi_k(\overline f)) - \gap(\pi_{k-1}(\overline f)) > \sum_{k=1}^{\overline K} t_k  \right) \\ \leq \sum_{k=1}^{\overline K}\PS\left(\sup_{f \in \overline{\mathcal{F}_r}}  \gap(\pi_k(\overline f)) - \gap(\pi_{k-1}(\overline f)) >  t_k  \right) \\
         \leq \sum_{k=1}^{\overline K}\mathcal{N}^2\left(\overline{\mathcal{F}}_r, \epsilon_k, \|.\|_{\infty, \mathcal{B}_T}\right)\exp\left(-\frac{t_k^2 T}{10SM_T^2 \epsilon_k^2 p_i} \right).
    \end{aligned}
\end{equation*}

For a scale $k$, set the RHS of the display above to at most $\delta/2^{k}$, to obtain 
\begin{equation*}
\begin{aligned}
    t_k &= \sqrt{\frac{10 S p_i}{T}}M_T \epsilon_k \sqrt{\log\left(\mathcal{N}^2\left(\overline{\mathcal{F}_r}, \epsilon_k, \|.\|_{\infty, \mathcal{B}_T}  \right) 2^{k}/\delta \right)} \\
    &= \sqrt{\frac{10 S p_i}{T}}M_T \epsilon_k   \left\{\sqrt{\log\mathcal{N}^2\left(\overline{\mathcal{F}_r}, \epsilon_k, \|.\|_{\infty, \mathcal{B}_T}\right)} + \sqrt{\log\left(2^{k} \right)} + \sqrt{\log\left(\frac{1}{\delta}\right)} \right\},
\end{aligned}
\end{equation*}
and summing over $k = 1, 2, \dots$ gives us
\begin{equation*}
\begin{aligned}
    \sum_{k=1}^{\overline K} t_k &= \sqrt{\frac{10 S p_i}{T}}M_T   \left\{\sum_{k=1}^{\overline K} \epsilon_k \sqrt{\log\mathcal{N}^2\left(\overline{\mathcal{F}_r}, \epsilon_k, \|.\|_{\infty, \mathcal{B}_T}\right)} +\sum_{k=1}^{\overline K} \epsilon_k \sqrt{\log(2^{k})} +  \sum_{k=1}^{\overline K}\epsilon_k \sqrt{\log\left(\frac{1}{\delta}\right)} \right\} \\
    &=: \sqrt{\frac{10 S p_i}{T}}M_T  \left\{\sum_{k=1}^{\overline K}t_{1,k} + \sum_{k=1}^{\overline K}t_{2,k} + \sum_{k=1}^{\overline K}t_{3,k} \right\}.
\end{aligned}
\end{equation*}
Since $\epsilon \mapsto \log\mathcal{N}(\overline{\mathcal{F}}_r, \epsilon, \|.\|_{\infty, \mathcal{B}_T})$ is non-increasing, we have 
\begin{equation*}
    \begin{aligned}
      \int_{r2^{-(\overline K + 1)}}^{r/2} \sqrt{\log\mathcal{N}(\overline{\mathcal{F}}_r, \epsilon, \|.\|_{\infty, \mathcal{B}_T}}) d\epsilon &=   \int_{\epsilon_{\overline K+1}}^{\epsilon_1} \sqrt{\log\mathcal{N}(\overline{\mathcal{F}}_r, \epsilon, \|.\|_{\infty, \mathcal{B}_T}}) d\epsilon \\
      &\geq \frac{1}{2}\sum_{k=1}^{\overline K}r2^{-k} \sqrt{\log\mathcal{N}(\overline{\mathcal{F}}_r, r2^{-k}, \|.\|_{\infty, \mathcal{B}_T})} \\
      &= \frac{1}{2}\sum_{k=1}^{\overline K}\epsilon_k \sqrt{\log\mathcal{N}\left(\overline{\mathcal{F}_r}, \epsilon_k, \|.\|_{\infty, \mathcal{B}_T}\right)}.
    \end{aligned}
\end{equation*}
Thus, we have 
\begin{equation*}
    \begin{aligned}
        \sum_{k=1}^{\overline K}t_{1,k} &\leq 2\sqrt{2}\int_{r2^{-(\overline K + 1)}}^{r/2} \sqrt{\log\mathcal{N}(\overline{\mathcal{F}}_r, \epsilon, \|.\|_{\infty, \mathcal{B}_T}}) d\epsilon .
    \end{aligned}
\end{equation*}
Moreover, by the polylogarithmic function, we have  
\begin{equation*}
\begin{aligned}
    \sum_{k =1}^{\overline K}t_{2,k}  &\leq \sum_{k=1}^\infty t_{2,k}     
    = \sqrt{\log(2)} \sum_{k =1}^{\infty} r2^{-k}\sqrt{k} \leq 2r,
\end{aligned}
\end{equation*}
and by the geometric series,
\begin{equation*}
    \sum_{k=1}^{\overline K} t_{3,k} \leq \sqrt{\log(1/\delta)}\sum_{k=1}^{\infty} r2^{-k}  = r\sqrt{\log(1/\delta)}.
\end{equation*}

Gathering facts, we have with probability at least $1 - \delta$
\begin{equation*}
\begin{aligned}
    &\sup_{f \in \overline{\mathcal{F}_r}} \sum_{k=1}^{\overline K} \gap(\pi_k(\overline f)) - \gap(\pi_{k-1}(\overline f)) \leq \\  &10M_T \sqrt{\frac{Sp_i}{T}}\left\{\int_{r2^{-(\overline K + 1)}}^{r/2} \sqrt{\log\mathcal{N}(\overline{\mathcal{F}}_r, \epsilon, \|.\|_{\infty, \mathcal{B}_T})} d\epsilon + r + r\sqrt{\log(1/\delta)}  \right\}.
\end{aligned}
\end{equation*}
\subsection{Proof of Lemma \ref{lem:approx-ulln}} \label{subsec:proof_approx-ulln}
Using $a^2 - b^2 = (a+b)(a-b)$ and 1-Lipschitzness of $\|.\|_2$, we have
\begin{equation*}
\begin{aligned}
    \|\overline f(X_t)\|_2^2 - \|\pi_{\overline K}(\overline f)(X_t) \|_2^2 &= \left\{\|\overline f(X_t)\|_2 + \|\pi_{\overline K}(X_t)\|_2\right\}\left\{\|\overline f(X_t)\|_2 - \|\pi_{\overline K}(X_t)\|_2 \right\} \\
    &\leq 2M_T\|\overline f(X_t) - \pi_{\overline K}(X_t)\|_2.
\end{aligned}
\end{equation*}
Set $U_t := \|\overline f(X_t) - \pi_{\overline K}(X_t)\|_2\mathbf{1}\{s(t+1) = i\}$. Then we can decompose the fine scale approximation error as
\begin{equation*}
    \begin{aligned}
        \gap(\overline f) - \gap(\pi_{\overline K}(\overline f)) &= \ltrans(\overline f) - \ltrans(\pi_{\overline K}(\overline f) - (\lemp(\overline f) - \lemp(\pi_{\overline K}(\overline f))) \\
    &\leq \frac{2M_T}{T}\sum_{t=0}^{T-1}\left(\mathbb{E}[U_t] - U_t \right)
    \end{aligned}
\end{equation*}
Note that on the event $\mathcal{S}_T$, we have 
\begin{equation*}
    \begin{aligned}
       \mathbb{E}[U_t] - U_t &= \mathbb{E}\left[\|\overline f(X_t) - \pi_{\overline K}(X_t)\|_2\mathbf{1}\{s(t+1) = i\}\right] - \|\overline f(X_t) - \pi_{\overline K}(X_t)\|_2\mathbf{1}\{s(t+1) = i\} \\
       &\leq p_i\|\overline f - \pi_{\overline K}\|_{\infty, \mathcal{B}_T}  - \|\overline f - \pi_{\overline K}\|_{\infty, \mathcal{B}_T}\mathbf{1}\{s(t+1) = i \} \\
       &\leq p_i \epsilon_{\overline K} - \epsilon_{\overline K}\mathbf{1}\{s(t+1) = i \} \\
       &= \epsilon_{\overline K}\left(p_i - \mathbf{1}\{s(t+1) = i \} \right),
    \end{aligned}
\end{equation*}
where we remark that the upper bound in the last display no longer depends on $\overline f$. Thus, we have
\begin{equation*}
    \begin{aligned}
        \PS\left(\sup_{\overline f \in \overline{\mathcal{F}}_r} \gap(\overline f) - \gap(\pi_{\overline K}(\overline f)) > \overline t \right)  &\leq
        \PS\left(\sup_{\overline f \in \overline{\mathcal{F}}_r} \frac{2M_T}{T}\sum_{t=0}^{T-1}\left(\mathbb{E}[U_t] - U_t \right) > \overline t \right) \\ &\leq \PS\left(\frac{2M_T\epsilon_{\overline K}}{T}\sum_{t=0}^{T-1} \left(p_i - \mathbf{1}\{s(t+1) = i \} \right) > \overline t \right) \\
        &= \PS\left(\sum_{t=0}^{T-1} \left(p_i - \mathbf{1}\{s(t+1) = i \} \right) > \frac{\overline t T}{2M_T\epsilon_{\overline K}} \right).
    \end{aligned}
\end{equation*}
In the last display above, we have the sum of centered i.i.d random variables. Moreover, they are bounded, with each summand $p_i - 1 \leq  p_i - \mathbf{1}\{s(t+1) = i \} \leq p_i$ almost surely. By Hoeffding's inequality, we have
\begin{equation*}
    \begin{aligned}
        \PS\left(\sum_{t=0}^{T-1} \left(p_i - \mathbf{1}\{s(t+1) = i \} \right) > \frac{\overline t T}{2M_T\epsilon_{\overline K}} \right) &\leq \mathbb{P}\left(\sum_{t=0}^{T-1} \left(p_i - \mathbf{1}\{s(t+1) = i \} \right) > \frac{\overline t T}{2M_T\epsilon_{\overline K}} \right) \\
        &\leq \exp\left(-\frac{2\overline t^2 T^2}{4M_T^2 \epsilon_{\overline K}^2} \times \frac{1}{\sum_{t=0}^{T-1} (-1)^2}\right) \\
        &= \exp\left(-\frac{\overline t^2 T}{2M_T^2 \epsilon_{\overline K}^2} \right).
    \end{aligned}
\end{equation*}
Setting the RHS to at most $\delta$, and solving for $\overline t$, we obtain with probability at least $1 - \delta$, 
\begin{equation*}
    \begin{aligned}
        \sup_{\overline f \in \overline{\mathcal{F}}_r}\left[ \gap(\overline f) - \gap(\pi_{\overline K}(\overline f))\right] \leq  \sqrt{2}M_T r2^{-\overline K} \sqrt{\frac{\log(1/\delta)}{T}}.
    \end{aligned}
\end{equation*}

\newpage
\section{Proof of Corollaries}\label{sec:instantiations}
We now instantiate our general results on different classes of functions, to get a clear view on their respective rates of convergence. We will focus on the linear and Hölder class of functions. 
\subsection{Proof of Corollary \ref{cor:holder-rates}}
We first address the envelope $F(x) = \sup_{g \in \mathcal{F}_i^*} \|g(x)\|_2$. Since $g = f - f_i^*, f \in \mathcal{F}$, we have 
\begin{equation}\label{eq:g-shift-bound}
    \|g(x)\|_2 \leq \|f(x)\|_2 + \|f_i^*(x)\|_2.
\end{equation}
Assuming both $f$ and $f_i^*$ belongs to $\mathcal{S}(\alpha, \rho, B)$, we have 
\begin{equation*}
    F(x) \leq C_{\alpha, \rho, B}(1 + \|x\|_2).
\end{equation*}

This is because by Jensen's inequality and the fact that $\varepsilon$ is centered, we have
\begin{align*}
    \exp(\alpha \|f(x)\|_2^2) &= \exp(\alpha \|\mathbf{E}[f(x) + \varepsilon\|_2^2] \\
    &\leq \mathbf{E}[\exp(\alpha \|f(x) + \varepsilon\|_2^2] \\
    &\leq \rho\exp(\alpha \|x\|_2^2) + B \\
    &\leq (\rho + B)\exp(\alpha \|x\|_2^2),
\end{align*}
which after taking logarithms on both sides of the last display and using $\sqrt{a + b} \leq \sqrt{a} + \sqrt{b}$, 
\begin{align*}
    \|f(x)\|_2 \leq \|x\|_2 + \underbrace{\sqrt{\log(\rho + B)/\alpha}}_{c_{\alpha,\rho,B}}.
\end{align*}
Choosing $C_{\alpha, \rho, B} := \max\{1, c_{\alpha, \rho, B}\}$ and taking supremums on both sides of the last display, we have by \eqref{eq:g-shift-bound}, 
\begin{equation*}
    F(x) \leq C_{\alpha, \rho, B}(1 + \|x\|_2).
\end{equation*}

Thus, the envelope is upper bounded by
\begin{equation*}
    \begin{aligned}
        I_T(F) &= \frac{1}{T}\sum_{t=0}^{T-1}\mathbb{E}[F^2(X_t) \mathbf{1}\{\|X_t\|_2 > R_T \}]  \\
        &\leq \frac{C_{\alpha, \rho, B}}{T}\sum_{t=0}^{T-1}\mathbb{E}[(1 + \|X_t\|_2^2)\mathbf{1}\{\|X_t\|_2^2 > R_T^2\}].
    \end{aligned}
\end{equation*}

We now bound the RHS of the last display by a simple reduction. Note that for any $u > 0$ and $\eta > 0$, we have
\begin{equation*}
    \begin{aligned}
        \mathbf{1}\{u > R_T^2\} &= \mathbf{1}\{u - R_T^2 > 0 \} \leq 1 \leq \exp(\eta(u- R_T^2)),
    \end{aligned}
\end{equation*}
which implies that for any $0 < \eta < \alpha$, we have 
\begin{equation*}
\begin{aligned}
   (1+u) \mathbf{1}\{u > R_T^2\} &\leq \exp(\eta u)(1+u)\exp(-\eta R_T^2), \quad u, \eta > 0 \\
   &\leq (1+u)\exp(-(\alpha - \eta) u)\exp(\alpha u)\exp(-\eta R_T^2) \\
   &\leq C_{\alpha, \eta}\exp(\alpha u)\exp(-\eta R_T^2), \quad C_{\alpha, \eta} < \infty.
\end{aligned}
\end{equation*}
Thus, for $u = \|X_t\|_2^2$, we have by Claim \ref{cl: uniform-exp-moment},
\begin{equation*}
\begin{aligned}
    \mathbb{E}[(1+\|X_t\|_2^2) \mathbf{1}\{\|X_t\|_2^2 > R_T^2\}] &\leq C_{\alpha, \eta}\exp(-\eta R_T^2)\mathbf{E}[\exp(\alpha \|X_t\|_2^2)] \\
    &\leq C_{\alpha, \eta}K_{B, \rho}\exp(-\eta R_T^2) \\
    &= C_{\alpha, \eta, B, \rho}\exp(-\eta R_T^2).
\end{aligned}
\end{equation*}
Thus, if $R_T \gtrsim C'_{\alpha, \eta, B, \rho}\sqrt{\log(T)}$ for a suitably large $C'_{\alpha, \eta, B, \rho} > 0$, we have
\begin{equation*}
    \mathbb{E}[(1+\|X_t\|_2^2) \mathbf{1}\{\|X_t\|_2^2 > R_T^2\}] \lesssim \exp(-\log(T)) = T^{-1},
\end{equation*}
which is negligible with respect to the other terms in the critical radius fixed point equation.

The metric entropy of the localized class can be bounded by the metric entropy of the localized Hölder class $\mathcal{H}_T^\beta(L)$, which is given by \cite{nickl2007}
    \begin{equation*}
\log\mathcal{N}\left(\overline{\mathcal{F}}_r, \epsilon, \|.\|_{\infty, \mathcal{B}_R} \right) \lesssim R_T^{d}\left(\frac{L_T}{\epsilon} \right)^{\frac{d}{\beta}}.
    \end{equation*}
Let $\alpha_f = d/2\beta$, $C_T = R_T^{\alpha}$, and $\underline \gamma = r$. We now split the analysis into two different regimes. 

\textbf{Case I: $2\beta > d$}. In this smooth regime, we have $\alpha_f < 1$. Thus, the entropy integral is given by
\begin{equation*}
     \int_{\gamma_f}^{r} \epsilon^{-\alpha_f} d\epsilon = \int_{\gamma_f}^{r} \epsilon^{-\alpha_f} d\epsilon = \frac{1}{1-\alpha_f}\left[r^{1-\alpha_f} - \gamma_f^{1-\alpha_f} \right].
\end{equation*}
Since the integral above converges, we can take the limit $\gamma_f \rightarrow 0$ so the approximation at the fine scale goes to zero. By the arguments in the envelope bound above, we have for $g = f - f_i^*$, 
\begin{equation*}
    \|g(x)\|_2 \leq 2C_{\alpha, \rho, B}(1 + R_T),
\end{equation*}
and taking supremums over $x \in \mathcal{B}_T, g \in \mathcal{F}_i^*$, 
\begin{equation*}
    M_T \leq 2C_{\alpha, \rho, B}(1 + R_T) \lesssim 1 + R_T \asymp \sqrt{\log(T)}.
\end{equation*}
Combined with $L_T \leq C_L \log^{q_L} (eT)$, the fixed point equation in Proposition \ref{prop: ulln-master} reduces to 
\begin{equation*}
\begin{aligned}
    (r^*)^2 &\asymp M_T\sqrt{\frac{S}{\mu_i}} \int_{r\gamma_f}^{r/2} \sqrt{\log\mathcal{N}\left(\overline{\mathcal{F}}_r, \epsilon, \|.\|_{\infty, \mathcal{B}_T} \right)} d\epsilon \\
    &\asymp M_T\sqrt{\frac{S}{\mu_i}}R_T^{d/2}L^{\alpha_f}\int_0^{r^*} 
   \epsilon^{-\alpha_f} d\epsilon \\
    &\asymp M_T\sqrt{\frac{S}{\mu_i}}\log^{d/4}(T) L^{\alpha_f} (r^*)^{1 - \alpha_f},
    \\ \implies  (r^*)^{\alpha_f + 1} &\asymp \operatorname{polylog}(T)\sqrt{\frac{S}{\mu_i}}.
\end{aligned}
\end{equation*}
Since $\alpha_f + 1 = (d+2\beta)/2\beta$, we have 
\begin{equation*}
\begin{aligned}
  r^* &\asymp \operatorname{polylog}(T)S^{\frac{\beta}{2\beta+d}}\mu_i^{-\frac{\beta}{2\beta+d}}.
\end{aligned}
\end{equation*}
Substituting this choice of $r^*$, the martingale offset bound has order 
\begin{equation*}
    \begin{aligned}
       \int_0^{r^*} \sqrt{\frac{J_T(\epsilon)}{\mu_i}} d \epsilon 
        &\asymp \frac{\operatorname{polylog}(T)}{\sqrt{\mu_i}} (r^*)^{1-\alpha_f} \\
        &\asymp \frac{\operatorname{polylog}(T)}{\sqrt{\mu_i}} \left[\left(\sqrt{S}\right)^{\frac{2\beta}{2\beta+d}}\mu_i^{-\frac{\beta}{2\beta+d}} \right]^{\frac{2\beta-d}{2\beta}}.
    \end{aligned}
\end{equation*}
We focus on the exponent on the $\mu_i$ term. Firstly, $-\beta/(2\beta+d) \times (2\beta-d)/2\beta = -(2\beta-d)/[2(2\beta+d)]$. Writing $-1/2 = 2\beta + d/[2(2\beta+d)]$, we have $-(2\beta-d)/[2(2\beta+d)] - (2\beta+d)/[2(2\beta+d)] = -2\beta/[2\beta+d]$. Thus, the martingale offset bound has order
\begin{equation*}
    \begin{aligned}
    \operatorname{polylog}(T) S^{\frac{2\beta-d}{2(2\beta+d)}} \mu_i^{-\frac{2\beta}{2\beta+d}}.
    \end{aligned}
\end{equation*}
On the other hand, we have 
\begin{equation*}
    (r^*)^2 \asymp  \operatorname{polylog}(T)\left(\sqrt{S}\right)^{\frac{4\beta}{2\beta+d}}\mu_i^{-\frac{2\beta}{2\beta+d}} = \operatorname{polylog}(T)S^{\frac{2\beta}{2\beta+d}}\mu_i^{-\frac{2\beta}{2\beta+d}}.
\end{equation*}
Thus by Theorem \ref{thm:final-risk}, we have with probability at least $1 - \delta$
\begin{equation*}
    R_{T,i}^{cond}(\widehat f_i) \lesssim \operatorname{polylog}(T)
    \left(\frac{S}{\mu_i}\right)^{\frac{2\beta}{2\beta + d}}.
\end{equation*}

\textbf{Case II: $2\beta < d$}. In the rough regime, we have $1 - \alpha_f < 0$. The entropy integral is given by
\begin{equation*}
\begin{aligned}
    \int_{r\gamma_f}^{r/2} \sqrt{J_T(\epsilon)} d\epsilon &\lesssim R_T^{d/2}L_T^{\alpha_f}\int_{r\gamma_f}^{r/2} \epsilon^{-\alpha_f} d\epsilon  \\
    &= R_T^{d/2}L_T^{\alpha_f} \frac{(r\gamma_f)^{1 - \alpha_f} - (r/2)^{1-\alpha_f}}{\alpha_f - 1} \\
    &\lesssim R_T^{d/2}L_T^{\alpha_f} (r\gamma_f)^{1 - \alpha_f}.
\end{aligned}
\end{equation*}
Thus, the fixed point equation in Proposition \ref{prop: ulln-master} reduces to 
\begin{equation*}
    (r^*)^2 \asymp M_T R_T^{d/2}L_T^{\alpha_f} \sqrt{\frac{S}{\mu_i}} (r\gamma_f)^{1 - \alpha_f} + M_T r \gamma_f,
\end{equation*}
where we note the presence of the fine scale approximation error above. The RHS of the equation above is minimized by balancing 
\begin{equation*}
    \begin{aligned}
        R_T^{d/2}L_T^{\alpha_f} \sqrt{\frac{S}{\mu_i}} (r\gamma_f)^{1 - \alpha_f} \asymp r\gamma_f \\
        \iff (r\gamma_f)^{\alpha_f} \asymp R_T^{d/2}L_T^{\alpha_f} \sqrt{\frac{S}{\mu_i}} \\
        \iff r\gamma_f \asymp \left(R_T^{d/2}L_T^{\alpha_f} \sqrt{\frac{S}{\mu_i}}\right)^{1/\alpha_f}.
    \end{aligned}
\end{equation*}
Substituting the fine scale chain into the critical radius equation, we obtain 
\begin{equation*}
    \begin{aligned}
        (r^*)^2 &\asymp M_T \left(R_T^{d/2}L_T^{\alpha_f}\sqrt{\frac{S}{\mu_i}} \right)^{1/\alpha_f} \\
        &= M_TR_T^{\beta}L_T \left(\frac{S}{\mu_i} \right)^{\beta/d} \\
        &\asymp \operatorname{polylog}(T) \left(\frac{S}{\mu_i} \right)^{\beta/d}.
    \end{aligned}
\end{equation*}
On the other hand, the martingale offset bound has order 
\begin{equation*}
 \frac{R_T^{d} L_T^{2\alpha_f}}{\mu_i}\gamma_f^{1 - 2\alpha_f} + \sqrt{d}\gamma_f,
\end{equation*}
since the dominating term is given by the lower cutoff $\gamma_f$. Balancing the two terms above, we obtain
\begin{equation*}
    \gamma_f^* \asymp \left(\frac{R_T^d L_T^{2\alpha_f}}{\sqrt{d}\mu_i} \right)^{1/2\alpha_f},
\end{equation*}
which after substituting back gives us the fine scale error of order 
\begin{equation*}
    d^{(d-\beta)/2d}R_T^{\beta}L_T \mu_i^{-\beta/d} \leq \sqrt{d}R_T^{\beta}L_T \mu_i^{-\beta/d} \asymp \operatorname{polylog(T)}\sqrt{d}\mu_i^{-\beta/d}, 
\end{equation*}
and by Theorem \ref{thm:final-risk}, we obtain with probability at least $1 - \delta$
\begin{equation*}
    R_{T,i}^{cond}(\widehat f_i) \lesssim \operatorname{polylog}(T)\left[\sqrt{d} + \left(\frac{S}{\mu_i} \right)^{\beta/d} \right].
\end{equation*}
Finally, for the boundary case where $2\beta = d$, we have $\int_{\gamma_f}^{\gamma_c} \epsilon^{-1} d\epsilon = \log(\gamma_c / \gamma_f)$, and a similar calculation yields the same rate $\operatorname{polylog(T)}\mu_i^{-\beta/d}$.

\subsection{Proof of Corollary \ref{cor:linear-rates}}
Using similar arguments as in Corollary \ref{cor:holder-rates}, the envelope function is of order at most $T^{-1}$. The metric entropy is given by 
    \begin{equation*}
\log\mathcal{N}\left(\overline{\mathcal{F}}_r, \epsilon, \|.\|_{\infty, \mathcal{B}_R} \right) \asymp d^2\log\left(\frac{r}{\epsilon} \right).
    \end{equation*}
The entropy integral is thus
\begin{equation*}
    \int_{r\gamma_f}^{r/2}\sqrt{d^2 \log\left(\frac{r}{\epsilon} \right)} d\epsilon = d\int_{r\gamma_f}^{ r/2}\sqrt{\log\left(\frac{r}{\epsilon} \right)} d\epsilon.
\end{equation*}
Since the integral converges, we can let $\gamma_f \rightarrow 0$. In this case, the fine  scale approximation error goes to zero. Using a change-of-variable argument, by setting $u = \epsilon / r$,  we have 
\begin{equation*}
\begin{aligned}
    \int_0^r \sqrt{\log\left(\frac{r}{\epsilon} \right)} d\epsilon &= r\int_0^1 \sqrt{\log\left( \frac{1}{u}\right)} du  \\
    &= r \int_0^1 \sqrt{-\log(u)} du \\
    &= r\sqrt{\frac{\pi}{2}}.
\end{aligned}
\end{equation*}
Thus, the fixed point equation is given by
\begin{equation*}
    \begin{aligned}
        (r^*)^2 &\asymp \sqrt{\frac{S}{\mu_i}}\int_0^{r^*}\sqrt{\log\left(\mathcal{N}(\overline{\mathcal{F}}_r, \epsilon, \|.\|_{\infty, \mathcal{B}_R} \right)} d\epsilon \\
        &\asymp d\sqrt{\frac{S}{\mu_i}}r^* \\
        \implies r^* &\asymp d\sqrt{\frac{S}{\mu_i}}.
    \end{aligned}
\end{equation*}
With the choice of $r^*$ according to the last display, the martingale offset bound has order 
\begin{equation*}
\begin{aligned}
\int_0^{r^*} \sqrt{\frac{J_T(\epsilon)}{\mu_i}} d\epsilon &\asymp \frac{d}{\sqrt{\mu_i}}\int_0^{r^*} \sqrt{\log(r/\epsilon)} d\epsilon \\
&\asymp \frac{d}{\sqrt{\mu_i}}r^* \\
&\asymp \frac{d^2\sqrt{S}}{\mu_i}.
\end{aligned}
\end{equation*}
By Theorem \ref{thm:final-risk}, we have with probability at least $1 - \delta$,
\begin{equation*}
\begin{aligned}
    R_{T,i}^{cond}(\widehat f_i) \lesssim \frac{d^2 \sqrt{S}}{\mu_i} + \frac{d^2 S}{\mu_i} 
    \asymp \frac{d^2 S}{\mu_i}. 
\end{aligned}
\end{equation*}
To prove the second part, we note that $f_i^*(x) = A_i^* x$ and $\est{f}_i(x) = \est{A}_i x$. Then  starting with the definition of $R_{T,i}^{cond}(\widehat f_i)$, in \eqref{eq:cond-pop-risk}, we obtain
\begin{align*}
     R_{T,i}^{cond}(\widehat f_i) 
     &= \frac{1}{Tp_i}\sum_{t=0}^{T-1}\mathbb{E}\left[\norm{(A_i^* - \est{A}_i) \widetilde X_t}_2^2 \mathbf{1}\{\widetilde s(t+1) = i \} \mid \mathcal{D}_T \right] \\
     &= \frac{1}{Tp_i}\sum_{t=0}^{T-1}\mathbb{E}\left[\norm{(A_i^* - \est{A}_i) A^*_{\tilde{s}(t)} \widetilde X_{t-1} + (A_i^* - \est{A}_i) \widetilde \varepsilon_t}_2^2 \mathbf{1}\{\widetilde s(t+1) = i \} \mid \mathcal{D}_T \right] \\
     &= \frac{1}{Tp_i}\sum_{t=0}^{T-1}\mathbb{E}\left[\norm{(A_i^* - \est{A}_i) A^*_{\tilde{s}(t)} \widetilde X_{t-1}}_2^2 \mathbf{1}\{\widetilde s(t+1) = i \} \mid \mathcal{D}_T \right] \\
     &+ \frac{1}{Tp_i}\sum_{t=0}^{T-1}\mathbb{E}\left[\norm{(A_i^* - \est{A}_i) \widetilde \varepsilon_t}_2^2 \mathbf{1}\{\widetilde s(t+1) = i \} \mid \mathcal{D}_T \right] \\
     &\geq \frac{1}{Tp_i}\sum_{t=0}^{T-1}\mathbb{E}\left[\norm{(A_i^* - \est{A}_i) \widetilde \varepsilon_t}_2^2 \mathbf{1}\{\widetilde s(t+1) = i \} \mid \mathcal{D}_T \right] \\
     &= \frac{1}{T} \sum_{t=0}^{T-1}\mathbb{E}\left[\norm{(A_i^* - \est{A}_i) \widetilde \varepsilon_t}_2^2 \mid \mathcal{D}_T \right] \\
     &= \norm{A_i^* - \est{A}_i}_F^2.
\end{align*}
This completes the proof.

\newpage
\section{Experiments}
 
We illustrate the numerical properties of the empirical risk minimizer in the linear and Hölder case, as seen in our corollaries. The Python code for the experiments can be found at \url{https://github.com/sunnywang93/switchnds}.

\subsection{Linear setting}

We explore the simple case of $K = 2$ regimes and $d = 2$ dimensions. Trajectories were simulated according to model \eqref{eq:model}, which simplifies in the linear setting to 
\begin{equation*}
    X_{t+1} = A^*_{s(t+1)}X_t + \varepsilon_t, \quad 0 \leq t \leq T-1,
\end{equation*}
with $X_0 = [0, 0]^\top$. The excitation sequence was chosen to be Gaussian random variables $(\varepsilon_t) \sim \mathcal{N}(0, \sigma^2 I_d)$. The state matrices were chosen to be 
\begin{equation*}
    A_0 = \begin{bmatrix}
        0.4 & 0.3 \\
        -0.1 & 0.5
    \end{bmatrix}, \quad
    A_1 = \begin{bmatrix}
        \cos(\theta) & -\sin(\theta) \\
        \sin(\theta) & \cos(\theta)
    \end{bmatrix},
\end{equation*}
respectively, with $\theta = 30^\circ$. Note that $\|A_0\|_2 \approx 0.6$, so Assumption \ref{ass:1} is satisfied, whereas $\|A_1\|_2 = 1$. The latter choice is made to investigate the numerical behavior in the presence of a marginally stable regime. Moreover, $A_1$ has a natural interpretation of being a rotation matrix. The trajectory lengths $T \in \{200, 500, 1000\}$, probabilities $(p_0, p_1) \in \{(0.1, 0.9), (0.5, 0.5), (0.9, 0.1) \}$ and noise levels $\sigma \in \{0.1, 0.5 \}$ were investigated, with all $18$ combinations of these parameters explored. $500$ replications were run for each configuration, leading to $9000$ simulation runs.

In each replication, estimation of the matrices $A_i, i \in \{0, 1\}$ were performed using least-squares with the regime-specific observations of the training trajectory. An independent test trajectory drawn from the same law is then generated, and the conditional excess prediction error computed, of the form
\begin{equation}\label{eq:cond-risk-linear-sim}
    \widehat L_{n_i, i}(\widehat A_i - A_i^*) = \frac{1}{\widetilde n_i}\sum_{t=0}^{T-1} \|\widehat A_i \widetilde X_t - A_i \widetilde X_t \|_2^2 \mathbf{1}\{\widetilde s(t+1) = i \}.
\end{equation}
Note that looking at the quantiles in \eqref{eq:cond-risk-linear-sim} gives us an estimate of the $R_{T,i}^{cond}(\widehat f_i)$, our quantity of interest. This is because for any measurable $g$, $\lemp(g) = n_i/T \widehat L_{n_i,i}(g)$ is the empirical proxy $\ltrans(g)$, and since $n_i/T \approx p_i$ with high probability which implies $\lemp(g) \approx p_i \widehat L_{n_i,i}(g)$, one has that $\widehat L_{n_i,i}(g)$ is the empirical analogue of $1/p_i \ltrans(g) = L_{T,i}^{cond}(g)$.

Simulation results can be seen in Figure \ref{fig:lin-experiments}, where log-log plots on trajectory lengths vs the prediction error are provided. In this case, note that if the error $L_{n_i, i} \approx CT^{-a}$, then the log-log plots gives $\log(L_{n_i, i}) \approx \log(C) - a \log(T)$, so $a$ serves as a proxy for the rate of convergence. The solid lines in each group represents the median prediction error, which is obtained by fitting a linear regression on the trajectory lengths. The shaded region represents the $0.05$-th to the $0.95$-th quantile. The estimated slope $\widehat a$ is very close to 1, implying a convergence rate of $O_{\mathbb{P}}(T^{-1})$ for a fixed regime $i$. Moreover, whenever the regime probability differs, the log errors only differ by a constant shift, as predicted by the theory, with the higher probability regime naturally enjoying lower prediction errors. Lastly, the errors move in the same direction as the noise levels, as expected.

\begin{figure}
    \centering
    \includegraphics[width=1\linewidth]{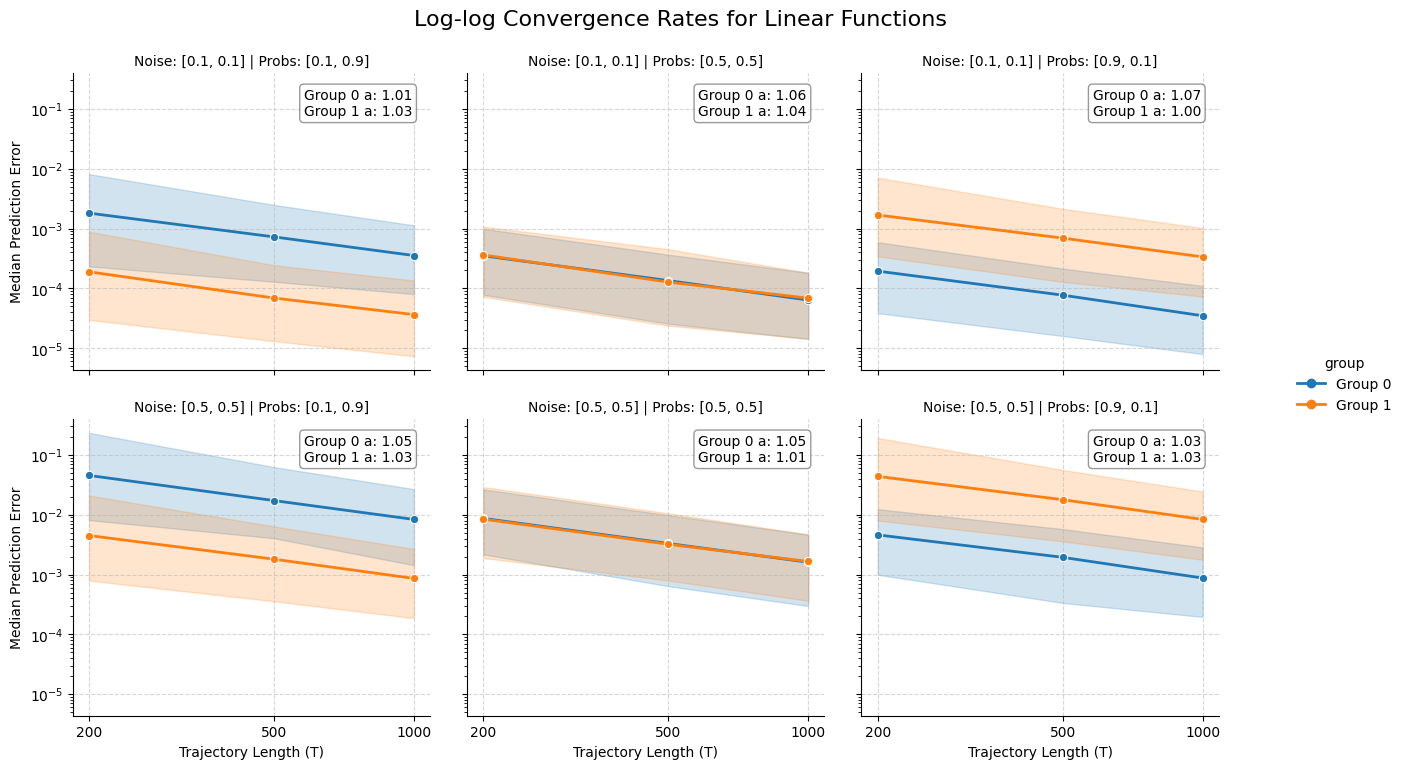}
    \caption{Log-log plots for the linear case.}
    \label{fig:lin-experiments}
\end{figure}

\subsection{Hölder setting}
In the non-linear setting, we restrict ourselves to dimension $d = 1$, again with $K = 2$ regimes. The true function was taken to be the fractional power function, given by
\begin{equation*}
    f_i^*(x) = c_i \operatorname{sgn}(x) |x|^{\beta_i}, \quad \beta_i \in (0, 1), \quad i = 0, 1,
\end{equation*}
which is clearly $\beta$-Hölder, and satisfies the Lyapunov type condition in Assumption \ref{ass:1} (see Claim \ref{cl:simu-lyapunov}). In the smooth regime, $c_1 = 0.7$, $c_2 = -0.6$, and $\beta_0 = 0.8$, $\beta_1 = 0.6$ were selected. The trajectory lengths $T \in \{200, 500, 1000, 2000, 4000\}$, probability $(p_0, p_1) \in \{(0.1, 0.9), (0.5, 0.5), (0.9, 0.1)\}$, and noise levels $\sigma \in \{0.1, 0.5\}$ were investigated, with all $30$ combinations explored. 500 replications were run for each configuration, leading to $15000$ simulation runs. 

In each replication, training trajectories were simulated and used to produce the estimates $\widehat f_i$ of the true functions $f_i^*$, for $i = 0, 1$. A B-splines basis expansion (with cubic splines) was used to estimate $f_i^*$, of the form
\begin{equation*}
    \widehat f_i(.) = \sum_{j=1}^{J_i} \theta_j \varphi_j(.).
\end{equation*}
The number of basis functions $J_i$ were selected as
\begin{equation*}
    J^*_i \asymp  \mu_i^{\frac{1}{2\beta_i + 1}}, \quad i = 0, 1,
\end{equation*}
according to the familiar optimal rate for splines estimation, assuming the Hölder exponent $\beta_i, i = 0, 1$ are known. A floor of at least 4 basis functions were taken to prevent numerical issues, which might occur with shorter trajectory lengths $T$. 

Analogous to the linear case, the prediction errors were evaluated on a fresh test trajectory drawn from the same law, computed according to
\begin{equation*}
   \widehat  L_{n_i, i}\left(f_i^* - \widehat f_i \right) = \frac{1}{\widetilde n_i}\sum_{t=0}^{T-1} \left|\widehat f_i(\widetilde X_t) - \ f_i^*(\widetilde X_t)\right|^2.
\end{equation*}
Simulation results can be seen in Figure \ref{fig:nonlin-experiments}. Again, a polynomial fit was performed on the log prediction errors against the trajectory lengths, with the solid line displaying the median prediction error and the shaded region containing the $0.05$-th to the $0.95$-th quantiles. The experiment results largely aligns with the theory, with estimated slopes mostly between $0.6$ and $0.9$ for a fixed regime $i$, which is in align with our choices of the Hölder exponent. Moreover, the regime probabilities scales inversely with the prediction error, as expected. Finally, the prediction errors are larger with higher noise, and vice-versa.

\begin{figure}
    \centering
    \includegraphics[width=1\linewidth]{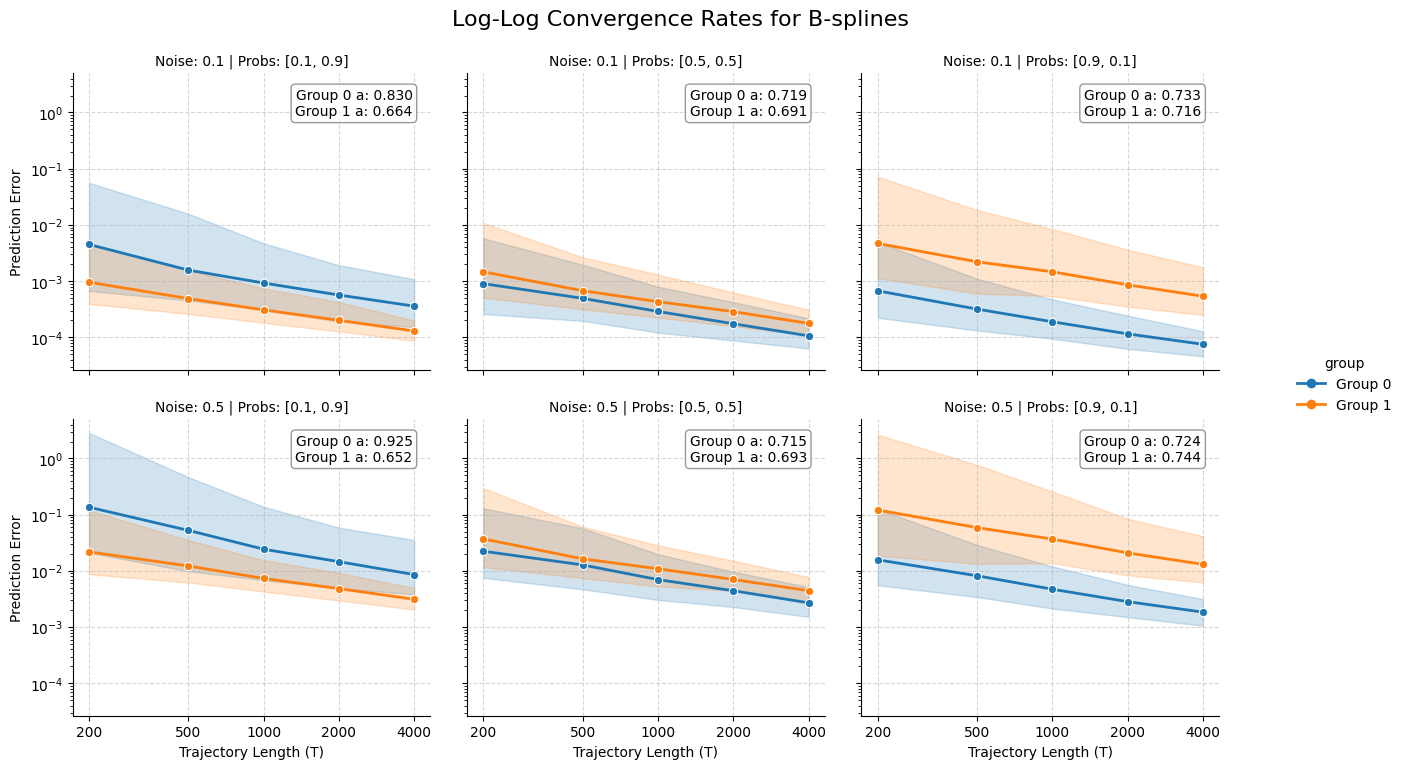}
    \caption{Log-log plots for the Hölder regime.}
    \label{fig:nonlin-experiments}
\end{figure}


\section*{Acknowledgments}
This work was supported by a Nanyang Associate Professrship (NAP) award from NTU Singapore.

\newpage
\appendix

\section{Probabilistic Tools}\label{app:appendix-prob}

\begin{theorem}[\cite{ville1939}]\label{thm:ville}
Let $X_0, X_1, \dots$ be a non-negative supermartingale. Then for any real number $a > 0$, it holds that
\begin{equation*}
    \mathbb{P}(\sup_{n \geq 0} X_n \geq a) \leq \frac{\mathbb{E}[X_0]}{a}.
\end{equation*}
\end{theorem}

\begin{proposition}\label{prop:unif-pena}
    Let $\{d_t, \mathcal{F}_t\}_{t \geq 1}$ be a martingale difference sequence. Suppose there exists $\sigma^2_t  > 0, c > 0$ such that $\mathbb{E}[d_t^2 \mid \mathcal{F}_{t-1}] \leq \sigma^2_t$ almost surely, and the following Bernstein condition holds for all $p \geq 2$:
    \begin{equation}\label{eq:bern-mds}
        \mathbb{E}[|d_t|^p \mid \mathcal{F}_{t-1}] \leq \frac{p!}{2}\sigma^2_t c^{p-2}.
    \end{equation}
Let $M_n = \sum_{t=1}^n d_t$ be a martingale, and $V_n^2 = \sum_{t=1}^n \sigma^2_t$ be the quadratic variation. Then for all $x, y > 0$, we have 
\begin{equation*}
    \mathbb{P}\left(\max_{1 \leq j \leq n}M_j > x, V_n^2 \leq y \right) \leq \exp\left(-\frac{x^2}{2(y + cx)} \right).
\end{equation*}
\end{proposition}
\begin{proof}

\textit{Step 1:Moment generating function (mgf) bounds. }

    Fix $0 \leq \lambda < c^{-1}$. Using the mgf, together with the Bernstein condition and variance bound, we have 
    \begin{equation*}
        \begin{aligned}
            \mathbb{E}\left[\exp\left(\lambda d_t \right) \mid \mathcal{F}_{t-1} \right] &= 1 + \frac{\lambda^2}{2}\mathbb{E}[d_t^2 \mid \mathcal{F}_{t-1}] + \sum_{p \geq 3}\frac{\lambda^p}{p!}\mathbb{E}[d_t^p \mid \mathcal{F}_{t-1}] \\
            &\leq 1 + \frac{\lambda^2 \sigma^2_t}{2} + \sum_{p \geq 3}\frac{\lambda^2\sigma^2_t}{2}(\lambda c)^{p-2} \\
            &= 1 + \frac{\lambda^2 \sigma^2_t}{2}\sum_{p \geq 2}(\lambda c)^{p-2} \\
            &\leq 1 + \frac{\lambda^2 \sigma^2_t}{2} \frac{1}{1 - \lambda c} \\
            &\leq \exp\left(\frac{\lambda^2 \sigma^2_t}{2(1 - \lambda c)} \right).
        \end{aligned}
    \end{equation*}

\textit{Step 2: Constructing the supermartingale.}

Define
\begin{equation*}
    Z_j := \exp\left(\lambda M_j - \frac{\lambda^2 V_j^2}{2(1 - \lambda c)} \right).
\end{equation*}
Since $M_j = M_{j-1} + d_j$ and $V_j^2 = V_{j-1}^2 + \sigma_j^2$, we have 
\begin{equation*}
    \begin{aligned}
        Z_j &= \exp\left(\lambda\{M_{j-1} + d_j \} - \frac{\lambda^2\{V^2_{j-1} - \sigma_j^2 \}}{2(1 - \lambda c)} \right) \\
        &= Z_{j-1}\exp\left(\lambda d_j - \frac{\lambda^2\sigma_j^2}{2(1 - \lambda c)} \right).
    \end{aligned}
\end{equation*}
Taking conditional expectations and using the mgf bounds from Step 1, we have 
\begin{equation*}
    \begin{aligned}
        \mathbb{E}\left[Z_j \mid \mathcal{F}_{j-1} \right] &= Z_{j-1}\exp\left(-\frac{\lambda^2 \sigma^2_j}{2(1 - \lambda c)} \right)\mathbb{E}\left[\exp\left(\lambda d_j \mid \mathcal{F}_{j-1} \right) \right] \\
        &\leq Z_{j-1}\exp\left(-\frac{\lambda^2 \sigma_j^2}{2(1 - \lambda c)} \right)\exp\left(\frac{\lambda^2 \sigma^2_t}{2(1 - \lambda c)} \right) \\
        &= Z_{j-1},
    \end{aligned}
\end{equation*}
so $Z_j$ is a non-negative supermartingale. By Ville's inequality, we have
\begin{equation*}
    \mathbb{P}\left(\max_{1 \leq j \leq n}Z_j \geq a \right) \leq \frac{\mathbb{E}[Z_1]}{a} \leq \frac{\mathbb{E}[Z_0]}{a} = \frac{1}{a}.
\end{equation*}

\textit{Step 3: Event inclusion.}

Let 
\begin{equation*}
    \mathcal{C} = \left\{\max_{1 \leq j \leq n} M_j \geq x, V_n^2\leq y \right\}.
\end{equation*}
Note that if $\mathcal{C}$ occurs, then there exists some $j^* \in \{1, \dots, n\}$ such that $M_{j^*} \geq x$. Since the variances are non-negative, we have $V^2_{j^*} \leq V_n^2 \leq y$. Now set
\begin{equation*}
    a = \exp\left(\lambda x - \frac{\lambda^2 y}{2(1 - \lambda c)} \right).
\end{equation*}
Then we have 
\begin{equation*}
    \begin{aligned}
        Z_j^* &= \exp\left(\lambda M_{j^*} - \frac{\lambda^2 V^2_{j^*}}{2(1 - \lambda c)} \right) \\
        &\geq \exp\left(\lambda x - \frac{\lambda^2 y}{2(1 - \lambda c)} \right) \\
        &= a.
    \end{aligned}
\end{equation*}
Since $\max_{1 \leq j \leq n} Z_j \geq Z_{j^*} \geq a$, we have that $\mathcal{C} \subseteq \{\max_{1 \leq j \leq n} Z_j \geq a \}$. Thus, we have
\begin{equation}\label{eq:mgf-bound-bf-opt}
    \begin{aligned}
        \mathbb{P}\left(\max_{1 \leq j \leq n} M_j \geq x, V_n^2 \leq y \right) &\leq \mathbb{P}\left(\max_{1 \leq j \leq n}Z_j \geq a \right) \\
        &\leq \frac{1}{a} \\
        &= \inf_{\lambda \in [0, 1/c)} \exp\left(-\left(\lambda x - \frac{\lambda^2 y}{2(1 - \lambda c)} \right) \right).
    \end{aligned}
\end{equation}

\textit{Step 4: Optimizing the mgf bound.}
Let 
\begin{equation*}
    f(\lambda ) = \lambda x - \frac{\lambda^2 y}{2(1 - \lambda c)}.
\end{equation*}
Taking the derivative with respect to $\lambda$, we have 
\begin{equation*}
    f^\prime(\lambda) = x - \frac{y\lambda}{1 - c\lambda}\frac{2 - c\lambda}{2(1 - c\lambda)}.
\end{equation*}
Since $c\lambda < 1$, we have 
\begin{equation*}
    \begin{aligned}
        2 - 2c\lambda \leq 2 - c\lambda \ 
        \iff \ \frac{2 - c\lambda}{2 - 2c\lambda} \geq 1.
    \end{aligned}
\end{equation*}
This implies that 
\begin{equation*}
\begin{aligned}
    \frac{y\lambda(2 - c\lambda)}{2(1 - c\lambda)^2} \geq \frac{y\lambda}{1 - c\lambda} \\
    \implies f^\prime(\lambda) \leq x - \frac{y\lambda}{1 - c\lambda} =: \widetilde f^\prime(\lambda).
\end{aligned}
\end{equation*}
Setting $\widetilde f^\prime(\lambda) = 0$, we obtain
\begin{equation*}
    \widetilde \lambda^* = \frac{x}{y + cx}.
\end{equation*}
By a simple calculation, the second derivative is given by
\begin{equation*}
     f^{\prime\prime}(\lambda) = -\frac{y}{(1 - c\lambda)^3} < 0,
\end{equation*}
so any critical point $\lambda \in [0, 1/c)$ is a global maximum. Let $\lambda^*$ be the maximizer of $f(\lambda)$. Since $\widetilde \lambda^*$ gives a slightly suboptimal solution compared to $\lambda^*$, $\exp(-f(\lambda^*)) \leq \exp(-f(\widetilde \lambda^*))$. Substituting in the value of $\widetilde \lambda^*$ into $-f(\lambda)$, we have
\begin{equation*}
    \begin{aligned}
        -f(\widetilde \lambda^*) &= -\frac{x^2}{y + cx} + \frac{(x/y+cx)^2 y}{2(1 - cx/(y+cx))} \\
        &= -\frac{x^2}{y + cx} + \frac{x^2y/(y + cx)^2}{2(y/(y+cx))} \\
        &= -\frac{x^2}{2(y + cx)}.
    \end{aligned}
\end{equation*}
Finally, using the expression above for the RHS of \eqref{eq:mgf-bound-bf-opt}, we obtain
\begin{equation*}
     \mathbb{P}\left(\max_{1 \leq j \leq n} M_j \geq x, V_n^2 \leq y \right) \leq \exp\left(-\frac{x^2}{2(y + cx)} \right).
\end{equation*}
\end{proof}

\section{Dynamics, stability and examples}\label{app:appendix-dynamics}

\subsection{Localization from exponential drift}

\begin{lemma}\label{lem:traj-ball}
    Suppose that assumption \ref{ass:1} holds. Then with probability at least $1 - \delta$, it holds that
    \begin{equation*}
        \max_{0 \leq t \leq T-1}\|X_t\|_2 \leq \sqrt{\log(TC_{\rho, B}/\delta)/\alpha}. 
    \end{equation*}
\end{lemma}
\begin{proof}
    By the tower property and unrolling the recurrence in assumption \ref{ass:1}, we have
    \begin{equation*}
        \begin{aligned}
            \mathbb{E}[\exp(\alpha \|X_{t}\|_2^2) &\leq \rho \mathbb{E}[\exp(\alpha \|X_{t-1}\|_2^2)] + B \\
            &\leq \rho^2 \mathbb{E}[\exp(\alpha \|X_{t-2}\|_2^2)] + \rho B + B \\
            &\leq \rho^3 \mathbb{E}[\exp(\alpha \|X_{t-3}\|_2^2] + \rho^2B + \rho B + B \\
            &\vdots \\
            &\leq \rho^t \mathbb{E}[\exp(\alpha \|X_0\|_2^2] + B\sum_{j=0}^{t-1}\rho^j \\
            &= \rho^t + B\sum_{j=0}^{t-1}\rho^j.
        \end{aligned}
    \end{equation*}
Since $\rho < 1$, the series on the RHS above converges, and we have
\begin{equation}\label{eq:exp-tracj-unif}
    \mathbb{E}[\exp(\alpha \|X_t\|_2^2)] \leq C_{\rho, B},
\end{equation}
where $C_{\rho, B}$ is a constant depending only on $(\rho, B)$.

By the union bound and Markov's inequality, we have 
\begin{equation*}
    \begin{aligned}
        \mathbb{P}(\max_{0 \leq t \leq T-1}\|X_t\|_2 > u) &\leq \sum_{t=0}^{T-1}\mathbb{P}(\|X_t\|_2 > u) \\
        &= \sum_{t=0}^{T-1}\mathbb{P}(\exp(\alpha \|X_t\|_2^2) > \exp(\alpha u^2)) \\
        &\leq \sum_{t=0}^{T-1} \frac{\mathbb{E}[\exp(\alpha \|X_t\|_2^2)] }{\exp(\alpha u^2)} \\
        &\leq T C_{\rho, B} \exp(-\alpha u^2),
    \end{aligned}
\end{equation*}
where the last inequality is a consequence of \eqref{eq:exp-tracj-unif}. Setting the probability threshold to $\delta$, we obtain with probability at least $1 - \delta$, 
\begin{equation*}
  \max_{0 \leq t \leq T-1} \|X_t\|_2 \leq   \sqrt{\frac{1}{\alpha}\log\left(\frac{TC_{\rho, B}}{\delta} \right)},
\end{equation*}
as desired.
\end{proof}

\begin{definition}[Localized $\beta$-Hölder class]\label{def:local-Hölder}
Let $\beta = s + \gamma$, where $s \in \mathbb{N}_0$, $\gamma \in (0, 1)$, and $\mathcal{B} \subset \mathbb{R}^d$ to be a compact set. For a multi-index $j = (j_1, \dots, j_d) \in \mathbb{N}_0^d$, let 
\begin{equation*}
    |j| = j_1 + \dots + j_d, \qquad D^j f = \frac{\partial^{|j| } f}{\partial x_1^{j_1}\dots \partial x_d^{j_d}},
\end{equation*}
Define the class $C^{s,\gamma}(\mathcal{B}, \mathbb{R}^d)$ as 
\begin{equation*}
    C^{s;\gamma}(\mathcal{B}, \mathbb{R}^d) := \left\{f:\mathcal{B} \rightarrow \mathbb{R}^d \mid \|f\|_{C^{s,\gamma}(\mathcal{B})} < \infty \right\},
\end{equation*}
where the $C^{s, \gamma}$ norm on $\mathcal{B}$ is given by
\begin{equation*}
    \|f\|_{C^{s,\gamma}(\mathcal{ B})} := \sum_{|j| \leq s}\sup_{x \in \mathcal{B}} \|D^j f(x)\|_2 + \sum_{|j| = s}\sup_{\underset{x \neq y}{x, y \in \mathcal{B}}} \frac{\|D^j f(x) - D^j f(y) \|_2 }{\|x - y\|_2^{\gamma}}.
\end{equation*}
The localized Hölder class is given by 
\begin{equation*}
    \mathcal{H}^{\beta}(L, \mathcal{B}) := \left\{f \in C^{s,\gamma}(\mathcal{B}; \mathbb{R}^d): \|f\|_{C^{s,\gamma}(\mathcal{ B})} \leq L \right\}, \qquad \beta = s + \gamma.
\end{equation*}

\end{definition}
 
\subsection{Verification and examples for different dynamics}
\begin{proof}[Proof of Claim \ref{rmk:linear-stab}]
    For any $\delta > 0$, choose $\kappa_A := (1+\delta)\rho_A^2 < 1$, which is possible since $\rho_A < 1$. By the $\delta$-Young's inequality 
    $\|a+b\|_2^2 \leq (1+\delta)\|a\|_2^2 + (1 + 1/\delta)\|b\|_2^2$, we have for every $i \in \{1, \dots, K\}$
    \begin{equation*}
        \begin{aligned}
            \|X_{t+1}\|_2^2 &= \|A^*_i X_t + \varepsilon_{t+1}\|_2^2 \\
            &\leq (1+\delta) \|A^*_i X_t\|_2^2 + (1 + 1/\delta)\|\varepsilon_{t+1}\|_2^2 \\
            &\leq \kappa_A \|X_t\|_2^2 + (1 + 1/\delta)\|\varepsilon_{t+1}\|_2^2.
        \end{aligned}
    \end{equation*}
Since the inequality above holds for every mode $i$, we have after multiplication, exponentiation and taking conditional expectations
\begin{equation*}
    \begin{aligned}
\mathbb{E}\left[\exp\left(\alpha_A \|X_{t+1}\|_2^2 \right) \mid X_t = x \right] \leq \exp\left(\alpha_A \kappa_A \|x\|_2^2 \right)\mathbb{E}\left[\exp\left(\alpha_A (1 + 1/\delta)\|\varepsilon_{t+1}\|_2^2  \right)\right]
    \end{aligned}
\end{equation*}
Since $\varepsilon$ is sub-gaussian, there exists $c_0 > 0$ such that
\begin{equation*}
    \mathbb{E}[\exp(c_0 \|\varepsilon_{t+1}\|_2^2)] < \infty.
\end{equation*}
Choose $\alpha_A > 0$ sufficiently small such that 
\begin{equation*}
    \alpha_A (1 + 1/\delta) \leq c_0,
\end{equation*}
and define
\begin{equation*}
    C_A := \mathbb{E}[\exp(\alpha_A (1 + 1/\delta) \|\varepsilon_{t+1}\|_2^2)],
\end{equation*}
which implies
\begin{equation*}
    \mathbb{E}[\exp(\alpha_A \|X_{t+1}\|_2^2) \mid X_t = x] \leq C_A \exp(\alpha_A \kappa_A \|x\|_2^2).
\end{equation*}
Now choose any $\rho \in (0, 1)$, and define
\begin{equation*}
    B_A := \sup_{x \in \mathbb{R}^d}\left\{C_A \exp(\alpha_A \kappa_A \|x\|_2^2) - \rho \exp(\alpha_A \|x\|_2^2) \right\},
\end{equation*}
which is finite since $\kappa_A < 1$, so the RHS of the display above converges to $-\infty$ as $\|x\|_2 \rightarrow \infty$. Since the function inside the supremum is continuous, combined with the fact that it converges to negative infinity, it has a finite global supremum, and we have
\begin{equation*}
    C_A \exp(\alpha_A \kappa_A \|x\|_2^2) \leq \rho \exp(\alpha_A \|x\|_2^2) + B_A, 
\end{equation*}
and gathering facts, we obtain
\begin{equation*}
    \mathbb{E}[\exp(\alpha_A \|X_{t+1}\|_2^2) \mid X_t = x] \leq \rho \exp(\alpha_A \|x\|_2^2) + B_A,
\end{equation*}
so the exponential drift condition holds whenever $\max_{1\leq i \leq K}\|A_i^*\|_2 < 1$.

\end{proof}

\begin{claim}\label{cl:simu-lyapunov}
    Fix $d = 1$. The function 
    \begin{equation*}
        f^*_i(x) = c_i\operatorname{sgn}(x) |x|^{\beta_i}, \quad \beta_i \in (0, 1), \forall i \in [K],
    \end{equation*}
under the model \eqref{eq:model} where the $(\varepsilon_t)_{t \in \mathbb{Z}}$ are gaussian random variables satisfies Assumption \ref{ass:1}, with 
\begin{equation*}
    B := \max\left\{0, \max_{x \in [-R,R]}h(x) \right\} < \infty,
\end{equation*}
for some $R > 0$, and $h$ is the continuous function 
\begin{equation*}
      h(x) := \mathbb{E}[\exp(\alpha X^2_{t+1}) \mid X_t = x] - \rho \exp(\alpha x^2).
\end{equation*}
\end{claim}
\begin{proof}
Fix any $\alpha \in (0, 1/2)$, $\rho \in (0, 1)$. Under the model \eqref{eq:model}, conditionally on $X_t = x$ and $s(t+1) = i$, $X_{t+1}$ is a Gaussian random variable with mean $f_i^*(x)$ and unit variance. Recall that for a gaussian random variable $Y \sim \mathcal{N}(\mu, 1)$, one has 
\begin{equation*}
    \mathbb{E}[\exp(\alpha Y^2)] = \frac{1}{\sqrt{1 - 2\alpha}} \exp\left(\frac{\alpha \mu^2}{1 - 2\alpha} \right), \quad \alpha < 1/2.
\end{equation*}
Thus, we have 
\begin{equation*}
    \mathbb{E}[\exp(\alpha X^2_{t+1}) \mid X_t = x, s(t+1) = i] = \frac{1}{\sqrt{1-2\alpha}}\exp\left(\frac{\alpha c_i^2 |x|^{2\beta_i}}{1 - 2\alpha} \right), \quad \alpha < 1/2.
\end{equation*}
Since $s(t+1)$ is independent of $X_t$, the law of total expectation gives us 
\begin{equation*}
\begin{aligned}
    \mathbb{E}[\exp(\alpha X^2_{t+1}) \mid X_t = x] &= \sum_{i=1}^K p_i \mathbb{E}[\exp(\alpha X^2_{t+1}) \mid X_t = x, s(t+1) = i] \\
    &= \frac{1}{\sqrt{1 - 2\alpha}}\sum_{i=1}^K p_i \exp\left(\frac{\alpha c_i^2 |x|^{2\beta_i}}{1 - 2\alpha} \right).
\end{aligned}
\end{equation*}
Note that for every $i \in [K]$, since $\beta_i \in (0, 1)$, we have $|x|^{2\beta_i} = o(x^2)$, as $|x| \rightarrow \infty$. Consequently, this implies that 
\begin{equation*}
    \frac{\alpha c_i^2 |x|^{2\beta_i}}{1 - 2\alpha} - \alpha x^2 \longrightarrow -\infty, \quad \text{as } |x| \rightarrow \infty,
\end{equation*}
so we have 
\begin{equation*}
    \exp\left(\frac{\alpha c_i^2 |x|^{2\beta_i}}{1 - 2\alpha} - \alpha x^2 \right) = o(1).
\end{equation*}
Since we have a finite number of modes, it follows that
\begin{equation*}
    \frac{\mathbb{E}[\exp(\alpha X^2_{t+1}) \mid X_t = x]}{\exp(\alpha x^2)} = o(1), \quad \text{as } |x| \rightarrow \infty.
\end{equation*}

Define the continuous function
\begin{equation*}
    h(x) := \mathbb{E}[\exp(\alpha X^2_{t+1}) \mid X_t = x] - \rho \exp(\alpha x^2).
\end{equation*}
It remains to show that there exists a finite $B$. 

To this end, note that since $h(x) \rightarrow -\infty$ as $|x| \rightarrow \infty$, there exists $R > 0$ such that $h(x) \leq 0$ for every $|x| > R$. On the other hand, by the continuity of $h$, a finite maximum is attained on the interval $[-R, R]$ by the extreme value theorem. That is, there exists $x^* \in [-R, R]$ such that 
\begin{equation*}
    h(x^*) = \max_{x \in [-R,R]} h(x) < \infty.
\end{equation*}
Consequently, 
\begin{equation*}
    \sup_{x \in \mathbb{R}} h(x) \leq \max\left\{0, \max_{x \in [-R,R]} h(x)\right\} < \infty.
\end{equation*}
Moreover, since $h(x) \leq 0$ outside $[-R, R]$, we can directly define
\begin{equation*}
    B := \max\left\{0, \max_{x \in [-R,R]}h(x) \right\} < \infty,
\end{equation*}
so for every $x \in \mathbb{R}$, we have $h(x) \leq B$. Gathering facts, we have 
\begin{equation*}
    \mathbb{E}[\exp(\alpha X^2_{t+1}) \mid X_t = x )] \leq \rho \exp(\alpha x^2) + B, \qquad \forall x \in \mathbb{R}.
\end{equation*}

\end{proof}

%
\section{Technical Lemmas}\label{app:tech-lemmas}

\begin{claim}\label{cl:dist-eq-sum}
    Let $\varepsilon_t$ and $\widetilde \varepsilon_k$ be i.i.d copies of a centered sub-gaussian random variable $\varepsilon$ with unit variance, and $s(t)$ be i.i.d discrete random variables with pmf $p_i$ for all $i \in [K]$. Denote $\mathcal{I}_i$ to be the set of observations such that $s(t+1) = i$, and $n_i = |\mathcal{I}_i|$. Suppose that the trajectories follow model \eqref{eq:model}. Then, it holds that
    \begin{equation*}
         \sum_{t \in \mathcal{I}_i} \|\varepsilon_{t+1}\|_2 =\sum_{t=0}^{T-1} \|\varepsilon_{t+1}\|_2 \mathbf{1}\{s(t+1) = i \} \stackrel{d}{=} \sum_{k=1}^{n_i}\|\widetilde \varepsilon_{k}\|_2.
    \end{equation*}
\end{claim}
\begin{proof}
    Let $W_{t+1} = \|\varepsilon_{t+1}\|_2$ and $\widetilde W_k = \|\widetilde \varepsilon_k \|_2$. Since $\varepsilon_{t+1}$ and $\widetilde \varepsilon_k$ are i.i.d copies of the random variable $\varepsilon$, $W_{t+1}$ and $\widetilde W_k$ are i.i.d copies of a non-negative random variable $W$ with characteristic function $\varphi_W$. Denote $I = \left(\mathbf{1}\{s(1) = i \}, \dots, \mathbf{1}\{s(T) = i \} \right) =: (I_1, \dots, I_T) \in \{0, 1\}^T$. We will show that the characteristic functions of $U = \sum_{t=0}^{T-1}W_{t+1}I_{t+1} $ and $V = \sum_{k=1}^{n_i} \widetilde W_k$, denoted $\varphi_U$ and $\varphi_V$ respectively, are equal.

     For a deterministic, constant vector $c \in \{0, 1\}^T$, let
    \begin{equation*}
    \begin{aligned}
        g(c) &:= \mathbb{E}\left[\exp\left(iu \sum_{t=0}^{T-1} W_{t+1}c_{t+1} \right) \right].
    \end{aligned}
    \end{equation*}
Since $W_{t+1}$ are independent, we have
\begin{equation*}
    \begin{aligned}
      g(c)  = \mathbb{E}\left[\prod_{t=0}^{T-1} \exp\left(iu W_{t+1}c_{t+1} \right) \right] 
        &= \prod_{t=0}^{T-1}\mathbb{E}\left[\exp\left(iu W_{t+1} c_{t+1} \right)\right] \\
        &= \prod_{t=0}^{T-1} \left(\varphi_W(u)\right)^{c_{t+1}} 
        = \left\{\varphi_W(u)\right\}^{\sum_{t=0}^{T-1} c_{t+1}}.
    \end{aligned}
\end{equation*}
Substituting $I$ in place of $c$, we have by the tower property, 
    \begin{equation*}
        \begin{aligned}
  \varphi_U(u) &= \mathbb{E}\left[\exp\left(iu \sum_{t=0}^{T-1} W_{t+1}I_{t+1} \right) \right]  
  \\
  &=\mathbb{E}\left[\mathbb{E}\left[\prod_{t=0}^{T-1}\exp\left(iu W_{t+1} I_{t+1} \right) \mid I \right] \right] \\
    &= \mathbb{E}\left[\left\{\varphi_W(u) \right\}^{\sum_{t=0}^{T-1} I_{t+1}} \right] \\
    &= \mathbb{E}\left[\left\{\varphi_W(u) \right\}^{n_i} \right].
        \end{aligned}
    \end{equation*}

Similarly, for a fixed integer $m$, let 
\begin{equation*}
\begin{aligned}
    h(m) := \mathbb{E}\left[\exp\left(iu \sum_{k=1}^m \widetilde W_k \right) \right] 
    = \prod_{k=1}^m \mathbb{E}\left[\exp\left(iu \widetilde W_k \right)\right] 
    = \left\{\varphi_W(u)\right\}^m.
\end{aligned}
\end{equation*}
Now substituting $n_i$ in place of $m$, we obtain
\begin{equation*}
    \begin{aligned}
        \varphi_V(u) = \mathbb{E}\left[\exp\left(iu \sum_{k=1}^{n_i} \widetilde W_k \right) \right] 
        &= \mathbb{E}\left[\mathbb{E}\left[\exp\left( iu \sum_{k=1}^{n_i} \widetilde W_k\right) \mid n_i\right] \right] \\
        &= \mathbb{E}\left[\left\{\varphi_W(u)^{n_i}\right\} \right] 
        = \varphi_U(u),
    \end{aligned}
\end{equation*}
completing the proof.

\end{proof}

\begin{claim}\label{cl:eps-stop-inde}
    For a measurable set $B$ and any event $A \in \mathcal{F}_{\tau_{\ell}}$, we have 
    \begin{equation*}
        \mathbb{P}\left(\left\{\varepsilon_{\tau_{\ell}+1} \in B \right\} \cap A \right) = \mathbb{P}(\varepsilon \in B)\mathbb{P}(A).
    \end{equation*}
\end{claim}
\begin{proof}
    Enumerating over all possible times that the hitting times could occur, we have 
    \begin{equation*}
        \begin{aligned}
            \mathbb{P}\left(\left\{\varepsilon_{\tau_{\ell}+ 1} \in B \right\} \cap A \right) &= \sum_{t=0}^{\infty} \mathbb{P}\left(\left\{\varepsilon_{t+1} \in B \right\} \cap A \cap \{\tau_{\ell} = t \} \right).
        \end{aligned}
    \end{equation*}
Recall that 
\begin{equation*}
    \mathcal{F}_{\tau_{\ell}} = \left\{A \in \mathcal{A}: \forall t, A \cap \{\tau_{\ell} \leq t \} \in \mathcal{F}_t \right\}.
\end{equation*}
Thus, if $A \in \mathcal{F}_{\tau_{\ell}}$, then $A \cap \{\tau_{\ell} = t \} \in \mathcal{F}_t$, which means that $A \cap \{\tau_{\ell} = t \}$ is $\mathcal{F}_t$-measurable, so it only depends on information up to $t$. Since $\varepsilon_{t+1}$ is independent of the filtration $\mathcal{F}_t$, it is independent of all the events inside $\mathcal{F}_t$. Thus, 
\begin{equation*}
    \begin{aligned}
        \sum_{t \geq 0}\mathbb{P}\left(\{\varepsilon_{t+1} \in B \} \cap A \cap \{\tau_{\ell} = t\} \right) &= \sum_{t \geq 0}\mathbb{P}\left(\varepsilon_{t+1} \in B \right)\mathbb{P}\left(A \cap \{\tau_{\ell} = t \} \right) \\
        &= \mathbb{P}(\varepsilon \in B)\sum_{t \geq 0}\mathbb{P}(A \cap \{\tau_{\ell} = t \} ) \\
        &= \mathbb{P}(\varepsilon \in B )\mathbb{P}(A).
    \end{aligned}
\end{equation*}

\end{proof}

\begin{claim}\label{cl: uniform-exp-moment}
    Let $(X_t)_{t \geq 0}$ satisfy Assumption \ref{ass:1}. Then, it holds that 
    \begin{equation*}
        \sup_{t \geq 0} \mathbb{E}[\exp(\alpha \|X_t\|_2^2)] \leq K_{B, \rho}.
    \end{equation*}
\end{claim}
\begin{proof}
    By Assumption \ref{ass:1}, we have
    \begin{equation*}
        \mathbb{E}[\exp(\alpha \|X_{t+1}\|_2^2) \mid X_t] \leq \rho \exp(\alpha \|X_t\|_2^2) + B,
    \end{equation*}
which by the Tower property, implies that
\begin{equation*}
    \mathbb{E}[\exp(\alpha \|X_{t+1}\|_2^2] \leq \rho \mathbb{E}[\exp(\alpha \|X_t\|_2^2)] + B.
\end{equation*}
Iterating, we obtain
\begin{equation*}
    \begin{aligned}
        \mathbb{E}[\exp(\alpha \|X_t\|_2^2)] &\leq \rho^t \mathbb{E}[\exp(\alpha \|X_0\|_2^2)] + B\sum_{j=0}^{t-1} \rho_j 
        \leq \frac{B}{1 - \rho} 
        =: K_{B, \rho},
    \end{aligned}
\end{equation*}
where $K_{B,\rho}$ is some constant, since $X_0 = 0$ and $\rho < 1$. Taking supremums on both sides of the last display, we obtain
\begin{equation*}
    \sup_{t \geq 0}\mathbb{E}[\exp(\alpha \|X_t\|_2^2)] \leq K_{B, \rho},
\end{equation*}
as desired.

\end{proof}

\newpage 

\bibliographystyle{abbrvnat}
\bibliography{references}

\end{document}